\documentclass{article}
\usepackage{latexsym}
\usepackage{amsmath, amsfonts, amsthm}
\usepackage{epsfig}
\usepackage{stmaryrd}
\usepackage{multicol}
\usepackage{pdfsync}
\usepackage{dsfont}
\usepackage{comment}
\usepackage{algorithmic}
\usepackage{algorithm}
\usepackage{caption}
\usepackage{psfrag}
\usepackage{subfig}
\usepackage{xcolor}
\usepackage{graphics}
\usepackage{enumitem}
\usepackage{mathtools}
\usepackage{bbm}
\usepackage{mathrsfs} 
\usepackage{color}
\usepackage[title]{appendix}
\usepackage{bm}
\usepackage{epstopdf}
\usepackage{diagbox}
\usepackage{tikz-cd}

\newcommand{\prob}{{\bf P}}

\newcommand{\e}{{\bf E}}

\newcommand{\norm}[1]{\left\lVert #1 \right\rVert}
\newcommand{\bae}{\begin{equation}\begin{aligned}}
\newcommand{\eae}{\end{aligned}\end{equation}}
\newcommand{\beq}{\begin{equation}}
\newcommand{\eeq}{\end{equation}}

\newtheorem{theorem}{Theorem}[section]

\newtheorem{lemma}[theorem]{Lemma}
\newtheorem{proposition}[theorem]{Proposition}
\theoremstyle{definition}

\newtheorem{assumption}[theorem]{Assumption}

\theoremstyle{remark}
\newtheorem{remark}[theorem]{Remark}

\numberwithin{equation}{section}

\title{Continuous-time stochastic gradient descent for optimizing over the stationary distribution of stochastic differential equations}

\author{Ziheng Wang\footnote{Mathematical Institute, University of Oxford, E-mail: wangz1@math.ox.ac.uk.} \ and Justin Sirignano\footnote{Mathematical Institute, University of Oxford, Corresponding Author, E-mail: Justin.Sirignano@maths.ox.ac.uk.}}

\begin{document}
\maketitle

\begin{abstract}
We develop a new continuous-time stochastic gradient descent method for optimizing over the stationary distribution of stochastic differential equation (SDE) models. The algorithm continuously updates the SDE model's parameters using an estimate for the gradient of the stationary distribution. The gradient estimate is simultaneously updated using forward propagation of the SDE state derivatives, asymptotically converging to the direction of steepest descent. We rigorously prove convergence of the online forward propagation algorithm for linear SDE models (i.e., the multi-dimensional Ornstein-Uhlenbeck process) and present its numerical results for nonlinear examples. The proof requires analysis of the fluctuations of the parameter evolution around the direction of steepest descent. Bounds on the fluctuations are challenging to obtain due to the online nature of the algorithm (e.g., the stationary distribution will continuously change as the parameters change). We prove bounds for the solutions of a new class of Poisson partial differential equations (PDEs), which are then used to analyze the parameter fluctuations in the algorithm. Our algorithm is applicable to a range of mathematical finance applications involving statistical calibration of SDE models and stochastic optimal control for long time horizons where ergodicity of the data and stochastic process is a suitable modeling framework. Numerical examples explore these potential applications, including learning a neural network control for high-dimensional optimal control of SDEs and training stochastic point process models of limit order book events.
\end{abstract}

\section{Introduction}

\hspace{1.4em} Consider a parametric process $X^\theta_t \in \mathbb{R}^d$ which satisfies the stochastic differential equation (SDE):
\bae
\label{ergodic process}
dX_t^\theta &= \mu(X_t^\theta, \theta) dt + \sigma(X_t^\theta, \theta) dW_t, \\
X_0^\theta &= x,
\eae
where $\theta \in \mathbb{R}^\ell, \mu \in \mathbb{R}^d, \sigma \in \mathbb{R}^{d \times d}$, and $W_t$ is a standard Brownian motion. Suppose $X_t^\theta$ is ergodic with the stationary distribution $\pi_\theta$.\footnote{ Sufficient conditions (\cite{pardoux2003poisson}) for the existence and uniqueness of  $\pi_\theta$ are: (1) both coefficients $\mu$ and $\sigma$ are assumed to be bounded and $\sigma$ is uniformly continuous with respect to $x$ variable, (2) $ \displaystyle \lim _{|x| \to \infty} \sup _{\theta} \mu(x, \theta) x = -\infty$, and (3) there exist two constants $0<\lambda<\Lambda<\infty$  such that
$ \lambda I_d \leq \sigma \sigma^{\top}(x, \theta) \leq \Lambda I_d $ where $I_d$ is the $d\times d$ identity matrix.}

Our goal is to select the parameters $\theta$ which minimize the objective function 

\beq
\label{objective function}
J(\theta) = \sum\limits_{n=1}^N  \left( \e_{ \pi_\theta} \left[ f_n(Y) \right] - \beta_n \right)^2,
\eeq
where $Y$ is a random variable with distribution $\pi_\theta$, $f_n$ are known functions, and $\beta_n$ are the target quantities. Thus, we are interested in optimizing the parameterized SDEs (\ref{ergodic process}) such that their stationary distribution matches, as closely as possible, the target statistics $\beta_n$. In practice, the target statistics may be data from real-world observations which are then used to calibrate the SDE model (\ref{ergodic process}). 

\subsection{Existing methods to optimize over the stationary distribution of SDEs} \label{ExistingMethods}

\hspace{1.4em} The stationary distribution $\pi_{\theta}$ is typically unknown and therefore it is challenging to optimize over $J(\theta)$. The quantity $\e_{Y \sim\pi_\theta} \big{[} f_n(Y) \big{]}$ as well as its gradient $\nabla_{\theta} \e_{Y \sim\pi_\theta} \big{[} f_n(Y) \big{]}$ must be estimated in order to minimize $J(\theta)$. $\e_{Y \sim\pi_\theta} \big{[} f_n(Y) \big{]}$ can be evaluated using the forward Kolmogorov equation

\begin{eqnarray}
\mathcal{L}^{\theta, *}_{x} p_\infty(x, \theta) = 0,
\label{ForwardKolmogorov}
\end{eqnarray}
where $\mathcal{L}^{\theta}_x$ is the infinitesimal generator of the process $X_t^{\theta}$ and $\mathcal{L}^{\theta, *}_x$ is the adjoint operator of $\mathcal{L}^{\theta}_x$. $\nabla_{\theta} \e_{Y \sim\pi_\theta} \left[ f_n(Y) \right]$ can be calculated using an appropriate adjoint PDE for (\ref{ForwardKolmogorov}) \cite{annunziato2013fokker, butt2022numerical, fleig2017optimal, kaltenbacher2018parameter}. However, if the dimension of $d$ for $X_t^{\theta}$ is large, solving the forward Kolmogorov equation and its adjoint PDE become extremely computationally expensive. In the special case where the drift function $\mu$ is the gradient of a scalar function and the volatility function $\sigma$ is constant, there exists a closed-form formula for the stationary distribution \cite{Pavliotis}.

Alternatively, $\e_{Y \sim\pi_\theta} \big{[} f_n(Y) \big{]}$ can be approximated by simulating (\ref{ergodic process}) over a long time $[0,T]$. Similar to \cite{carmona2021deep}, the gradient descent algorithm would be:

\begin{itemize}
\item Simulate $X^{\theta_k}_t$ for $ t \in [0,T]$.
\item Evaluate the gradient of $J_T(\theta_k) := \displaystyle \sum_{n=1}^N \left( \frac1T \int_0^T f_n(X^{\theta_k}_t) dt - \beta_n \right)^2$ .
\item Update the parameter as $\theta_{k+1} = \theta_k - \alpha_k \nabla_\theta J_T(\theta_k)$,
\end{itemize}
where $\alpha_k$ is the learning rate. This gradient descent algorithm will be slow; a long simulation time $T$ will be required for each optimization iteration. A second disadvantage is that $J_T(\theta)$ is an approximation to $J(\theta)$ and therefore error is introduced into the algorithm, i.e. $\nabla_{\theta} J_T(\theta) \neq \nabla_{\theta} J(\theta)$.

\subsection{An Online Optimization Algorithm}

\hspace{1.4em} We propose a new continuous-time stochastic gradient descent algorithm which allows for computationally efficient optimization of \eqref{objective function}. The algorithm uses \textbf{online forward propagation} to asymptotically estimate the gradient of the objective function with respect to the parameters. For notational convenience (and without loss of generality), we will set $N = 1$ and $\beta_1 = \beta$. The online forward propagation algorithm for optimizing \eqref{objective function} is:
\bae
\label{nonlinear update}
\frac{d\theta_t}{dt} &= -2\alpha_t \left(f(\bar X_t) - \beta \right) \left(\nabla f(X_t) \tilde X_t\right)^\top, \\
d \tilde X_t &= \left( \nabla_x \mu( X_t, \theta_t )\tilde X_t + \nabla_\theta \mu(X_t, \theta_t) \right) dt + \left( \nabla_x \sigma(X_t,\theta_t)\tilde X_t + \nabla_\theta \sigma(X_t, \theta_t) \right) dW_t, \\
dX_t &= \mu(X_t, \theta_t) dt + \sigma(X_t, \theta_t) dW_t, \\
d\bar X_t &= \mu( \bar X_t, \theta_t ) dt + \sigma(\bar X_t, \theta_t) d \bar W_t,
\eae
where $W_t$ and $\bar W_t$ are independent Brownian motions and $\alpha_t$ is the learning rate. Before proceeding with our analysis, we first clarify the notation in \eqref{nonlinear update}. In this paper, the Jacobian matrix of a vector value function $f: x \in \mathbb{R}^n \to \mathbb{R}^m $ is an $m \times n$ matrix, i.e. $\nabla_x f(x) \in \mathbb{R}^{n \times m}$. When the function has only one variable, we may omit the subscript in the gradient. For example, we may use $\nabla f(x)$ to denote $\nabla_x f(x)$. For functions of several variables, we use the subscript in the gradient to denote the partial derivative with respect to a subset of variables. For example, we will use $\nabla_x \mu(X_t^\theta, \theta)$ to denote $\nabla_x \mu(x, \theta) \Big|_{x = X_t^\theta}$. Therefore, the variables have the following dimensions:
$$
\tilde X_t \in \mathbb{R}^{d \times \ell}, \quad \nabla_x \mu \in \mathbb{R}^{d \times d}, \quad \nabla_\theta \mu \in \mathbb{R}^{ d \times \ell}, \quad \nabla_x \sigma \in \mathbb{R}^{d \times d \times d}, \quad \nabla_\theta \sigma \in \mathbb{R}^{d \times d \times \ell}.
$$
Let $\tilde X^i_t$ denote the i-th row of $\tilde X_t$ and then the dynamics of $\tilde X_t$ in \eqref{nonlinear update} are:
$$
d \tilde X^i_t = \left( \nabla_x \mu_i( X_t, \theta_t )\tilde X_t + \nabla_\theta \mu_i(X_t, \theta_t) \right) dt + \sum_{j=1}^d \left( \nabla_x \sigma_{i, j} (X_t,\theta_t)\tilde X_t + \nabla_\theta \sigma_{i, j}(X_t, \theta_t) \right) dW^j_t.
$$
In \eqref{nonlinear update}, $\bar X_t$ and $X_t$ have the same dynamics, although they are driven by independent Brownian motions. The role of $\bar X_t$ will be explained in detail later in this section. The learning rate $\alpha_t$ in \eqref{nonlinear update} must be chosen such that $\int_0^{\infty} \alpha_s ds = \infty$ and $\int_0^{\infty} \alpha_s^2 ds < \infty$. (An example is $\alpha_t = \frac{C}{1 + t}$.) $\tilde X_t$ estimates the derivative of $X_t$ with respect to $\theta_t$. The parameter $\theta_t$ is continuously updated using $\left(f(\bar X_t) - \beta \right) \left( \nabla f(X_t) \tilde X_t \right)^\top$ as a stochastic estimate for $\nabla_{\theta} J(\theta_t)$. Deterministic gradient descent in continuous-time is often referred to as a ``gradient flow"; therefore, the proposed algorithm can be viewed as a ``stochastic gradient flow".

To better understand the algorithm (\ref{nonlinear update}), let us first re-write the gradient of the objective function using the ergodicity of $X_t^{\theta}$:
\begin{eqnarray}
\label{gradient}
\nabla_\theta J(\theta) &=&  2 \left( \e_{Y \sim \pi_\theta} f(Y) - \beta \right) \nabla_\theta \e_{Y \sim\pi_\theta} f(Y) \notag \\
&\overset{a.s.}=& 2\left( \lim_{T \to \infty}\frac{1}{T} \int_0^T f(X^{\theta}_t) dt- \beta \right) \cdot \nabla_\theta \left( \lim_{T \to \infty}\frac1T \int_0^T f(X^\theta_t)dt \right).
\end{eqnarray}
If the derivative and the limit can be interchanged, the gradient can be expressed as
\begin{eqnarray}
\nabla_\theta J(\theta)= 2\left(\lim_{T \to \infty}\frac{1}{T} \int_0^T f(X^{\theta}_t) dt - \beta \right) \cdot \lim_{T \to \infty} \frac1T \int_0^T \nabla f(X^\theta_t) \nabla_\theta X_t^\theta dt.
\end{eqnarray}
Define $\tilde X_t^{\theta} = \nabla_\theta X_t^\theta$ and, under mild regularity conditions for the coefficients (see for example \cite{rockner2021strong, wang2022forward}), $\tilde X_t^{\theta}$ will satisfy
\beq
\label{tilde X theta}
d \tilde X_t^{\theta} = \left( \nabla_x \mu( X_t^{\theta}, \theta ) \tilde X_t^{\theta} + \nabla_\theta \mu(X_t^{\theta}, \theta) \right) dt + \left( \nabla_x \sigma(X_t^{\theta},\theta) \tilde X_t^{\theta} + \nabla_\theta \sigma(X_t^{\theta}, \theta) \right) dW_t.
\eeq
Note that $\tilde X_t$ and $\tilde X_t^{\theta}$ satisfy the same equations, except $\theta$ is a fixed constant for $\tilde X_t^{\theta}$ while $\theta_t$ is updated continuously in time for $\tilde X_t$. Then, we have that

\begin{eqnarray}
\nabla_\theta J(\theta)= 2\left( \lim_{T \to \infty}\frac{1}{T} \int_0^T f(X^{\theta}_t ) dt - \beta \right) \cdot \lim_{T \to \infty} \frac1T \int_0^T \nabla f(X^\theta_t) \tilde X_t^{\theta} dt.
\label{IntroJ00}
\end{eqnarray}

The formula (\ref{IntroJ00}) can be used to evaluate $\nabla_{\theta} J(\theta)$ and thus allows for optimization via a gradient descent algorithm. However, as highlighted in Section \ref{ExistingMethods}, $X_t^{\theta}$ must be simulated for a large time period $[0,T]$ for each optimization iteration, which is computationally costly. A natural alternative is to develop a \emph{continuous-time} stochastic gradient descent algorithm which updates $\theta$ using a stochastic estimate $G(\theta_t)$ for $\nabla_\theta J(\theta_t)$, where $G(\theta_t)$ asymptotically converges to an unbiased estimate for the direction of steepest descent $\nabla_{\theta} J(\theta_t)$. (The random variable $G(\theta_t)$ is called an unbiased estimate for $\nabla_{\theta} J(\theta_t)$ if $\e[ G(\theta_t) | \theta_t ] = \nabla_{\theta} J(\theta_t)$.)  The online algorithm (\ref{nonlinear update}) does exactly this using $G(\theta_t) = 2 \left(f(\bar X_t) - \beta \right) \nabla f(X_t) \tilde X_t$ as a stochastic estimate for $\nabla_{\theta} J(\theta_t)$. 

For large $t$, we expect that $\e\left[ f(\bar X_t) - \beta \right] \approx \e_{Y \sim \pi_{\theta_t}}\left[ f(Y) - \beta \right]$ and $\e \left[ \nabla f(X_t) \tilde X_t \right] \approx \nabla_\theta \left( \e_{Y \sim \pi_{\theta_t}}\left[ f(X) - \beta \right] \right)$ since $\theta_t$ is changing very slowly as $t$ becomes large due to $\displaystyle \lim_{t \rightarrow \infty} \alpha_t = 0$. Here we highlight that for random variables $X$ and $Y$, it is not typically true that $\e[ X Y ] = \e X \cdot \e Y$ unless $X$ and $Y$ are independent. This is the reason why the process $\bar X_t$ is introduced. Since $\bar X_t$ and $X_t$ are driven by independent Brownian motions, we expect that $\e \left[ 2 \left(f(\bar X_t) - \beta \right) \nabla f(X_t) \tilde X_t  \right] \approx \nabla_{\theta} J(\theta_t)$ for large $t$ due to $\bar X_t$ and $(X_t, \tilde X_t)$ becoming asymptotically independent since $\theta_t$ will be changing very slowly for large $t$. Thus, we expect that for large $t$, the stochastic sample $G(\theta_t) = 2 \left(f(\bar X_t) - \beta \right) \nabla f(X_t) \tilde X_t$ will provide an asymptotically unbiased estimate for the direction of steepest descent $\nabla_{\theta} J(\theta_t)$ and $\norm{\nabla_{\theta} J(\theta_t)}$ will converge to zero as $t \rightarrow \infty$.

\subsection{Contributions of this Paper}

\hspace{1.4em} We rigorously prove the convergence of the algorithm (\ref{nonlinear update}) when $\mu(\cdot)$ is linear and for constant $\sigma$. Even in the linear case, the distribution of $(X_t, \bar X_t, \tilde X_t, \theta_t)$ will be non-Gaussian and convergence analysis is non-trivial. Unlike in the traditional stochastic gradient descent algorithm, the data is not i.i.d. (i.e., $X_t$ is correlated with $X_s$ for $s \neq t$) and, for a finite time $t$, the stochastic update direction $G(\theta_t)$ is not an unbiased estimate of $\nabla_{\theta} J(\theta_t)$. One must show that asymptotically $G(\theta_t)$ becomes an unbiased estimate of the direction of steepest descent $\nabla_{\theta} J(\theta_t)$. Furthermore, it must be proven that the stochastic fluctuations of $G(\theta_t)$ around the direction of steepest descent vanish in an appropriate way as $t \rightarrow \infty$.

The proof therefore requires analysis of the fluctuations of the stochastic update direction $G(\theta_t)$ around $\nabla_{\theta} J(\theta_t)$. Bounds on the fluctuations are challenging to obtain due to the online nature of the algorithm. The stationary distribution $\pi_{\theta_t}$ will continuously change as the parameters $\theta_t$ evolve. We prove bounds on a new class of Poisson partial differential equations, which are then used to analyze the parameter fluctuations in the algorithm. The fluctuations are re-written in terms of the solution to the Poisson PDE using Ito's Lemma, the PDE solution bounds are subsequently applied, and then we can show asymptotically that the fluctuations vanish. Our main theorem proves for the multi-dimensional Ornstein-Uhlenbeck process that:
\begin{eqnarray}
\lim_{t \rightarrow \infty} \left| \nabla_{\theta} J(\theta_t) \right| \overset{a.s.} = 0.
\end{eqnarray}

In the numerical section of this paper, we evaluate the performance of our online algorithm (\ref{nonlinear update}) for a variety of linear and nonlinear examples. In these examples, we show that the algorithm can also perform well in practice for nonlinear SDEs. We also demonstrate that the online algorithm can optimize over path-dependent SDEs and pathwise statistics of SDEs such as the auto-covariance. In addition, we also demonstrate the applications of the online optimization algorithm to mathematical finance problems, such as SDE model calibration, parameter estimation for partially-observed SDEs, high dimensional stochastic control problems, and limit order book models.

\subsection{Literature Review}

\hspace{1.4em} In this paper we show that, if $\alpha_{t}$ is appropriately chosen, then $\nabla_\theta J(\theta_t) \rightarrow 0$ as $t \rightarrow \infty$ with probability 1. Similar results have been previously proven for stochastic gradient descent (SGD) in discrete time. \cite{bertsekas2000gradient} proves the convergence of SGD with i.i.d. data samples. \cite{benveniste2012adaptive} proves the convergence of SGD in discrete time with the correlated data samples under stronger conditions than \cite{bertsekas2000gradient}. We refer readers to \cite{benveniste2012adaptive, bertsekas2000gradient, bottou2018optimization, goodfellow2016deep, kushner2003stochastic} for a thorough review of the very large literature on SGD and similar stochastic optimization algorithms (e.g., SGD with momentum, Adagrad, ADAM, and RMSprop). However, these articles do not study stochastic gradient descent methods for optimizing over the stationary distribution of stochastic models, which is the focus of our paper. 

Recent articles such as \cite{bhudisaksang2021online,  sharrock2020two, sirignano2017stochastic, sirignano2020stochastic, surace2018online} have studied continuous-time stochastic gradient descent. \cite{sirignano2017stochastic} proposed a ``stochastic gradient descent in continuous time" (SGDCT) algorithm for estimating parameters $\theta$ in an SDE $X_t^{\theta}$ from continuous observations of $X_t^{\theta^{\ast}}$ where $\theta^{\ast}$ is the true parameter. \cite{sirignano2017stochastic} proves convergence of the algorithm to a stationary point. \cite{bhudisaksang2021online} extended SGDCT to estimate the drift parameter of a continuous-time jump-diffusion process. \cite{sirignano2020stochastic} analyzed proved a central limit theorem for the SGDCT algorithm and a convergence rate for strongly convex objective functions. \cite{sharrock2020two} established the almost sure convergence of two-timescale stochastic gradient descent algorithms in continuous time.  \cite{surace2018online} designed an online learning algorithm for estimating the parameters
of a partially observed diffusion process and studied its convergence.  \cite{sharrock2021parameter} proposes an online estimator for the parameters of
the McKean-Vlasov SDE and proves that this estimator converges in $L_1$ to the stationary points of the asymptotic log-likelihood.

Our paper has several important differences as compared to \cite{bhudisaksang2021online, sharrock2020two,  sharrock2021parameter, sirignano2017stochastic, sirignano2020stochastic, surace2018online}. These previous papers estimate the parameter $\theta$ for the SDE $X_t^{\theta}$ from observations of $X_t^{\theta^{\ast}}$ where $\theta^{\ast}$ is the true parameter. In this paper, our goal is to select $\theta$ such that the stationary distribution of $X_t^{\theta}$ matches certain target statistics. Therefore, unlike the previous papers, we are directly optimizing over the stationary distribution of $X_t^{\theta}$. The presence of the $X$ process in SGDCT makes the mathematical analysis challenging as the $X$ term introduces correlation across times, and this correlation does not disappear as time tends to infinity. In order to prove convergence, \cite{sirignano2017stochastic, sirignano2020stochastic} use an appropriate Poisson PDE \cite{gilbarg2015elliptic, pardoux2001poisson, pardoux2003poisson} associated with $X$ to describe the evolution of the parameters for large times and analyze the fluctuations of the parameter around the direction of steepest descent. However, the theoretical results from \cite{pardoux2001poisson, pardoux2003poisson} do not apply to the PDE considered in this paper since the diffusion term in our PDE is not uniformly elliptic. This is a direct result of the process $\tilde X_t$ in (\ref{nonlinear update}), which shares the same Brownian motion with the process $X_t$. In the case of constant $\sigma$, the PDE operator will not be uniformly elliptic and, furthermore, the coefficient for derivatives such as $\frac{\partial^2 }{\partial \tilde x^2}$ is zero. Consequently, we must analyze a new class of Poisson PDEs which is different than the class of Poisson PDEs studied in \cite{pardoux2001poisson, pardoux2003poisson}. We prove there exists a solution to this new class of Poisson PDEs which satisfies polynomial bounds. The polynomial bounds are crucial for analyzing the fluctuations of the parameter evolution in the algorithm (\ref{nonlinear update}).

\subsection{Organization of Paper}

The paper is organized into three main sections.  In Section \ref{main result}, we present the assumptions and the main theorem. Section \ref{detail proof} rigorously proves the convergence of our algorithm for multi-dimensional linear SDEs. Section \ref{numerical experiment} studies the numerical performance of our algorithm for a variety of linear and nonlinear SDEs, including McKean-Vlasov and path-dependent SDEs. Applications of the online optimization algorithm in mathematical finance are discussed, including SDE model calibration, parameter estimation for partially-observed SDE models, stochastic optimal control, and mean-field games. Numerical examples demonstrate how the method can be used to numerically solve high-dimensional stochastic optimal control problems and high-dimensional stochastic models of limit order book events.

\section{Main Result}\label{main result}

\hspace{1.4em} In this section, we rigorously prove convergence of the algorithm \eqref{nonlinear update} for the following multi-dimensional Ornstein–Uhlenbeck process:
\bae
\label{process}
dX^\theta_t &= \left( g(\theta) - h(\theta) X^\theta_t \right)dt + \sigma dW_t, \\
X_0^\theta &= x,
\eae
where $\theta \in \mathbb{R}^\ell$, $g(\theta) \in \mathbb{R}^d$, $h(\theta) \in \mathbb{R}^{d \times d}_+$, $W_t \in \mathbb{R}^d$, $X^\theta_t \in \mathbb{R}^d$, and $\sigma$ is a scalar constant. Since $h(\theta)$ is positive definite, the solution to the SDE (\ref{process}) is
\beq
\label{solution}
X^\theta_t = e^{-h(\theta) t} x + \left(h(\theta)\right)^{-1}\left(I_d- e^{-h(\theta)t}\right)g(\theta) + e^{-h(\theta)t} \int_0^t e^{h(\theta)s} \sigma dW_s,
\eeq
where $I_d$ is the $d\times d$ identity matrix. Let $\pi_\theta$ be the stationary distribution of $X_t^\theta$. ($\pi_{\theta}$ exists and is unique; for example, see \cite{Pavliotis}.) Our goal is to solve the optimization problem 
\beq
\label{object}
\min\limits_{\theta} J(\theta) = \min\limits_{\theta}  \left( \e_{Y \sim \pi_\theta}f(Y) - \beta \right)^2,
\eeq
where $\beta$ is a constant. To solve \eqref{object}, our online algorithm \eqref{nonlinear update} becomes:
\bae
\label{update}
\frac{d\theta_t}{dt} &= -2\alpha_t \left(f(\bar X_t) - \beta \right)  \nabla f(X_t) \tilde X_t, \\
dX_t &= ( g(\theta_t) - h(\theta_t) X_t ) dt + \sigma d W_t, \\
\frac{d \tilde X_t}{dt} &= \nabla_\theta g(\theta_t) - \nabla_\theta h(\theta_t) X_t - h(\theta_t) \tilde X_t,\\
d\bar X_t &= ( g(\theta_t) - h(\theta_t) \bar X_t ) dt + \sigma d \bar W_t,
\eae
where $W_t$ and $\bar W_t$ are independent Brownian motions, $\nabla_\theta g(\theta_t) \in \mathbb{R}^{d \times \ell}, \  \nabla_\theta h(\theta_t) \in \mathbb{R}^{d \times d \times \ell}$ and $\tilde X_t \in \mathbb{R}^{d \times \ell}$ is the gradient process for $X_t$. The element $(i,j)$ of the process $\tilde{X}_t$ satisfies:
\beq
\frac{d}{dt} \tilde{X}_t^{i, j} = \frac{\partial g_i(\theta_t)}{\partial \theta_j} - \sum\limits_{k=1}^d \frac{\partial h_{ik}(\theta_t) }{\partial \theta_j} X_t^{k} - \sum\limits_{k=1}^d h_{ik}(\theta_t) \tilde X_t^{k, j}, \quad i \in \{1,2,\cdots, d\}, \quad j \in \{1,2,\cdots, \ell\}.
\eeq

For the rest of this article, we will use $C, C_k, C_p$ to denote generic constants. Our convergence theorem will require the following assumptions.

\begin{assumption}
\label{condition}
\begin{itemize}
\item[(1)] $g(\theta)$, $\nabla^i_\theta g(\theta)$, $h(\theta)$ and $\nabla^i_\theta h(\theta)$ are uniformly bounded functions for $i=1,2$. 
\item[(2)] $h$ is symmetric and uniformly positive definite, i.e. there exists a constant $c>0$ such that 
$$ 
\min\left\{ x^\top h(\theta) x \right\} \ge c|x|^2, \quad \forall \theta \in \mathbb{R}^\ell, x\in \mathbb{R}^d.
$$
\item[(3)] $f, \nabla^i f, i=1,2,3$ are polynomially bounded\footnote{$|\cdot|$ denotes the Euclidean norm. Sometimes for a square matrix $x$, $|x|$ will be used to denote its spectral norm which is equivalent to the Euclidean norm.}:
\beq
\label{poly}
|f(x)| + \sum\limits_{i=1}^3|\nabla^i f(x)| \le C(1+|x|^{\hat{m}}),\quad \forall x\in \mathbb{R}^d 
\eeq
for some constant $C, \hat{m}>0$. 
\item[(4)] The learning rate $\alpha_t$ satisfies $\int_{0}^{\infty} \alpha_{t} d t=\infty$, $\int_{0}^{\infty} \alpha_{t}^{2} d t<\infty$, $\int_{0}^{\infty}\left|\alpha_{s}^{\prime}\right| d s<\infty$, and there is a $\hat{p}>0$ such that $\displaystyle \lim _{t \rightarrow \infty} \alpha_{t}^{2} t^{\frac12 + 2 \hat{p}}=0$.
\end{itemize} 
\end{assumption}

Under these assumptions, we are able to prove the following convergence result.

\begin{theorem}
\label{conv f}
Under Assumption \ref{condition} and for the Ornstein–Uhlenbeck process \eqref{process}, the algorithm \eqref{update} will converge to a stationary point almost surely:
\beq
\lim_{t \rightarrow \infty} \left| \nabla_\theta J(\theta_t) \right|  \overset{a.s.} =  0. 
\eeq
\end{theorem}

\section{Proof of Theorem \ref{conv f}}\label{detail proof}

\hspace{1.4em} In this section, we present the proof of Theorem \ref{conv f}.  We begin by decomposing the evolution of $\theta_t$ in \eqref{update} into several terms:
\begin{eqnarray}
\label{gradient with error}
\frac{d\theta_t}{dt} &=&  -2\alpha_t ( f(\bar X_t) - \beta ) \left( \nabla f(X_t) \tilde X_t\right)^\top \notag \\
&=& -2\alpha_t (\e_{Y \sim \pi_{\theta_t}}f(Y) - \beta) \left(\nabla f(X_t) \tilde X_t\right)^\top - 2\alpha_t \left( f(\bar X_t) - \e_{Y \sim \pi_{\theta_t}}f(Y) \right)  \left(\nabla f(X_t) \tilde X_t\right)^\top \notag \\
&=&  \underbrace{-\alpha_t \nabla_\theta J(\theta_t)}_{\textrm{Direction of Steepest Descent}} - \underbrace{ 2\alpha_t (\e_{Y \sim \pi_{\theta_t}}f(Y)-\beta) \left( \nabla f(X_t) \tilde X_t - \nabla_\theta \e_{Y \sim \pi_{\theta_t}}f(Y) \right)^\top }_{\textrm{Fluctuation term $1$}} \notag \\
&-& \underbrace{2\alpha_t \left( f(\bar X_t) - \e_{Y \sim \pi_{\theta_t}}f(Y) \right)  \left(\nabla f(X_t)\tilde X_t\right)^\top }_{\textrm{Fluctuation term $2$}}.
\end{eqnarray}
Define the error terms
\bae
\label{error}
Z_t^1 &= (\e_{Y \sim \pi_{\theta_t}}f(Y)-\beta) \left( \nabla f(X_t) \tilde X_t -  \nabla_\theta \e_{Y \sim \pi_{\theta_t}}f(Y) \right)^\top, \\
Z_t^2 &= \left( f(\bar X_t) - \e_{Y \sim \pi_{\theta_t}}f(Y) \right)  \left(\nabla f(X_t) \tilde X_t \right)^\top.
\eae
We have therefore decomposed the evolution of $\theta_t$ into the direction of steepest descent $-\alpha_t \nabla_\theta J(\theta_t)$ and the two fluctuation terms $2\alpha_t Z_t^1$ and $2\alpha_t Z_t^2$.

As in \cite{sirignano2017stochastic}, we study a cycle of stopping times to control the time periods where $|\nabla_\theta J(\theta_t)|$ is close to zero and away from zero. Let us select an arbitrary constant $\kappa>0$ and also define $\mu=\mu(\kappa)>0$ (to be chosen later). Then set $\sigma_{0}=0$ and define the cycles of random times
$$
0=\sigma_{0} \leq \tau_{1} \leq \sigma_{1} \leq \tau_{2} \leq \sigma_{2} \leq \ldots,
$$
where for $k=1,2, \ldots$
\bae
\label{cycle of time}
&\tau_{k}=\inf \left\{t>\sigma_{k-1}:\left|\nabla_\theta  J\left(\theta_{t}\right)\right| \geq \kappa\right\} \\
&\sigma_{k}=\sup \left\{t>\tau_{k}: \frac{\left|\nabla_\theta J\left(\theta_{\tau_{k}}\right)\right|}{2} \leq\left|\nabla_\theta J\left(\theta_{s}\right)\right| \leq 2\left|\nabla_\theta J\left(\theta_{\tau_{k}}\right)\right| \text { for all } s \in\left[\tau_{k}, t\right] \text { and } \int_{\tau_{k}}^{t} \alpha_{s} d s \leq \mu \right\}.
\eae 
We define the random time intervals $J_{k}=\left[\sigma_{k-1}, \tau_{k}\right)$ and $I_{k}=\left[\tau_{k}, \sigma_{k}\right)$. We introduce $\eta > 0$ which will be chosen to be sufficiently small later. We first seek to control
\beq
\label{error integral}
\Delta^i_{\tau_k,\sigma_k + \eta} := \int_{\tau_k}^{\sigma_k + \eta} \alpha_s Z^i_s ds, \quad i =1,2
\eeq
and, as in \cite{sirignano2017stochastic}, we will use a Poisson equation to bound the online fluctuation terms $\Delta^{i}_{\tau_k,\sigma_k + \eta}$ where the ergodic properties of $X_t^\theta$ will be leveraged in the analysis.

In this paper, we focus on the Ornstein–Uhlenbeck process \eqref{process}. As in \eqref{tilde X theta}, its gradient process $\tilde X^\theta_t := \nabla_\theta X^\theta_t = \left( \frac{\partial X_t^{\theta,i}}{\partial \theta_j}\right)_{i,j} \in \mathbb{R}^{d \times \ell}$ now satisfies the SDE:
\beq
\label{tilde theta}
\frac{d \tilde X^\theta_t}{dt} = \nabla_\theta g(\theta) - \nabla_\theta h(\theta) X_t^\theta - h(\theta) \tilde X^\theta_t,
\eeq
which can be equivalently written as
\beq
\frac{d}{dt} \frac{\partial X_t^{\theta,i}}{\partial \theta_j} = \frac{\partial g_i(\theta)}{\partial \theta_j} - \sum\limits_{k=1}^d \frac{\partial h_{ik}(\theta) }{\partial \theta_j} X_t^{\theta, k} - \sum\limits_{k=1}^d h_{ik}(\theta) \frac{\partial X_t^{\theta,k}}{\partial \theta_j},
\eeq
for $i \in \{1,2,\cdots, d\}$ and $j \in \{1,2,\cdots, \ell\}$. Thus, we know the solution of \eqref{tilde theta} with initial point $\tilde x$ is 
\beq
\label{tilde}
\tilde X^\theta_t = e^{-h(\theta)t} \tilde x + e^{-h(\theta)t} \int_0^t e^{h(\theta)s} \left(  \nabla_\theta g(\theta) - \nabla_\theta h(\theta) X_s^\theta \right)ds.
\eeq
The independent Ornstein–Uhlenbeck process used to obtain the asymptotic unbiased gradient is
\bae
\label{independent process}
d\bar X^\theta_t &= ( g(\theta) - h(\theta) \bar X^\theta_t )dt + \sigma d\bar W_t,\\
\bar X_0^\theta &= \bar x,
\eae
where $\bar W_t$ is another Brownian motion independent of $W_t$. For the processes $X_t^\theta, \tilde X_t^\theta, \bar X_t^\theta$ in \eqref{process}, \eqref{tilde theta}, and \eqref{independent process}, we can prove the following convergence results.

\begin{proposition}
\label{ergodic estimation}
Let $p_t(x,x',\theta)$ and $p_\infty(x', \theta)$ denote the transition probability and invariant density of the multi-dimensional Ornstein–Uhlenbeck process \eqref{process}. Under Assumption \ref{condition}, we have the following ergodic result:
\begin{itemize}
\item[(\romannumeral1)] For any $m>0$, there exists a constant $C=C(m)$ such that
\beq
\label{invariant density}
\left|\nabla_\theta^i p_{\infty}\left(x^{\prime}, \theta\right)\right| \leq \frac{C}{1+\left|x^{\prime}\right|^m}, \quad i = 0, 1, 2.
\eeq
\item[(\romannumeral2)] For any $m', k$ there exist constants $C,m$ such that for any $t>1$
\beq
\label{x prime decay}
\left| \nabla_\theta^i p_{t}\left(x, x^{\prime}, \theta\right) - \nabla_\theta^i p_{\infty}\left(x^{\prime}, \theta\right)\right| \leq \frac{C\left(1+|x|^{m}\right)}{\left(1+\left|x^{\prime}\right|^{m^{\prime}}\right)(1+t)^{k}}, \quad i=0,1,2.
\eeq
\item[(\romannumeral3)] For any $m^{\prime}, k$ there exist constants $C, m$ such that for any $t>1$
\beq
\label{x decay}
\left|\nabla_{x}^{j} \nabla_\theta^i p_{t}\left(x, x^{\prime}, \theta\right)\right| \leq \frac{C\left(1+|x|^{m}\right)}{\left(1+\left|x^{\prime}\right|^{m^{\prime}}\right)(1+t)^{k}}, \quad i = 0, 1, \quad j = 1,2.
\eeq
\item[(\romannumeral4)] For any $m>0$, there exists a constant $C=C(m)$ such that for any $t\ge 0$
\beq
\label{normal bound}
\e_x \left| X_t^\theta \right|^m \le C(1+|x|^m), \quad \e_{x, \tilde x} \left| \tilde X_t^\theta \right|^m \le C(1+|x|^m+|\tilde x|^m). 
\eeq
Here $\e_{x}$ denotes that the initial condition for the process $X_t^\theta$ is x, i.e. $X_0^\theta = 0$. $\e_{x, \tilde x}$ denotes that the initial conditions of the processes $(X_t^\theta, \tilde X_t^\theta)$ in \eqref{tilde X theta} are $(x, \tilde x)$, i.e. $X^\theta_0 = x$ and $\tilde X^\theta_0 = \tilde x$.
\item[(\romannumeral5)] For any function $f$ satisfying \eqref{poly}, there exists constants $C, m$ such that for any $t\in[0, 1]$
\beq
\label{expectation bound}
\left| \nabla^j_x \nabla^i_\theta \e_x f(X_t^\theta) \right| \le C(1+|x|^m), \quad i = 0,1, \quad j=0,1,2.
\eeq
\end{itemize}
\end{proposition} 

\begin{remark}
Proposition \ref{ergodic estimation} is similar to Theorem $1$ in \cite{pardoux2003poisson}. However, the assumption of uniform boundedness in \cite{pardoux2001poisson} does not hold for the multi-dimensional Ornstein–Uhlenbeck process \eqref{process}. Thus we give a brief proof by direct calculations in Section \ref{ergodic appendix}.
\end{remark}

We must analyze the fluctuation terms $Z_t^1$ and $Z_t^2$. In order to do this, we prove a polynomially-bounded solution exists to a new class of Poisson PDEs. The polynomial bound is in the spatial coordinates and, importantly, the bound is uniform in the parameter $\theta$. A Poisson PDE was also used in \cite{sirignano2017stochastic}. However, several key innovations are required for the online optimization algorithm \eqref{update} that we consider in this paper. Unlike in \cite{sirignano2017stochastic}, $\tilde X_t^\theta$ in \eqref{tilde theta} does not have a diffusion term, which means $(X^\theta_t, \tilde X_t^\theta)$ is a degenerate diffusion process and its generator $\mathcal{L}_{x,\tilde x}^\theta$ is not a uniformly elliptic operator. Thus we cannot use the results from \cite{pardoux2001poisson, pardoux2003poisson}. Instead, we must prove existence and bounds for this new class of Poisson PDEs.

\begin{lemma}
\label{poisson eq}
Define the error function 
\beq
\label{function}
G^1(x,\tilde x, \theta) = (\e_{Y \sim \pi_\theta} f(Y)-\beta) \left(\nabla f(x) \tilde x - \nabla_\theta \e_{Y \sim \pi_\theta}f(Y) \right)^\top
\eeq
and
\beq
\label{representation}
v^1(x, \tilde x, \theta) = -\int_0^\infty \e_{x,\tilde x} G^1(X_t^\theta,\tilde X_t^\theta, \theta) dt,
\eeq
where $\e_{x,\tilde x}$ is a conditional expectation given $X^\theta_0 = x$ and $\tilde X^\theta_0 = \tilde x$. Then, under Assumption \ref{condition}, $v^1(x,\tilde x, \theta)$ is the classical solution of the Poisson equation 
\beq
\label{PDE}
\mathcal{L}_{x,\tilde x}^\theta u(x,\tilde x, \theta) = G^1(x,\tilde x, \theta),
\eeq
where $u = (u_1, \ldots, u_\ell)^\top \in \mathbb{R}^\ell$ is a vector,  $\mathcal{L}_{x,\tilde x}^\theta u(x,\tilde x, \theta) = (\mathcal{L}_{x,\tilde x}^\theta u_1(x,\tilde x, \theta), \ldots, \mathcal{L}_{x,\tilde x}^\theta u_{\ell}(x,\tilde x, \theta))^\top$, and $\mathcal{L}^\theta_{x,\tilde x}$ is the infinitesimal generator of the process $(X_\cdot^\theta, \tilde X_\cdot^\theta)$, i.e. for any test function $\varphi$
\beq
\mathcal{L}^\theta_{x,\tilde x} \varphi(x,\tilde x) = \mathcal{L}_x^\theta \varphi(x, \tilde x) + \text{tr}\left( \nabla_{\tilde x}\varphi(x, \tilde x)^\top \left( \nabla_\theta g(\theta) - \nabla_\theta h(\theta) x - h(\theta) \tilde x \right) \right) .
\eeq 
Furthermore, there exist an integer $m'$ and a constant $C = C(m')$ which do not depend upon $(x,\tilde x, \theta)$ such that the solution $v^1$ satisfies the bound
\bae
\label{control v1}
\left| v^1(x, \tilde x, \theta)\right| + \left|\nabla_\theta v^1(x, \tilde x, \theta)\right|+\left|\nabla_x v^1(x, \tilde x, \theta)\right| + \left|\nabla_{\tilde x} v^1(x, \tilde x, \theta)\right| \le C\left(1+ |x|^{m'} + |\tilde x|^{m^{\prime}}\right).
\eae
\end{lemma}

The proof of Lemma \ref{poisson eq} is in Appendix \ref{Poisson appendix}. We will next study the fluctuation terms $Z_t^i$. It will be necessary to prove bounds on the moments of $X_t$ and $\tilde X_t$ in order to analyze the error term $\Delta^i_{\tau_k, \sigma_k + \eta}$. 
\begin{lemma}
\label{moment}
For any $p>0$, there exists a constant $C_p$ that only depends on $p$ such that the processes $X_t, \tilde X_t$ from \eqref{update} satisfy
\beq
\label{moment bound}
\e_x |X_t|^p \le C_p\left( 1 + |x|^p \right), \quad \e_{x,\tilde x} |\tilde X_t|^p \le C_p\left( 1 + |x|^p + |\tilde x|^p \right).
\eeq 
Furthermore, we have the bounds
\bae
\label{uniform moment bound}
\e_x \left( \sup\limits_{0 \le t' \le t} |X_{t'}|^p \right) &= O(\sqrt t) \quad \text{as}\  t \to \infty,\\
\e_{x, \tilde x} \left( \sup\limits_{0 \le t' \le t} |\tilde X_{t'}|^p \right) &= O(\sqrt t) \quad \text{as}\ t \to \infty.\\
\eae
\end{lemma}
\begin{proof}
By adapting the method in \cite{fang2017adaptive}, we first prove \eqref{moment bound} for $p \geq 2$ and then the result for $0<p<2$ follows from Hölder's inequality. Let $p = 2m$ and applying Itô's formula to $e^{m \alpha  t}\left|X_{t}\right|^{2m}$, we have for any $t \ge 0$,
\bae
\label{p order}
e^{p \alpha t / 2}\left|X_t\right|^{p}-\left|X_{0}\right|^{p} &\leq \int_{0}^{t} p\left(\frac{\alpha}{2}\left|X_{t}\right|^{2}+ \langle X_{s}, g(\theta_s) - h(\theta_s) X_{s} \rangle \right) e^{p \alpha s / 2}\left|X_{s}\right|^{p-2} ds \\
&+\int_{0}^{t} \frac{p(p-1)d}{2} e^{p \alpha s / 2}\left|X_{s}\right|^{p-2} ds + \int_{0}^{t} p e^{p \alpha s / 2}\left|X_{s}\right|^{p-2} \langle X_s, \ dW_{s} \rangle,
\eae
where $\langle a,\  b\rangle := a^T b$. By Assumption \ref{condition}, we know there exists constants $\alpha>0, \beta>d$ such that for any $\theta$
\beq
\label{ergodic}
\langle x, g(\theta) - h(\theta) x \rangle \le -\alpha |x|^2 + \beta.
\eeq
Thus by taking expectations on both sides of \eqref{p order} and using \eqref{ergodic}, we obtain
$$
\e_x \left[ e^{p \alpha t / 2}\left|X_{t}\right|^{p}\right]- \left|x\right|^{p} \leq  \int_{0}^{t} -\frac{p \alpha}{2} \e_x \left[ e^{p \alpha s / 2}\left|X_{s}\right|^{p}\right] ds +\int_{0}^{t} \e_x \left[\frac{p(p+1) \beta}{2} e^{p \alpha s / 2}\left|X_{s}\right|^{p-2}\right] ds.
$$
Young's inequality implies that
$$
\frac{p(p+1) \beta}{2} \mathrm{e}^{p \alpha s / 2}\left|X_{s}\right|^{p-2} \leq \frac{p \alpha}{2} e^{p \alpha s / 2}\left|X_{s}\right|^{p}+c_{p}e^{p \alpha s / 2}
$$
where $c_{p}=\left(\frac{p-2}{p \alpha}\right)^{p / 2-1}(\beta(p+1))^{p / 2}$. Therefore, we obtain
$$
\e_x \left[ e^{p \alpha t / 2}\left|X_{t}\right|^{p}\right]- \left|x\right|^{p} \leq \int_{0}^{t} c_{p} \mathrm{e}^{p \alpha s / 2} ds
$$
and
$$
\e_x \left|X_{t}\right|^{p} \leq \frac{2 c_{p}}{p \alpha} + e^{-p \alpha t / 2} \left|x\right|^{p} \leq C_{p}\left( 1 + |x|^p \right).
$$

Using the moment bound for $X_t$, we can derive the moment bound for $\tilde X_t$. From \eqref{tilde} and \eqref{update} we know 
\beq
\label{tilde x}
\tilde X_t = e^{-\int_0^th(\theta_u)du}\tilde X_0 + \int_0^t e^{-\int_s^t h(\theta_u)du} \left( \nabla_\theta g(\theta_s) - \nabla_\theta h(\theta_s)X_s \right) ds
\eeq
and thus 
\bae
\label{tilde bound}
\e_{x,\tilde x}|\tilde X_t|^p &\le 2|\tilde x|^p + 2\e_{x, \tilde x}\left| \int_0^t \left| e^{-\int_s^t h(\theta_u)du}\right| \cdot \left| \nabla_\theta g(\theta_s) - \nabla_\theta h(\theta_s)X_s \right| ds \right|^p\\
&\overset{(a)}{\le} 2|\tilde x|^p + C_p \e_{x, \tilde x}\left| \int_0^t e^{-c(t-s)} \left( 1 + |X_s| \right) ds \right|^p\\
&\le 2|\tilde x|^p + C_p \e_{x}\left| \int_0^t \frac{e^{cs}}{e^{ct} - 1} \left( 1 + |X_s| \right) ds \right|^p e^{-cp t}\left( e^{ct} - 1 \right)^p \\
&\overset{(b)}{\le} 2|\tilde x|^p + C_p \e_x\left| \int_0^t \frac{e^{cs}}{e^{ct} - 1} \left( 1 + |X_s| \right)^p ds \right|\\
&\le C_p \left( 1 + |x|^p + |\tilde x|^p \right),
\eae		
where step $(a)$ is by Assumption \ref{condition} and the fact
\beq
\lambda_{\max} \left( e^{-\int_s^{t'} h(\theta_u)du} \right) = e^{ -\lambda_{\min}\left( \int_s^{t'} h(\theta_u)du \right)} \le e^{-c(t'-s)}
\eeq
and step (b) is by Jensen's inequality. 

To prove \eqref{uniform moment bound}, we use a similar method as in \cite{pardoux2001poisson}. By Itô's formula, we have for $p \ge 1$
\bae
\left|X_{t}\right|^{2p}-\left|X
_{0}\right|^{2p} &\leq \int_{0}^{t} 2p \left| X_s \right|^{2p-2} \left\langle X_s,\  g(\theta_s) - h(\theta_s) X_{s} \right\rangle ds + \int_{0}^{t} p(d+ 2p-1) \left|X_{s}\right|^{2p-2} ds + 2p \int_{0}^{t} \left|X_{s}\right|^{2p-2} \langle X_s, \ dW_{s} \rangle\\
&\le C_p \int_{0}^{t} \left|X_{s}\right|^{2p-2} ds + 2p \int_{0}^{t} \left|X_{s}\right|^{2p-2} \langle X_s, \ dW_{s} \rangle.
\eae
Using the Burkholder-Davis-Gundy inequality, there exists a constant $C$ such that
\beq
\label{uniform}
\e_{x}\left(\sup _{t^{\prime} \leq t}\left|X_{t^{\prime}}\right|^{2 p}\right) \leq |x|^{2 p} + C_p \left(\e_{x} \int_{0}^{t}\left|X_{s}\right|^{4 p-2} d s\right)^{1 / 2} +  C_p \e_{x} \int_{0}^{t}\left|X_{s}\right|^{2 p-2} ds,
\eeq
which together with estimate \eqref{moment bound} can be used to derive the bound
\beq
\e_{x}\left(\sup _{t' \leq t}\left|X_{t'}\right|^{2 p}\right) \leq |x|^{2p} + C_p \left(t+t^{1 / 2}\right)\left(1+|x|^{2 p-1}\right).
\eeq
Furthermore, for $t \geq 1$,
\beq
\label{p}
\e_{x}\left(\sup _{t^{\prime} \le t}\left|X_{t^{\prime}}\right|^{p}\right) \overset{(a)}{\leq} \left( \e_{x} \sup _{t^{\prime} \leq t}\left|X_{t^{\prime}}\right|^{2p}\right)^{\frac12}
\le \left( |x|^{2p} + C_p \left(t+t^{1 / 2}\right)\left(1+|x|^{2 p-1}\right) \right)^{\frac12}
\le |x|^{p} + C_p \left(1+|x|^{p - \frac12}\right) \sqrt{t},
\eeq
where step (a) is by H\"{o}lder inequality. Similarly, we have for any $p^{\prime}<p$ and $t \geq 1$ that
\beq
\label{p'}
\e_{x}\left(\sup_{t^{\prime} \leq t}\left|X_{t^{\prime}}\right|^{p^{\prime}}\right) \leq C|x|^{p^{\prime}}+C\left(1+|x|^{p-\frac12}\right) t^{\frac{p'}{2p}},
\eeq
and thus the result for $X_t$ in \eqref{uniform moment bound} follows. Finally, similarly as in \eqref{tilde bound}, 
\bae
\label{relation}
\e_{x, \tilde x} \sup_{t' \le t} |\tilde X_{t'}|^p &\le 2|\tilde x|^p + 2\e_{x, \tilde x}\sup_{t' \le t}\left| \int_0^{t'} \left| e^{-\int_s^{t'} h(\theta_u)du} \right| \cdot \left| \nabla_\theta g(\theta_s) - \nabla_\theta h(\theta_s)X_s \right| ds \right|^p\\
&\le 2|\tilde x|^p + C_p \e_x  \sup_{t' \le t} \left| \int_0^{t'} e^{-c(t'-s)} \left( 1 + |X_s| \right) ds \right|^p\\
&\le 2|\tilde x|^p + C_p \e_x \sup_{t' \le t} \left( 1 + |X_{t'}|^p \right).
\eae
Combining \eqref{p}, \eqref{p'}, and \eqref{relation}, we can prove the bound for $\tilde X_t$ in \eqref{uniform moment bound}.
\end{proof}

Using the estimates in Lemma \ref{poisson eq} and Lemma \ref{moment}, we can now bound the first fluctuation term $\Delta^1_{\tau_k, \sigma_k + \eta}$ in \eqref{error integral}.
\begin{lemma}
\label{fluctuation 1}
Under Assumption \ref{condition}, for any fixed $\eta>0$
\beq
\label{errors conv}
\left|\Delta^1_{\tau_k, \sigma_k + \eta}\right| \rightarrow 0 \text { as } k \rightarrow \infty, \quad \text{a.s.}
\eeq
\end{lemma}
\begin{proof} 
The idea is to use the Poisson equation in Lemma \ref{poisson eq} to derive an equivalent expression for the term $\Delta^i_{\tau_k, \sigma_k + \eta}$ which we can appropriately control as $k$ becomes large. Consider the function 
$$
G^1(x,\tilde x, \theta) = (\e_{Y \sim\pi_\theta} f(Y)- \beta) \left(\nabla f(x) \tilde x - \nabla_\theta \e_{Y \sim\pi_\theta}f(Y) \right)^\top.
$$
By Lemma \ref{poisson eq}, the Poisson equation $ \mathcal{L}^\theta_{x\tilde x} u(x,\tilde x, \theta) = G^1(x, \tilde x, \theta) $ will have a unique smooth solution $v^1(x, \tilde x, \theta)$ that grows at most polynomially in $(x, \tilde x)$. Let us apply Itô's formula to the function
$$
u^1(t, x, \tilde x, \theta) := \alpha_{t} v^1(x, \tilde x, \theta) \in \mathbb{R}^\ell,
$$
evaluated on the stochastic process $(X_t, \tilde X_t, \theta_t)$. Recall that $u_i$ denotes the $i$-th element of $u$ for $i \in \{1,2,\cdots, \ell\}$. Then,
\bae
u^1_i\left(\sigma, X_{\sigma}, \tilde X_\sigma, \theta_{\sigma}\right) =& u^1_i\left(\tau, X_{\tau}, \tilde X_\tau, \theta_{\tau}\right) + \int_{\tau}^{\sigma} \partial_{s} u^1_i\left(s, X_{s}, \tilde X_s, \theta_{s}\right) ds + \int_{\tau}^{\sigma} \mathcal{L}^{\theta_s}_{x\tilde x} u^1_i\left(s, X_{s}, \tilde X_s, \theta_{s}\right) ds \\
+& \int_{\tau}^{\sigma}  \nabla_\theta u^1_i\left(s, X_{s}, \tilde X_s, \theta_{s}\right) d\theta_s + \int_{\tau}^{\sigma} \nabla_{x} u^1_i\left(s, X_{s}, \tilde X_s, \theta_{s}\right) \sigma dW_{s}. 
\eae
Rearranging the previous equation, we obtain the representation
\bae
\label{representation 1}
\Delta^1_{\tau_k,\sigma_k + \eta} =& \int_{\tau_{k}}^{\sigma_{k}+ \eta} \alpha_{s} G^1(X_s, \tilde X_s, \theta_s) ds = \int_{\tau_{k}}^{\sigma_{k}+\eta} \mathcal{L}^{\theta_s}_{x\tilde x} u^1\left(s, X_{s}, \tilde X_s, \theta_{s}\right) ds \\
=& \alpha_{\sigma_{k}+ \eta} v^1\left(X_{\sigma_{k}+\eta}, \tilde  X_{\sigma_{k}+\eta},  \theta_{\sigma_{k}+\eta}\right)-\alpha_{\tau_{k}} v^1\left(X_{\tau_{k}}, \tilde X_{\tau_{k}}, \theta_{\tau_{k}}\right)-\int_{\tau_{k}}^{\sigma_{k}+\eta} \alpha'_{s} v^1\left(X_{s}, \tilde X_s,  \theta_{s}\right) ds  \\
+&\int_{\tau_{k}}^{\sigma_{k}+\eta} 2\alpha^2_{s} \nabla_\theta v^1\left(X_{s}, \tilde X_s, \theta_{s}\right) (f(\bar X_s) - \beta) \left( \nabla f(X_s) \tilde X_s \right)^\top ds - \int_{\tau_{k}}^{\sigma_{k}+\eta} \alpha_s \nabla_{x} v^1\left(X_{s}, \tilde X_s, \theta_{s}\right) dW_{s}.
\eae
The next step is to treat each term on the right hand side of \eqref{representation 1} separately. For this purpose, let us first set
\beq
J_{t}^{1,1}=\alpha_{t} \sup _{s \in[0, t]}\left|v^1\left(X_{s}, \tilde X_s, \theta_{s}\right)\right|.
\eeq
By \eqref{control v1} and \eqref{uniform moment bound}, there exists a constant $C$ that only depends on $m'$ such that 
\bae
\e \left|J_{t}^{1,1}\right|^{2} &\leq C \alpha_{t}^{2} \e \left[1 + \sup _{s \in[0, t]}\left|X_{s}\right|^{m'} + \sup _{s \in[0, t]}\left|\tilde X_{s}\right|^{m'} \right]\\
&= C \alpha_{t}^{2}\left[1+\sqrt{t} \frac{\e \sup _{s \in[0, t]}\left|X_{s}\right|^{m'} + \e \sup _{s \in[0, t]}\left|\tilde X_{s}\right|^{m'} }{\sqrt{t}}  \right] \\
&\leq C \alpha_{t}^{2} \sqrt{t}.
\eae
Let $p>0$ be the constant in Assumption \ref{condition} such that $\lim _{t \rightarrow \infty} \alpha_{t}^{2} t^{1 / 2+2 p}=0$ and for any $\delta \in(0, p)$ define the event $A_{t, \delta}=\left\{J_{t}^{1,1} \geq t^{\delta-p}\right\} .$ Then we have for $t$ large enough such that $\alpha_{t}^{2} t^{1 / 2+2 p} \leq 1$
$$
\prob \left(A_{t, \delta}\right) \leq \frac{\e\left|J_{t}^{1,1}\right|^{2}}{t^{2(\delta-p)}} \leq C \frac{\alpha_{t}^{2} t^{1 / 2+2 p}}{t^{2 \delta}} \leq C \frac{1}{t^{2 \delta}}.
$$
The latter implies that
$$
\sum_{n \in \mathbb{N}} \prob \left(A_{2^{n}, \delta}\right)<\infty.
$$
Therefore, by the Borel-Cantelli lemma we have that for every $\delta \in(0, p)$ there is a finite positive random variable $d(\omega)$ and some $n_{0}<\infty$ such that for every $n \geq n_{0}$ one has
$$
J_{2^{n}}^{1, 1} \leq \frac{d(\omega)}{2^{n(p-\delta)}}.
$$
Thus, for $t \in\left[2^{n}, 2^{n+1}\right)$ and $n \geq n_{0}$ one has for some finite constant $C<\infty$
$$
J_{t}^{1, 1} \leq C \alpha_{2^{n+1}} \sup _{s \in\left(0,2^{n+1}\right]}\left| v^1\left(X_{s}, \tilde X_s, \theta_{s}\right) \right| \leq C \frac{d(\omega)}{2^{(n+1)(p-\delta)}} \leq C \frac{d(\omega)}{t^{p-\delta}},
$$
which proves that for $t \geq 2^{n_{0}}$ with probability one
\beq
\label{conv 1}
J_{t}^{1, 1} \leq C \frac{d(\omega)}{t^{p-\delta}} \rightarrow 0, \text { as } t \rightarrow \infty.
\eeq

Next we consider the term
$$
J_{t, 0}^{1,2} = \int_{0}^{t}\left|\alpha_{s}^{\prime} v^1\left(X_{s}, \tilde X_s, \theta_{s}\right) - 2\alpha^2_{s} \nabla_\theta v^1\left(X_{s}, \tilde X_s, \theta_{s}\right) (f(\bar X_s) - \beta) \left( \nabla f(X_s) \tilde X_s \right)^\top \right| ds.
$$
There exists a constant $0<C<\infty$ (that may change from line to line ) and $0< m' <\infty$ such that
$$
\begin{aligned}
\sup _{t>0} \e \left|J_{t, 0}^{1,2}\right| & \overset{(a)}{\le} C \int_{0}^{\infty}\left(\left|\alpha_{s}^{\prime}\right|+\alpha_{s}^{2}\right)\left(1+\e\left|X_{s}\right|^{m'} + \e\left|\bar X_{s}\right|^{m'} + \e|\tilde X_t|^{m'}\right) ds \\
& \overset{(b)}{\le} C \int_{0}^{\infty}\left(\left|\alpha_{s}^{\prime}\right|+\alpha_{s}^{2}\right) d s \\
& \leq C,
\end{aligned}
$$
where step $(a)$ is by Assumption \ref{condition} and \eqref{control v1} and in step $(b)$ we use \eqref{moment bound}. Thus  there is a finite random variable $J_{\infty, 0}^{1, 2}$ such that
\beq
\label{conv 2}
J_{t, 0}^{1, 2} \rightarrow J_{\infty, 0}^{1, 2}, \text{as} \ t \rightarrow \infty \ \text{with probability one}.
\eeq 

The last term we need to consider is the martingale term
$$
J_{t, 0}^{1, 3}=\int_{0}^{t} \alpha_s  \nabla_{x} v^1\left(X_{s}, \tilde X_s, \theta_{s}\right) d W_{s}.
$$
By Doob's inequality, Assumption \ref{condition}, \eqref{control v1}, \eqref{moment bound}, and using calculations similar to the ones for the term $J_{t, 0}^{1, 2}$, we can show that for some finite constant $C<\infty$,
$$
\sup _{t>0} \e \left|J_{t, 0}^{1, 3} \right|^{2} \le C \int_{0}^{\infty} \alpha_{s}^{2} d s<\infty.
$$
Thus, by Doob's martingale convergence theorem there is a square integrable random variable $J_{\infty, 0}^{1,3}$ such that
\beq
\label{conv 3}
J_{t, 0}^{1, 3} \to J_{\infty, 0}^{1, 3},\quad  \text{as} \ t \to \infty \ \text{ both almost surely and in $L^{2}$}.
\eeq
Let us now return to \eqref{representation 1}. Using the terms $J_{t}^{1,1}, J_{t, 0}^{1,2}$, and $J_{t, 0}^{1,3}$ we can write
$$
\left|\Delta^1_{\tau_k, \sigma_k+ \eta}\right| \leq J_{\sigma_{k}+\eta}^{1,1}+J_{\tau_{k}}^{1,1}+\left|J_{\sigma_{k}+\eta, \tau_{k}}^{1,2}\right| + \left|J_{\sigma_{k}+\eta, \tau_{k}}^{1,3}\right|,
$$
which together with \eqref{conv 1}, \eqref{conv 2}, and \eqref{conv 3} prove the statement of the Lemma.
\end{proof}

Now we prove a similar convergence result for $\Delta^2_{\tau_k, \sigma_k + \eta}$. We first give an extension of Lemma \ref{poisson eq} for the Poisson equation.
\begin{lemma}
\label{poisson eq 2}
Define the error function 
\beq
\label{function 2}
G^2(x,\tilde x, \bar x, \theta) = [ f(\bar x) -  \e_{Y \sim \pi_{\theta}}f(Y) ] \left( \nabla f(x) \tilde x \right)^\top.
\eeq
and 
\beq
\label{representation2}
v^2(x, \tilde x, \bar x, \theta) =  -\int_0^\infty \e_{x,\tilde x,\bar x} G^2(X_t^\theta,\tilde X_t^\theta, \bar X_t^\theta, \theta) dt,
\eeq
where $\e_{x,\tilde x, \bar x}$ is a conditional expectation given $X^\theta_0 = x, \ \tilde X^\theta_0 = \tilde x,$ and $\ \bar X_0^\theta = \bar x$. Under Assumption \ref{condition}, $v^2(x, \tilde x, \bar x, \theta)$ is the classical solution of the Poisson equation 
\beq
\label{PDE 2}
\mathcal{L}^\theta_{x, \tilde x, \bar x} u(x, \tilde x, \bar x, \theta) = G^2(x, \tilde x, \bar x, \theta),
\eeq
where $\mathcal{L}^\theta_{x,\tilde x,\bar x}$ is generator of the process $(X_\cdot^\theta, \tilde X_\cdot^\theta, \bar X_\cdot^\theta)$, i.e. for any test function $\varphi$
\beq
\mathcal{L}^\theta_{x,\tilde x,\bar x} \varphi(x,\tilde x, \bar x) = \mathcal{L}_{x,\tilde x}^\theta \varphi(x, \tilde x, \bar x) + \mathcal{L}_{\bar x}^\theta \varphi(x, \tilde x, \bar x).
\eeq 
Furthermore, there exist an integer $m'$ and a constant $C = C(m')$ which do not depend upon $(x, \tilde x, \bar x, \theta)$ such that the solution $v^2$ satisfies the bound	
\bae
\label{control v2}
\left| v^2(x, \tilde x, \bar x, \theta)\right|+\left| \nabla_{\bar x} v^2(x, \tilde x, \bar x, \theta)\right| + \left|\nabla_\theta v^2(x, \tilde x, \bar x, \theta)\right| +  \left| \nabla_x v^2(x, \tilde x, \bar x, \theta)\right| \le C\left(1+ |x|^{m'} + |\tilde x|^{m'} + |\bar x|^{m^{\prime}}\right).
\eae
\end{lemma} 
The proof of Lemma \ref{poisson eq 2} can be found in Appendix \ref{Poisson appendix}.

\begin{lemma}
\label{fluctuation 2}
Under Assumption \ref{condition}, for any fixed $\eta>0$, we have
\beq
\label{errors conv 2}
\left|\Delta^2_{\tau_k, \sigma_k + \eta}\right| \rightarrow 0, \text { as } k \rightarrow \infty, \quad \text{a.s.}.
\eeq
\end{lemma}
\begin{proof}
Consider the function 
\beq
G^2(x,\tilde x, \bar x, \theta) = \left( f(\bar x) -  \e_{Y \sim \pi_{\theta}}f(Y)\right) \left(\nabla f(x) \tilde x\right)^\top.
\eeq
Let $v^2$ be the solution of \eqref{PDE 2} in Lemma \ref{poisson eq 2}. We apply Itô formula to the function $u^2(t, x, \tilde x, \bar x, \theta)=\alpha_{t} v^2(x, \tilde x, \bar x, \theta)$
evaluated on the stochastic process $(X_t, \tilde X_t, \bar X_t, \theta_t)$ and get for any $i \in \{1,2,\cdots, \ell\}$ 
\bae
&u^2_i\left(\sigma, X_{\sigma}, \tilde X_{\sigma}, \bar X_\sigma,\theta_{\sigma}\right)-u^2_i\left(\tau, X_{\tau}, \tilde X_\tau, \bar X_\tau, \theta_{\tau}\right) \\
=& \int_{\tau}^{\sigma} \partial_{s} u^2_i\left(s, X_{s}, \tilde X_s, \bar X_s, \theta_{s}\right) ds + \int_{\tau}^{\sigma} \mathcal{L}^{\theta_s}_{x,\tilde x} u^2_i\left(s, X_{s}, \tilde X_s, \bar X_s, \theta_{s}\right) ds + \int_{\tau}^{\sigma} \mathcal{L}^{\theta_s}_{\bar x} u^2_i\left(s, X_{s}, \tilde X_s, \bar X_s, \theta_{s}\right) ds\\
+& \int_{\tau}^{\sigma} \nabla_{\theta} u^2_i\left(s, X_{s}, \tilde X_s, \bar X_s, \theta_{s}\right) d\theta_s + \int_{\tau}^{\sigma} \nabla_{x} u_i^2\left(s, X_{s}, \tilde X_s, \bar X_s, \theta_{s}\right) dW_{s} + \int_{\tau}^{\sigma} \nabla_{\bar x} u^2_i\left(s, X_{s}, \tilde X_s, \bar X_s, \theta_{s}\right) d\bar W_{s}.
\eae
Rearranging the previous equation, we obtain the representation
\bae
\label{representation 2}
&\Delta^2_{\tau_k,\sigma_k + \eta} = \int_{\tau_{k}}^{\sigma_{k}+ \eta} \alpha_{s} G^2( X_s, \tilde X_s, \bar X_s, \theta_s) ds = \int_{\tau_{k}}^{\sigma_{k}+\eta} \mathcal{L}^{\theta_s}_{x, \tilde x, \bar x} u^2\left(s, X_{s}, \tilde X_s, \bar X_s, \theta_{s}\right) ds \\
=& \alpha_{\sigma_{k}+ \eta} v^2\left( X_{\sigma_{k}+\eta}, \tilde X_{\sigma_k+ \eta}, \bar X_{\sigma_k+ \eta}, \theta_{\sigma_{k}+\eta}\right)-\alpha_{\tau_{k}} v^2\left(X_{\tau_{k}}, \tilde X_{\tau_{k}}, \bar X_{\tau_k}, \theta_{\tau_{k}}\right)-\int_{\tau_{k}}^{\sigma_{k}+\eta} \alpha'_{s} v^2\left(X_{s}, \tilde X_s, \bar X_s, \theta_{s}\right) ds \\
-& \int_{\tau_{k}}^{\sigma_{k}+\eta} \alpha_s \nabla_{\bar x} v^2\left(X_{s}, \tilde X_s, \bar X_s, \theta_{s}\right) d\bar W_{s} + \int_{\tau_{k}}^{\sigma_{k}+\eta} 2\alpha^2_{s} \nabla_\theta v^2\left(X_{s}, \tilde X_s, \bar X_s, \theta_{s}\right) \left(f(\bar X_s) - \beta\right)  \left( \nabla f(X_s) \tilde X_s\right)^\top ds \\
-& \int_{\tau_{k}}^{\sigma_{k}+\eta} \alpha_s \nabla_{x} v^2\left(X_{s}, \tilde X_s, \bar X_s, \theta_{s}\right) dW_s.
\eae

The next step is to treat each term on the right hand side of \eqref{representation 2} separately. 
For this purpose, let us first set
\beq
J_{t}^{2,1}=\alpha_{t} \sup _{s \in[0, t]}\left|v^2\left(X_{s}, \tilde X_s, \bar X_s, \theta_{s}\right)\right|.
\eeq
Using the same approach as for $X_t$ in Lemma \ref{moment}, we can show that for any $p>0$ there exists a constant $C_p$ that only depends on $p$ such that
\beq
\label{moment bar}
\e_{\bar x} |\bar X_t|^p \le C_p\left( 1 + |\bar x|^p \right), \quad \e_{\bar x} \left( \sup\limits_{0 \le t' \le t} |\bar X_{t'}|^p \right) = O(\sqrt t) \quad \text{as}\  t \to \infty.
\eeq
Combining Lemma \ref{moment}, \eqref{control v2}, and \eqref{moment bar}, we know that there exists a constant $C$ such that 
\bae
\e \left|J_{t}^{2,1}\right|^{2} &\leq C \alpha_{t}^{2} \e\left[1 + \sup _{s \in[0, t]}\left|X_{s}\right|^{m'} + \sup _{s \in[0, t]}\left|\tilde X_{s}\right|^{m'} + \sup _{s \in[0, t]}\left|\bar X_{s}\right|^{m'} \right]\\
&= C \alpha_{t}^{2}\left[1+\sqrt{t} \frac{\e \sup _{s \in[0, t]}\left|X_{s}\right|^{m'} + \e \sup _{s \in[0, t]}\left|\tilde X_{s}\right|^{m'} + \e \sup _{s \in[0, t]}\left|\bar X_{s}\right|^{m'}}{\sqrt{t}}  \right] \\
&\leq C \alpha_{t}^{2} \sqrt{t}.
\eae
Let $p>0$ be the constant in Assumption \ref{condition} such that $\displaystyle \lim _{t \rightarrow \infty} \alpha_{t}^{2} t^{1 / 2+2 p}=0$ and for any $\delta \in(0, p)$ define the event $A_{t, \delta}=\left\{J_{t}^{2,1} \geq t^{\delta-p}\right\} .$ Then we have for $t$ large enough such that $\alpha_{t}^{2} t^{1 / 2+2 p} \leq 1$ and 
$$
\prob \left(A_{t, \delta}\right) \leq \frac{\e\left|J_{t}^{2,1}\right|^{2}}{t^{2(\delta-p)}} \leq C \frac{\alpha_{t}^{2} t^{1 / 2+2 p}}{t^{2 \delta}} \leq C \frac{1}{t^{2 \delta}}.
$$
The latter implies that
$$
\sum_{n \in \mathbb{N}} \prob\left(A_{2^{n}, \delta}\right)<\infty.
$$
Therefore, by the Borel-Cantelli lemma we have that for every $\delta \in(0, p)$ there is a finite positive random variable $d(\omega)$ and some $n_{0}<\infty$ such that for every $n \geq n_{0}$ one has
$$
J_{2^{n}}^{2,1} \leq \frac{d(\omega)}{2^{n(p-\delta)}}.
$$
Thus for $t \in\left[2^{n}, 2^{n+1}\right)$ and $n \geq n_{0}$ one has for some finite constant $C<\infty$
$$
J_{t}^{2,1} \leq C \alpha_{2^{n+1}} \sup _{s \in\left(0,2^{n+1}\right]}\left| v^2\left(X_{s}, \tilde X_s, \bar X_s, \theta_{s}\right) \right| \leq C \frac{d(\omega)}{2^{(n+1)(p-\delta)}} \leq C \frac{d(\omega)}{t^{p-\delta}},
$$
which derives that for $t \geq 2^{n_{0}}$ we have with probability one
\beq
\label{conv 1 2}
J_{t}^{2,1} \leq C \frac{d(\omega)}{t^{p-\delta}} \rightarrow 0, \text { as } t \rightarrow \infty.
\eeq

Next we consider the term
$$
J_{t, 0}^{2,2}=\int_{0}^{t}\left|\alpha_{s}^{\prime} v^2\left(X_{s}, \tilde X_s, \bar X_s, \theta_{s}\right) - 2\alpha^2_{s} \nabla_\theta v^2\left(X_{s}, \tilde X_s, \bar X_s, \theta_{s}\right) \left(f(\bar X_s) - \beta\right) \left( \nabla f(X_s) \tilde X_s \right)^\top \right| ds
$$
and thus we see that there exists a constant $0<C<\infty$ such that 
$$
\begin{aligned}
\sup_{t>0} \e \left|J_{t, 0}^{2,2}\right| & \overset{(a)}{\le} C \int_{0}^{\infty}\left(\left|\alpha_{s}^{\prime}\right|+\alpha_{s}^{2}\right)\left(1+\e\left|X_{s}\right|^{m'}+ \e\left|X_{s}\right|^{m'} + \e\left|\bar X_{s}\right|^{m'}\right) ds \\
& \overset{(b)}{\le} C \int_{0}^{\infty}\left(\left|\alpha_{s}^{\prime}\right|+\alpha_{s}^{2}\right) d s \\
& \leq C,
\end{aligned}
$$
where in step $(a)$ we use \eqref{control v2} and in step $(b)$ we use Lemma \ref{moment} and \eqref{moment bar}. Thus we know there is a finite random variable $J_{\infty, 0}^{2,2}$ such that
\beq
\label{conv 2 2}
J_{t, 0}^{2,2} \rightarrow J_{\infty, 0}^{2,2}, \ \text{as} \ t \rightarrow \infty \ \text{with probability one}.
\eeq 

The last term we need to consider is the martingale term
$$
J_{t, 0}^{2,3}=\int_{0}^{t} \alpha_s  \nabla_{x} v^2\left(X_{s}, \tilde X_s, \bar X_s, \theta_{s}\right) dW_{s} + \int_{0}^{t} \alpha_s \nabla_{\bar x} v^2\left(X_{s}, \tilde X_s, \bar X_s, \theta_{s}\right) d\bar W_{s}.
$$
Notice that Doob's inequality and the bounds of \eqref{control v2} (using calculations similar to the ones for the term $J_{t, 0}^{2,2}$ ) give us that for some finite constant $K<\infty$, we have
$$
\sup _{t>0} \e\left|J_{t, 0}^{2,3}\right|^{2} \leq K \int_{0}^{\infty} \alpha_{s}^{2} d s<\infty.
$$
Thus, by Doob's martingale convergence theorem there is a square integrable random variable $J_{\infty, 0}^{(3)}$ such that
\beq
\label{conv 3 2}
J_{t, 0}^{2,3} \to J_{\infty, 0}^{2,3}, \quad
\text{as}\  t \to \infty \ \text{both almost surely and in $L^{2}$}.
\eeq	
Let us now go back to \eqref{representation 2}. Using the terms $J_{t}^{2,1}, J_{t, 0}^{2,2}$ and $J_{t, 0}^{2,3}$ we can write
$$
\left|\Delta^2_{\tau_k, \sigma_k + \eta}\right| \leq J_{\sigma_{k}+\eta}^{2,1}+J_{\tau_{k}}^{2,1}+J_{\sigma_{k}+\eta, \tau_{k}}^{2,2}+\left|J_{\sigma_{k}+\eta, \tau_{k}}^{2,3}\right|,
$$
which together with \eqref{conv 1 2}, \eqref{conv 2 2} and \eqref{conv 3 2} prove the statement of the Lemma.
\end{proof}

Using \eqref{theta bound} and the dominated convergence theorem, we can establish a bound for the objective function $J(\theta)$ from \eqref{object}:
\beq
\left| \nabla^2_\theta J(\theta) \right| \le C \left( \left| \e_{Y \sim \pi_\theta} f(Y) - \beta \right|^2 + \left| \nabla_\theta \e_{Y \sim \pi_\theta} f(Y) \right|^2 \right) \le C,
\eeq
and therefore the gradient $\nabla_\theta J(\theta)$ is Lipschitz continuous with respect to $\theta$.
\begin{lemma}
\label{estimation f}
Under Assumption \ref{condition}, choose $\mu>0$ in \eqref{cycle of time} such that for the given $\kappa>0$, one has $3 \mu + \frac{\mu}{8 \kappa}=\frac{1}{2L_{\nabla J}}$, where $L_{\nabla J}$ is the Lipschitz constant of $\nabla_\theta J(\theta)$ in \eqref{object}. Then, for $k$ large enough (where $k$ can be random) and $\eta>0$ small enough (potentially random depending on $k$), $\int_{\tau_{k}}^{\sigma_k + \eta} \alpha_{s} d s>\mu$ with probability one. In addition, we also have $\frac{\mu}{2} \leq \int_{\tau_{k}}^{\sigma_{k}} \alpha_{s} d s \leq \mu$ with probability one.
\end{lemma}

\begin{proof}
We use a ``proof by contradiction". Assume that $\int_{\tau_k}^{\sigma_k+\eta} \alpha_s ds \le \mu $ and let $\delta>0$ be such that $\delta<\mu/8$. Without loss of generality, we assume that for any $k$, $\eta$ is small enough such that for any $s \in [ \tau_k, \sigma_k + \eta ]$ one has $| \nabla_\theta J(\theta_s) | \le 3 | \nabla_\theta J(\theta_{\tau_k})|$.

Combining \eqref{gradient with error} and \eqref{error} yields
\beq
\frac{d\theta_t}{dt} = -\alpha_t \nabla_\theta J(\theta_t) - 2\alpha_t Z_t^1 - 2\alpha_t Z_t^2
\eeq
and thus 
\bae
\label{control}
\left| \theta_{\sigma_k + \eta} - \theta_{\tau_k} \right| &\le \int_{\tau_k}^{\sigma_k+\eta} \alpha_t \left| \nabla_\theta J(\theta_t) \right| dt + 2\left|\int_{\tau_k}^{\sigma_k+\eta} \alpha_t Z_t^1dt \right| + 2 \left| \int_{\tau_k}^{\sigma_k+\eta} \alpha_t Z_t^2 dt \right| \\ 
&\le 3 \left| \nabla_\theta J(\theta_{\tau_k})\right| \mu + I_1 + I_2.
\eae
By Lemmas \ref{fluctuation 1} and \ref{fluctuation 2}, we have that for $k$ large enough, 
\bae
\label{control online}
I_1 &\le 2| \Delta^1_{\tau_k, \sigma_k+\eta} | \le \delta < \mu/16\\
I_2 &\le 2| \Delta^2_{\tau_k, \sigma_k+\eta} | \le \delta < \mu/16.
\eae
In addition, we also have by definition that
$\frac{\kappa}{\left|\nabla_\theta J\left(\theta\left(\tau_{k}\right)\right)\right|} \le 1$. Combining \eqref{control} and \eqref{control online} yields
$$
\left|\theta_{\sigma_{k}+\eta} -\theta_{\tau_{k}} \right| \le \left|\nabla_\theta J\left(\theta_{\tau_{k}}\right)\right| \left(3 \mu+\frac{\mu}{8 \kappa}\right) = \frac{1}{2L_{\nabla J}} \left|\nabla_\theta J\left(\theta_{\tau_{k}}\right)\right|.
$$
This means that
$$
\left|\nabla_\theta J\left(\theta_{\sigma_{k}+\eta} \right)-\nabla_\theta J\left(\theta_{\tau_{k}}\right)\right| \le L_{\nabla J} \left|\theta_{\sigma_{k}+\eta} - \theta_{\tau_{k}} \right| \le \frac{1}{2}\left|\nabla_\theta J\left(\theta_{\tau_{k}} \right)\right|,
$$
and thus 
$$
\frac{1}{2}\left|\nabla_\theta J\left(\theta_{\tau_{k}} \right)\right| \le \left|\nabla_\theta J\left(\theta_{\sigma_{k}+\eta} \right)\right|\le 2\left|\nabla_\theta J\left(\theta_{\tau_{k}}\right)\right|.
$$
However, this produces a contradiction since it implies $\int_{\tau_{k}}^{\sigma_{k}+\eta} \alpha_s ds > \mu$; otherwise, from the definition of $\sigma_k$ in \eqref{cycle of time}, we will have $\sigma_{k}+\eta \in$ $\left[\tau_{k}, \sigma_{k}\right]$. This concludes the proof of the first part of the lemma. 

The proof of the second part of the lemma is straightforward. By its definition in \eqref{cycle of time}, we have that $\int_{\tau_{k}}^{\sigma_{k}} \alpha_s ds \le \mu$. It remains to show that $\int_{\tau_{k}}^{\sigma_{k}} \alpha_s ds \ge \frac{\mu}{2}$. We have shown that $\int_{\tau_{k}}^{\sigma_{k}+\eta} \alpha_s ds>\mu$. For $k$ large enough and $\eta$ small enough we can choose that $\int_{\sigma_{k}}^{\sigma_{k}+\eta} \alpha_s ds \le \frac{\mu}{2}$. The conclusion then follows.
\end{proof}

\begin{lemma}
\label{decreasing f}
Under Assumption \ref{condition}, suppose that there exists an infinite number of intervals $I_k = [\tau_k, \sigma_k)$. Then there is a fixed constant $ \gamma_1 = \gamma_1(\kappa) > 0$ such that for k large enough (where $k$ can be random),
\beq
J(\theta_{\sigma_k}) - J( \theta_{\tau_k}) \le -\gamma_1.
\eeq
\end{lemma}
\begin{proof}
By chain rule, we have that
\bae
J(\theta_{\sigma_k}) - J(\theta_{\tau_k}) &= - \int_{\tau_k}^{\sigma_k} \alpha_\rho \left|\nabla J(\theta_\rho)\right|^2 d\rho - 2\int_{\tau_k}^{\sigma_k} \alpha_\rho \langle \nabla J(\theta_\rho),\ Z^1_\rho \rangle d\rho - 2\int_{\tau_k}^{\sigma_k} \alpha_\rho \langle \nabla J(\theta_\rho),\  Z^2_\rho \rangle d\rho \\
&=: M_{1,k} + M_{2,k} + M_{3,k}.
\eae
For $M_{1, k}$, note that for $\rho \in\left[\tau_{k}, \sigma_{k}\right]$ we have $\frac{\left|\nabla_\theta J\left(\theta_{\tau_{k}} \right)\right|}{2} \le
|\nabla_\theta J(\theta_{\rho})| \le 2\left|\nabla_\theta J\left(\theta_{\tau_{k}}\right)\right|$. Thus for sufficiently large $k$, we have by Lemma \ref{estimation f}
$$
M_{1, k} \le -\frac{\left|\nabla_\theta J\left(\theta_{\tau_{k}} \right)\right|^{2}}{4} \int_{\tau_{k}}^{\sigma_{k}} \alpha_\rho d\rho \le -\frac{\left|\nabla_\theta J\left(\theta_{\tau_{k}} \right)\right|^{2}}{8} \mu.
$$
For $M_{2, k}$ and $M_{3,k}$, we can use the same method of Poisson equations as in Lemmas \ref{fluctuation 1} and \ref{fluctuation 2}. Define  
\bae
\label{object error}
G^1(x,\tilde x, \theta) &= \langle \nabla_\theta J(\theta),\ (\e_{Y \sim \pi_{\theta}}f(Y)-\beta) \left( \nabla f(x) \tilde x -  \nabla_\theta \e_{Y \sim \pi_{\theta}}f(Y) \right)^\top \rangle \\
G^2(x,\tilde x, \bar x, \theta) &= \langle \nabla_\theta J(\theta),\ \left( f(\bar x) - \e_{Y \sim \pi_{\theta}}f(Y) \right) \left(\nabla f(x) \tilde x\right)^\top \rangle,
\eae 
and use the solution of the corresponding Poisson equations 
\begin{eqnarray}
\mathcal{L}^\theta_{x, \tilde x} v^1(x, \tilde x, \theta) &=& G^1(x, \tilde x, \theta), \notag \\
\mathcal{L}^\theta_{x, \tilde x, \bar x} v^2(x, \tilde x, \bar x, \theta) &=& G^2(x, \tilde x, \bar x, \theta), 
\end{eqnarray}
as in Lemmas \ref{fluctuation 1} and \ref{fluctuation 2} to prove $M_{2,k}, M_{3,k}\to 0$ as $k \to \infty$ almost surely.

Combining the above results, we obtain that for $k$ large enough such that $\left|M_{2, k}\right| + \left|M_{3, k}\right| \le \delta<\frac{\mu}{16} \kappa^{2}$
\begin{eqnarray}
J\left(\theta_{\sigma_{k}}\right)-J\left(\theta_{\tau_{k}} \right)  &\le& -\frac{\left|\nabla J\left(\theta_{\tau_{k}} \right)\right|^{2}}{8} \mu+\delta \notag \\
&\le& -\frac{\mu}{8} \kappa^{2}+\frac{\mu}{16} \kappa^{2} \notag \\
&=& -\frac{\mu}{16} \kappa^{2}.
\end{eqnarray}
Let $\gamma_1=\frac{\mu}{16} \kappa^{2}$, which concludes the proof of the lemma.
\end{proof}

\begin{lemma}
\label{increasing f}
Under Assumption \ref{condition}, suppose that there exists an infinite number of intervals $ I_k = [\tau_k, \sigma_k)$. Then, there is a fixed constant $\gamma_2 < \gamma_1$ such that for $k$ large enough (where $k$ can be random),
\beq
J(\theta_{\tau_k}) - J(\theta_{\sigma_{k-1}}) \le \gamma_2.
\eeq
\end{lemma}

\begin{proof}
By chain rule, we have 
\bae
J(\theta_{\tau_k}) - J(\theta_{\sigma_{k-1}}) &= - \int_{\sigma_{k-1}}^{\tau_k} \alpha_\rho \left|\nabla_\theta J(\theta_\rho)\right|^2 d\rho + \int_{\sigma_{k-1}}^{\tau_k} \alpha_\rho \langle \nabla_\theta J(\theta_\rho),\ Z^1_\rho \rangle d\rho +  \int_{\sigma_{k-1}}^{\tau_k} \alpha_\rho \langle \nabla_\theta J(\theta_\rho),\ Z^2_\rho \rangle d\rho \\
&\le \int_{\sigma_{k-1}}^{\tau_k} \alpha_\rho \langle \nabla_\theta J(\theta_\rho),\ Z^1_\rho \rangle d\rho +  \int_{\sigma_{k-1}}^{\tau_k} \alpha_\rho \langle \nabla_\theta J(\theta_\rho),\ Z^2_\rho \rangle d\rho.
\eae
As in the proof of Lemma \ref{decreasing f} we get that for $k$ large enough, the right hand side of the last display can be arbitrarily small, which concludes the proof of the lemma.
\end{proof}

\begin{proof}[Proof of Theorem \ref{conv f}:]
Recalling \eqref{cycle of time}, we know $\tau_k$ is the first time $|\nabla_\theta J(\theta_t)|> \kappa$ when $t > \sigma_{k-1}$. Thus, if for any fixed $\kappa>0$, there only exists a finite number of times $\tau_{k}$, then there is a finite $T^{*}$ such that $\left|\nabla_\theta J(\theta_t)\right| \le \kappa$ for $t \ge T^{*}$ and the proof of \eqref{conv f} is complete. We now use a ``proof by contradiction". Suppose there are an infinite number of times $\tau_{k}$, then by Lemma \ref{decreasing f} and \ref{increasing f}, we have for sufficiently large $k$ (integer k can be random) that
$$
\begin{aligned}
&J\left(\theta_{\sigma_{k}}\right)-J\left(\theta_{\tau_{k}}\right) \le-\gamma_{1} \\
&J\left(\theta_{\tau_{k}}\right)-J\left(\theta_{\sigma_{k-1}}\right) \le \gamma_{2}
\end{aligned}
$$
with $0<\gamma_{2}<\gamma_{1}$. Choose $N$ large enough so that the above relations hold simultaneously for $k \geq N$. Then for all $n \ge N$
\begin{eqnarray}
J\left(\theta_{\tau_{n+1}}\right) - J\left(\theta_{\tau_{N}}\right) &=& \sum_{k=N}^{n}\left[J\left(\theta_{\sigma_{k}}\right) - J\left(\theta_{\tau_{k}}\right) + J\left(\theta_{\tau_{k+1}}\right) - J\left(\theta_{\sigma_{k}}\right)\right] \notag \\
&\le& \sum_{k=N}^{n}\left(-\gamma_{1}+\gamma_{2}\right) \notag \\
&<& (n - N) \times \left(-\gamma_{1}+\gamma_{2}\right).
\end{eqnarray}
Letting $n \rightarrow \infty$, we observe that $J\left(\theta_{\tau_{n}}\right) \rightarrow -\infty$, which is a contradiction, since by definition $J(\theta_t) \ge 0$. Thus, there can be at most finitely many $\tau_{k}$. Thus, there exists a finite random time $T$ such that almost surely $|\nabla_\theta J(\theta_t)| < \kappa$ for $t \ge T$. Since $\kappa$ is arbitrarily chosen, we have proven that $|\nabla_\theta J(\theta_t)| \to 0$ as $t \to \infty$ almost surely.
\end{proof}

\section{Numerical Performance of the Online Algorithm}{\label{numerical experiment}}

\hspace{1.4em} In this section, we will implement the continuous-time stochastic gradient descent algorithm \eqref{nonlinear update} and evaluate its numerical performance. The algorithm is implemented for a variety of linear and nonlinear models. The algorithm is also implemented for the simultaneous optimization of both the drift and volatility functions, optimizing over a path-dependent SDE, and optimizing over the auto-covariance of an SDE. In our numerical experiments, we found that the performance of the algorithm can depend upon carefully selecting hyperparameters such as the learning rate and mini-batch size. The algorithm with mini-batch size $N$ is
\bae
\label{nonlinear update mini batch}
\frac{d\theta_t}{dt} &= -2\alpha_t \left( \frac1N \sum_{i=1}^N \left(f(\bar X^{(i)}_t) - \beta \right) \right) \cdot \left( \frac1N \sum_{i=1}^N \left( \nabla f\left(X^{(i)}_t\right) \tilde X^{(i)}_t \right)^\top \right), \\
d \tilde X^{(i)}_t &= \left( \nabla_x \mu\left( X^{(i)}_t, \theta_t \right)\tilde X^{(i)}_t + \nabla_\theta \mu\left(X^{(i)}_t, \theta_t\right) \right) dt + \left( \nabla_x \sigma\left(X^{(i)}_t,\theta_t\right)\tilde X^{(i)}_t + \nabla_\theta \sigma\left(X^{(i)}_t, \theta_t\right) \right) dW^{(i)}_t, \\
dX^{(i)}_t &= \mu\left(X^{(i)}_t, \theta_t\right) dt + \sigma\left(X^{(i)}_t, \theta_t\right) dW^{(i)}_t, \\
d\bar X^{(i)}_t &= \mu\left( \bar X^{(i)}_t, \theta_t \right) dt + \sigma\left(\bar X^{(i)}_t, \theta_t\right) d \bar W^{(i)}_t,
\eae
for $i = 1, 2, \cdots, N$. The notation $(i)$ indicates the $i$-th sample in the mini-batch. $\frac1N \sum_{i=1}^N \left( f\left(\bar X^{(i)}_t\right) - \beta \right)$ and $\frac1N  \sum_{i=1}^N \left(\nabla f\left(X^{(i)}_t\right) \tilde X^{(i)}_t \right)$ are stochastic estimates of $\e_{Y \sim \pi_{\theta_t}}\left[ f(Y) - \beta \right]$ and $\nabla_\theta \left( \e_{Y \sim \pi_{\theta_t}}\left[ f(X) - \beta \right] \right)$. A larger mini-batch size reduces the noise in the estimation of the gradient descent direction. The learning rate must decay as $t \rightarrow \infty$, but it should not be decreased too rapidly and the initial magnitude should be large enough so that the algorithm converges quickly. In our examples where there is a unique global minimizer, our algorithm will always converges to the optimum if we choose the correct learning rate. For the examples with multiple global minimizers, the algorithm will converge to one of the global minimizers.

\begin{remark}
We discuss below some important aspects of the numerical implementation:
\begin{itemize}
\item[(a)] Disretization of SDEs: To implement the algorithm \eqref{nonlinear update mini batch}, we use an Euler scheme with step size $\Delta = 10^{-3} - 10^{-2}$. For example, $X_t^{(i)}$ is simulated as:
\bae
X^{(i)}_{(n+1)\Delta} &= X^{(i)}_{n \Delta} +  \left( \nabla_x \mu( X^{(i)}_{n \Delta}, \theta_{n \Delta} )\tilde X^{(i)}_{n \Delta} + \nabla_\theta \mu(X^{(i)}_{n \Delta}, \theta_{n \Delta}) \right) * \Delta \\ 
&+ \left( \nabla_x \sigma(X^{(i)}_{n \Delta}, \theta_{n \Delta})\tilde X^{(i)}_{n \Delta} + \nabla_\theta \sigma(X^{(i)}_{n \Delta}, \theta_{n \Delta}) \right) * N(0, 1) * \sqrt{\Delta}, \\
\eae
\item[(b)] Learning Rate and mini-batch size: The learning rate can be chosen to be piecewise constant or gradually decreasing with learning rate schedule
$$
\alpha_t = \frac{C}{1 + t},
$$
where $C$ is also a hyper-parameter needs to be selected. The mini-batch size $N$ that we use is of the order $10^2 - 10^4$.
\item[(c)] Initial Values for SDE simulations: In \eqref{nonlinear mini batch}, the initial value of the gradient process $\tilde X_t$ must be zero. The choice of initial points for $X_t, \bar X_t$ is flexible. In our experiments, we usually choose $X_0 = \bar X_0 = 1$. $\theta_t$ can be randomly initialized or initialized at a deterministic point such as zero. 
\item[(d)] Objective Function: For some simple examples, we can directly calculate the objective function in closed form. For those examples, we directly use that formula to compute the objective function during training. For the more complex examples (with no closed-form formula), we always approximate the objective function $J(\theta)$ using a time-average since, due to the ergodic theorem,
\beq
\label{ergodic thm}
\lim\limits_{t\to \infty} \frac1t \int_0^t f(X_s^\theta) ds = \e_{Y \sim \pi_\theta} f(Y) \quad \text{a.s.}
\eeq
\end{itemize}   
\end{remark}

\subsection{One-Dimensional Ornstein–Uhlenbeck Process}

\hspace{1.4em} We start with a simple case of a one-dimensional Ornstein–Uhlenbeck process $X_t^\theta \in \mathbb{R}$:
\beq
dX_t^\theta = (\theta - X_t^\theta) dt + dW_t.
\eeq
We will use the algorithm \eqref{nonlinear update} to learn the minimizer for 
\beq
\label{linear object}
J(\theta) = (\e_{Y \sim \pi_\theta} Y - 2)^2.
\eeq
Note that in this case we have the closed-form solution $\pi_\theta \sim N\left(\theta, \frac12\right)$ and thus the global minimizer is $\theta^* = 2$. In Figure \ref{ou mean}, several different sample paths generated by the online algorithm are plotted where all trained parameters converges to the global minimizer ($\theta^* = 2$).

Similarly, we use the algorithm \eqref{nonlinear update} to learn the minimizer for 
\beq
\label{linear object 2}
J(\theta) = (\e_{Y \sim \pi_\theta} Y^2 - 2)^2.
\eeq
In this case, the two global minimizers are $\theta^* = \pm \sqrt{1.5}$. In Figure \ref{ou mean 2}, the parameter trained by the online algorithm converges to a global minimizer. The global minimizer which the algorithm converges to depends on the initial value of $\theta_0$.

\begin{figure}[htbp]
\centering
\begin{minipage}[t]{0.48\textwidth}
\centering
\includegraphics[width=6cm]{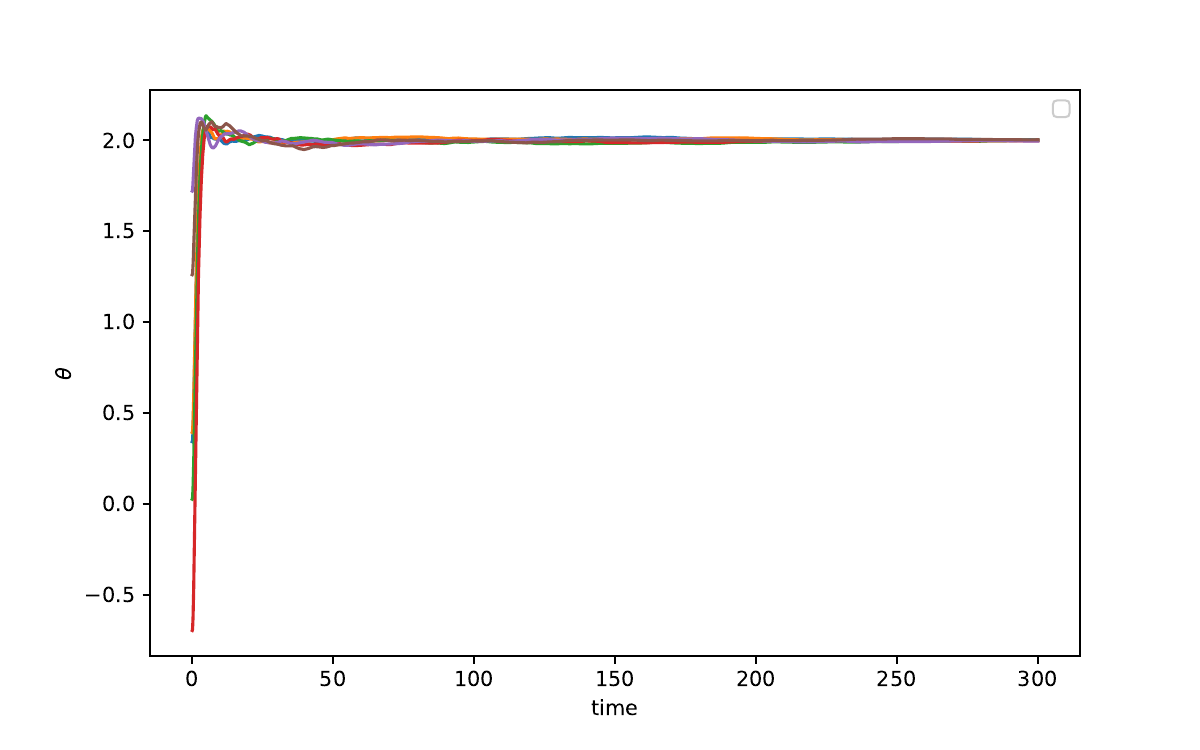}
\caption{ Online Algorithm for the objective function \eqref{linear object}. }
\label{ou mean}
\end{minipage}
\begin{minipage}[t]{0.48\textwidth}
\centering
\includegraphics[width=6cm]{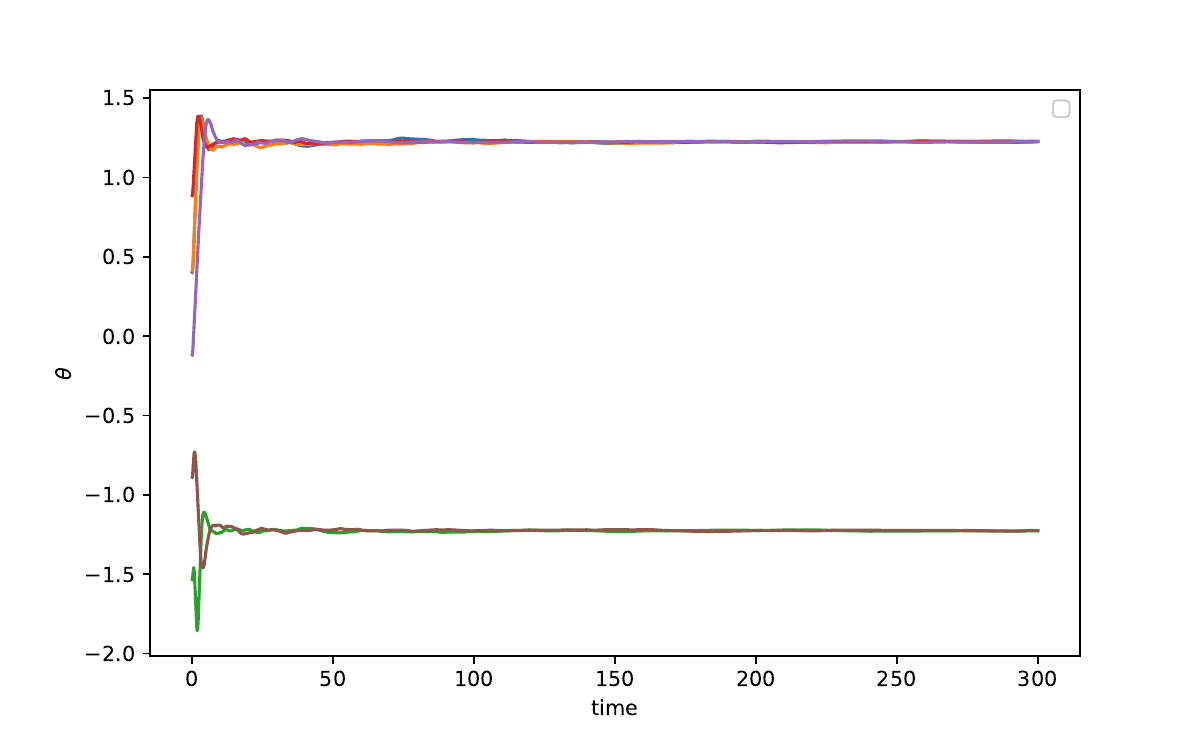}
\caption{ Online Algorithm for the objective function \eqref{linear object 2}. }
\label{ou mean 2}
\end{minipage}
\end{figure}

We now consider a more general Ornstein–Uhlenbeck process with parameters $\theta = \left(\theta^1, \theta^2\right)$:
\beq
\label{1 dim ou}
dX_t^\theta = \left(\theta^1 - \theta^2X_t^\theta \right) dt + dW_t,
\eeq
The online algorithm \eqref{nonlinear update} is used to learn the minimizer for the objective function $ J(\theta) = (\e_{Y \sim \pi_\theta} Y^2 - 2)^2 $. Algorithm \eqref{nonlinear update mini batch} will be used:
\bae
\label{ou}
d\theta^1_t &= -4 \alpha_t \left(\frac1N \sum\limits_{i=1}^{N} \left(\bar X^{(i)}_t\right)^2 - 2\right) \cdot \left( \frac1N \sum\limits_{i=1}^{N} X^{(i)}_t \tilde X^{1,(i)}_t \right)dt \\
d\theta^2_t &= -4 \alpha_t \left(\frac1N \sum\limits_{i=1}^{N} \left(\bar X^{(i)}_t\right)^2 - 2\right) \cdot \left( \frac1N \sum\limits_{i=1}^{N} X^{(i)}_t \tilde X^{2,(i)}_t \right)dt \\
dX^{(i)}_t  &= \left( \theta^1_t - \theta_t^2 X^{(i)}_t \right)dt + dW^i_t \\
d\tilde X^{1,(i)}_t &= \left(1- \theta^2_t \tilde X^{1,(i)}_t \right) dt \\
d\tilde X^{2,(i)}_t &= \left(- X^{(i)}_t - \theta^2_t \tilde X^{2,(i)}_t \right) dt \\
d\bar X^{(i)}_t  &= \left( \theta^1_t - \theta_t^2 \bar X^{(i)}_t \right)dt + d\bar W^i_t
\eae
for $i = 1,2,\cdots, N$. To make the training more stable and accelerate the convergence rate, we choose the batch size $N = 10000$. Figure \ref{ou 2p} and \ref{ou 2p object} show the dynamic of the parameters and objective function during training. 

\begin{figure}[htbp]
\centering
\begin{minipage}[t]{0.48\textwidth}
\centering
\includegraphics[width=6cm]{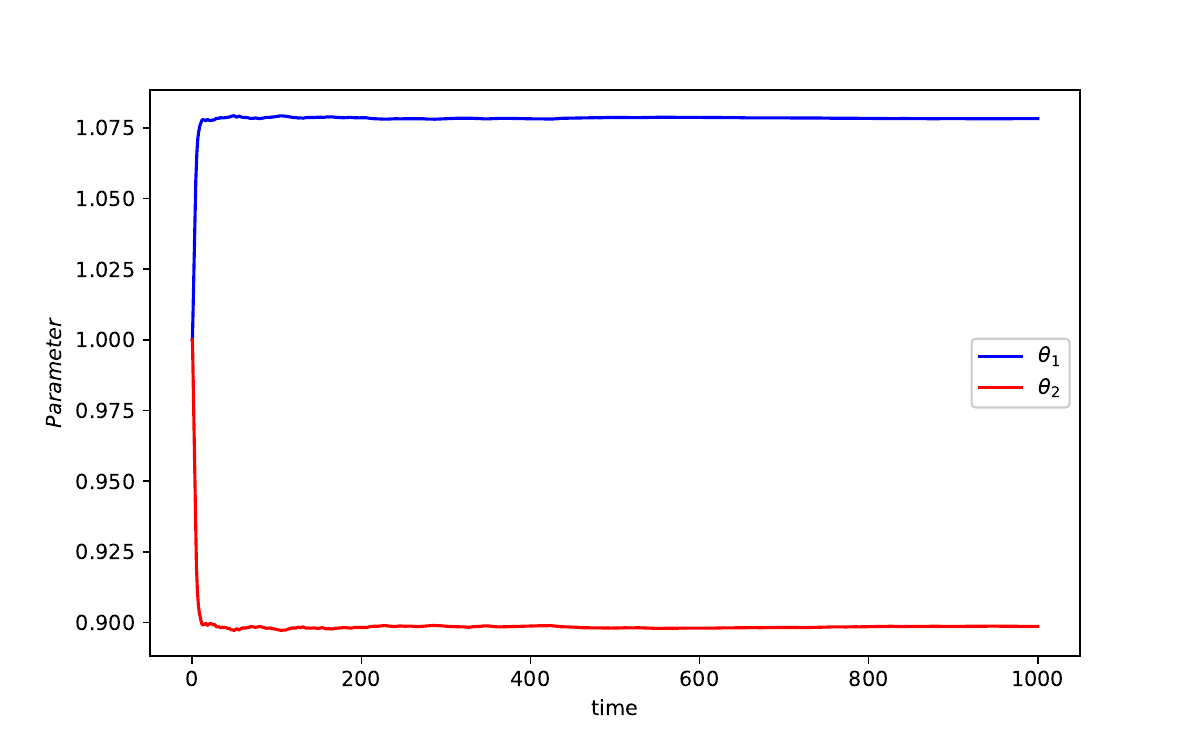}
\caption{ Parameters for algorithm \eqref{ou}.}
\label{ou 2p}
\end{minipage}
\begin{minipage}[t]{0.48\textwidth}
\centering
\includegraphics[width=6cm]{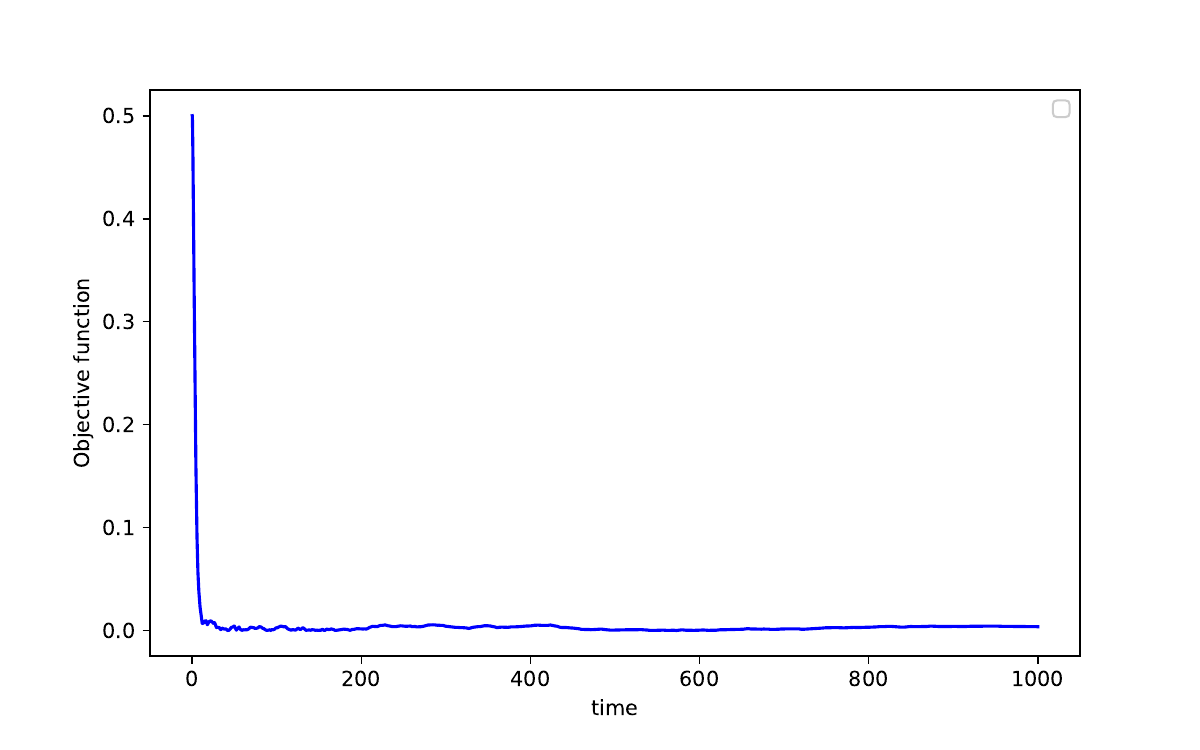}
\caption{Objective function for algorithm \eqref{ou}.}
\label{ou 2p object}
\end{minipage}
\end{figure}

\subsection{One-Dimensional Nonlinear Process}

\hspace{1.4em} We now use the online algorithm to optimize over the stationary distribution of a one-dimensional nonlinear process
\beq
\label{non process}
dX_t^\theta = \left(\theta - X_t^\theta - \left(X_t^\theta\right)^3 \right) dt + dW_t.
\eeq
We use the algorithm \eqref{nonlinear update} to learn the minimizer of $J(\theta) = \left(\e_{Y \sim \pi_\theta} Y^2 - 2\right)^2$. The mini-batch algorithm \eqref{nonlinear mini batch} is used:
\bae
\label{nonlinear mini batch}
d\theta_t &= -4\alpha_t \left( \frac1N \sum\limits_{i=1}^N \left(\bar X_t^{(i)}\right)^2 - 2 \right) \cdot \left( \frac1N \sum\limits_{i=1}^N X^{(i)}_t \tilde X^{(i)}_t \right)dt \\
dX^{(i)}_t  &= \left( \theta_t - X_t^{(i)} - \left(X^{(i)}_t\right)^3 \right)dt + dW^{(i)}_t \\
d\tilde X_t^{(i)} &= \left( 1- \tilde X_t^{(i)} -3 \left(X^{(i)}_t\right)^2 \tilde X_t^{(i)} \right) dt \\
d\bar X_t^{(i)}  &= \left( \theta_t - \bar X_t^{(i)} - \left( \bar X^{(i)}_t\right)^3 \right)dt + d\bar W^{(i)}_t
\eae
for $i = 1,2, \cdots, N$. Figure \ref{nonlinear} shows the convergence of the parameter $\theta_t$. In Figure \ref{nonlinear object}, the objective function decays to zero (the global minimum) very quickly.

\begin{figure}[htbp]
\centering
\begin{minipage}[t]{0.48\textwidth}
\centering
\includegraphics[width=6cm]{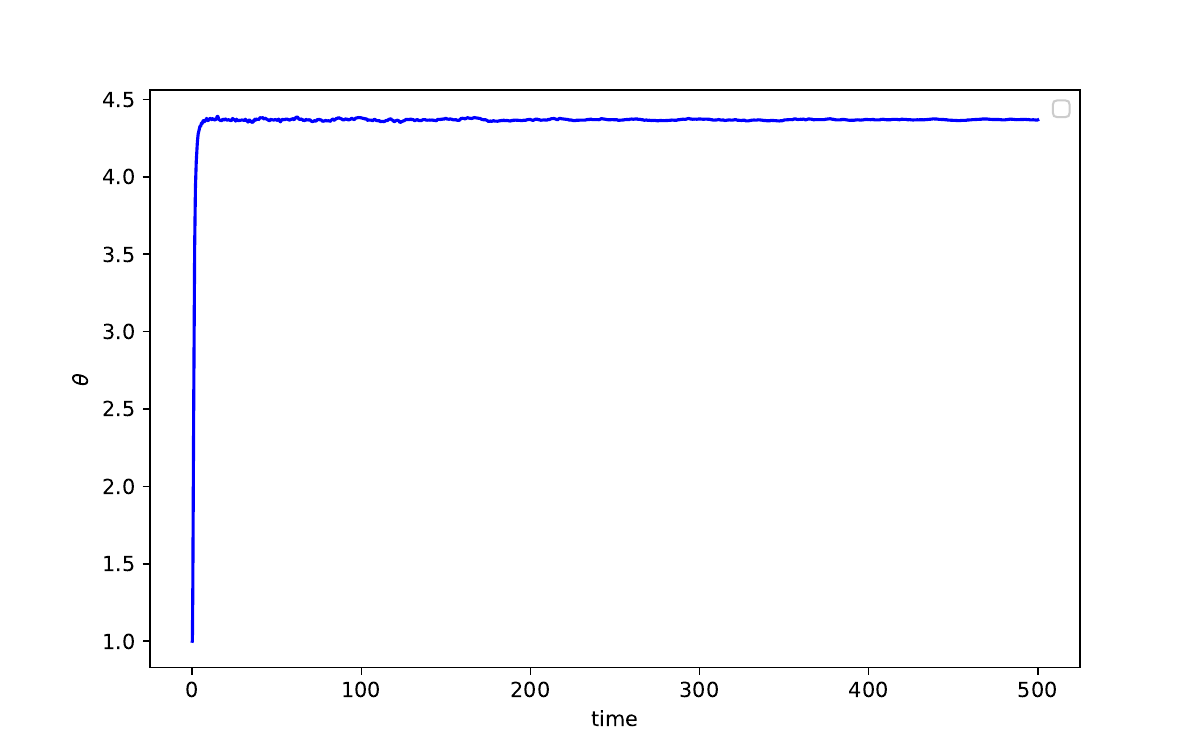}
\caption{ Parameter for algorithm \eqref{nonlinear mini batch}.}
\label{nonlinear}
\end{minipage}
\begin{minipage}[t]{0.48\textwidth}
\centering
\includegraphics[width=6cm]{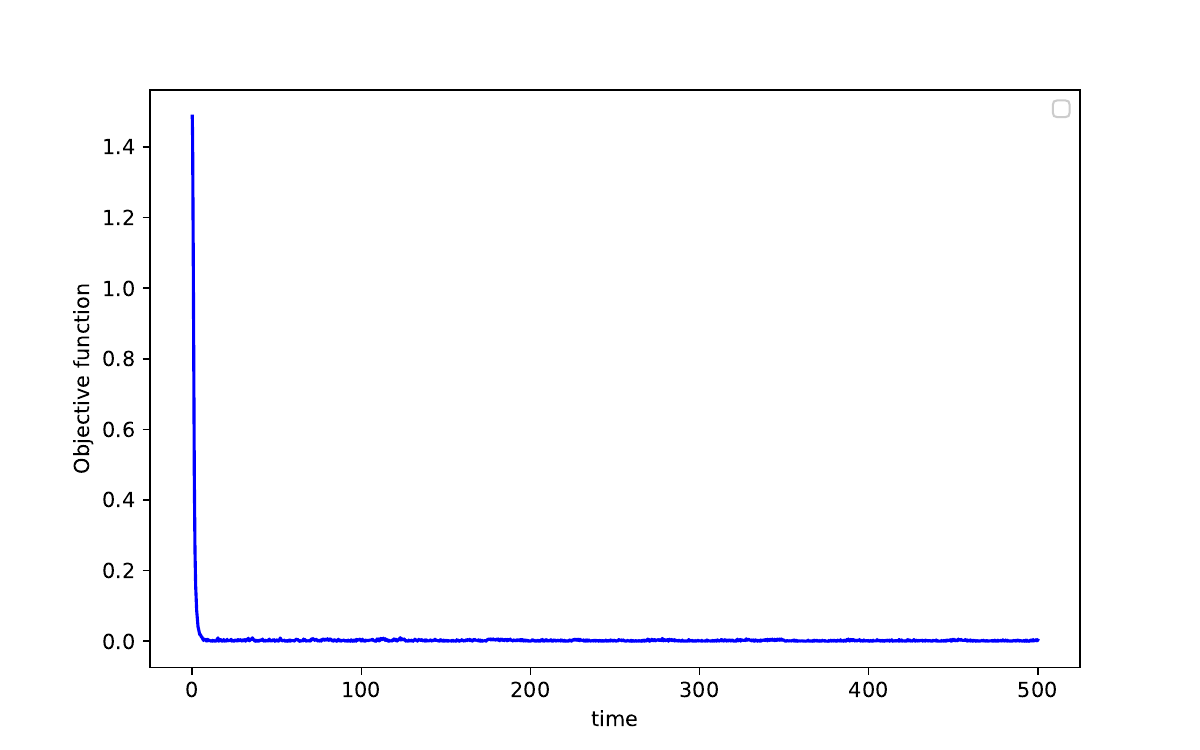}
\caption{ Objective function for algorithm \eqref{nonlinear mini batch}.}
\label{nonlinear object}
\end{minipage}
\end{figure}

\subsection{Optimizing over the Drift and Volatility Coefficients}

\hspace{1.4em} We now optimize over the drift and volatility functions of the process
\beq
dX_t^\theta = ( \mu - X_t^\theta) dt + \sigma dW_t
\eeq
with parameters $\theta = (\mu, \sigma)$. The online algorithm \eqref{nonlinear update} is used to learn the minimizer of $J(\theta) = (\e_{Y \sim \pi_\theta} Y^2 - 2)^2$. The mini-batch algorithm \eqref{nonlinear update mini batch} is used: 
\bae
\label{linear vol batch}
d\mu_t &= -4\alpha_t \left( \frac1N \sum_{i=1}^N \left(\bar X^{(i)}_t\right)^2 - 2 \right) \cdot \left( \frac1N \sum_{i=1}^N X^{(i)}_t \tilde X^{1,(i)}_t \right) dt \\
d\sigma_t &= -4\alpha_t \left( \frac1N \sum_{i=1}^N \left(\bar X^{(i)}_t\right)^2 - 2 \right) \cdot \left( \frac1N \sum_{i=1}^N X^{(i)}_t \tilde X^{2,(i)}_t \right) dt \\
dX^i_t  &= \left( \mu_t - X^{(i)}_t \right)dt + \sigma_t dW^{(i)}_t \\
d\tilde X^{1,(i)}_t &= \left(1- \tilde X^{1,(i)}_t\right) dt \\
d\tilde X^{2,(i)}_t &= - \tilde X^{2,(i)}_t dt + dW^{(i)}_t \\
d\bar X^{(i)}_t  &= \left( \mu_t - \bar X^{(i)}_t\right)dt + \sigma_t d\bar W^{(i)}_t
\eae
for $i = 1,2, \cdots, N$. In Figure \ref{linear vol}, the trained parameters $\mu_t, \sigma_t$ converge and in Figure \ref{linear vol object} the objective function $J(\theta_t) \to 0 $ very quickly.

\begin{figure}[htbp]
\centering
\begin{minipage}[t]{0.48\textwidth}
\centering
\includegraphics[width=6cm]{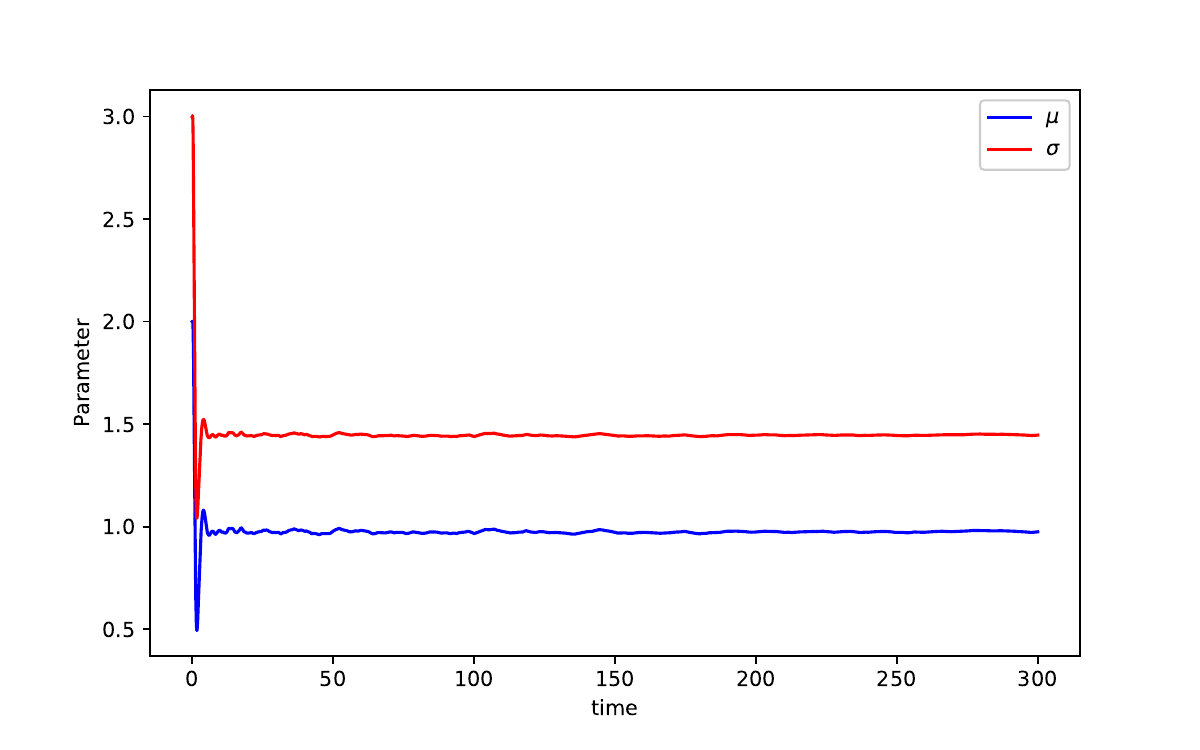}
\caption{ Parameters for algorithm \eqref{linear vol batch}.}
\label{linear vol}
\end{minipage}
\begin{minipage}[t]{0.48\textwidth}
\centering
\includegraphics[width=6cm]{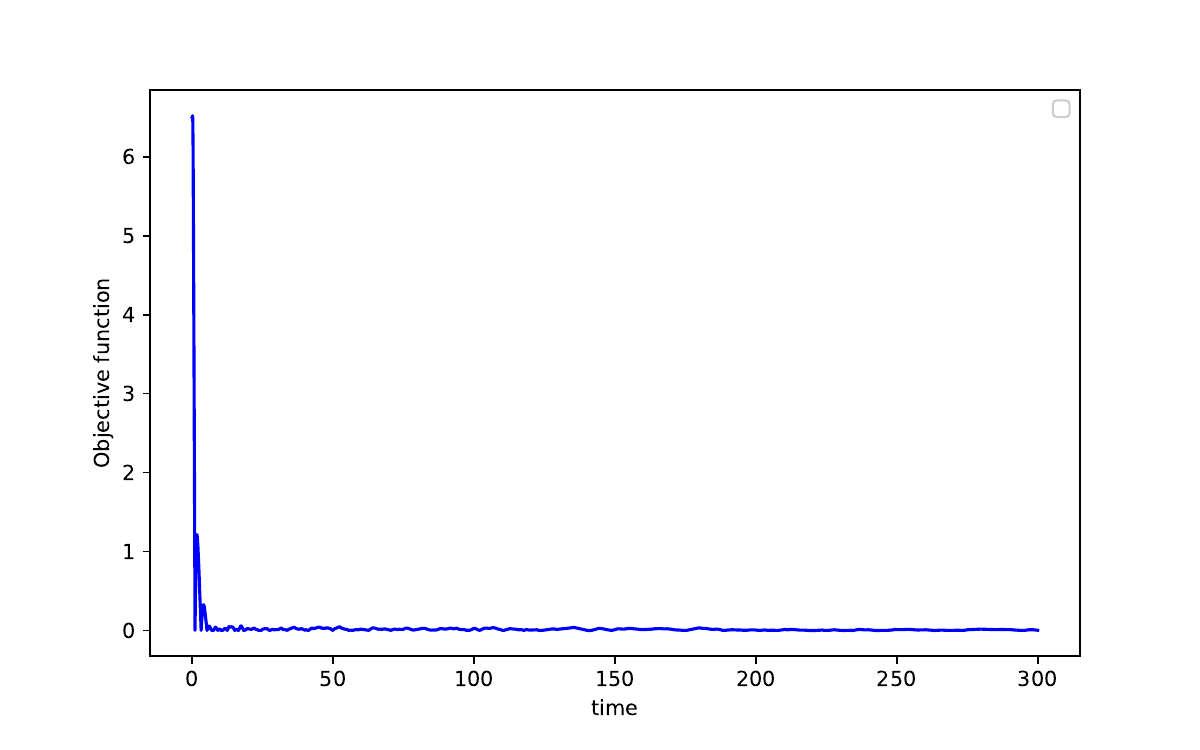}
\caption{Objective function for algorithm \eqref{linear vol batch}.}
\label{linear vol object}
\end{minipage}
\end{figure}

We also implement the online algorithm for the nonlinear process
\beq
dX_t^\theta = \left(\mu - \left(X_t^\theta\right)^3 \right) dt + \sigma X_t^\theta dW_t,
\eeq
where $\theta= (\mu, \sigma)$ are the parameters and the objective function is $J(\theta) = \left(\e_{Y \sim \pi_\theta} Y^2 - 10\right)^2$. The mini-batch algorithm \eqref{nonlinear update mini batch} now becomes:
\bae
\label{nonlinear vol batch}
d\mu_t &= -4\alpha_t \left( \frac1N \sum_{i=1}^N \left(\bar X^{(i)}_t\right)^2 - 2 \right) \cdot \left( \frac1N \sum_{i=1}^N X^{(i)}_t \tilde X^{1,(i)}_t \right) dt \\
d\sigma_t &= -4\alpha_t \left( \frac1N \sum_{i=1}^N (\bar X^{(i)}_t)^2 - 2 \right) \cdot \left( \frac1N \sum_{i=1}^N X^{(i)}_t \tilde X^{2,(i)}_t \right) dt  \\
dX^i_t  &= ( \mu_t - \left(X^{(i)}_t\right)^3 ) dt + \sigma_t X_t^{(i)} dW^{(i)}_t \\
d\tilde X^{1,(i)}_t &= \left( 1 - 3 \left(X_t^{(i)}\right)^2 \tilde X^{1,(i)}_t \right) dt + \sigma_t \tilde X_t^{1,(i)} dW_t^{(i)} \\
d\tilde X^{2,(i)}_t &= - 3 \left(X_t^{(i)}\right)^2 \tilde X^{2,(i)}_t dt + \left( X_t^{(i)} + \sigma_t \tilde X_t^{2,(i)} \right) dW^{(i)}_t \\
d\bar X^{(i)}_t  &= \left( \mu_t - \left(\bar X^{(i)}_t\right)^3 \right)dt + \sigma_t \bar X_t^{(i)} d\bar W^{(i)}_t
\eae
for $i = 1,2,\cdots, N$. In Figure \ref{nonlinear vol}, the trained parameters $\mu_t, \sigma_t$ converge and in Figure \ref{nonlinear vol object} the objective function $J(\theta_t) \to 0 $ very quickly.
\begin{figure}[htbp]
\centering
\begin{minipage}[t]{0.48\textwidth}
\centering
\includegraphics[width=6cm]{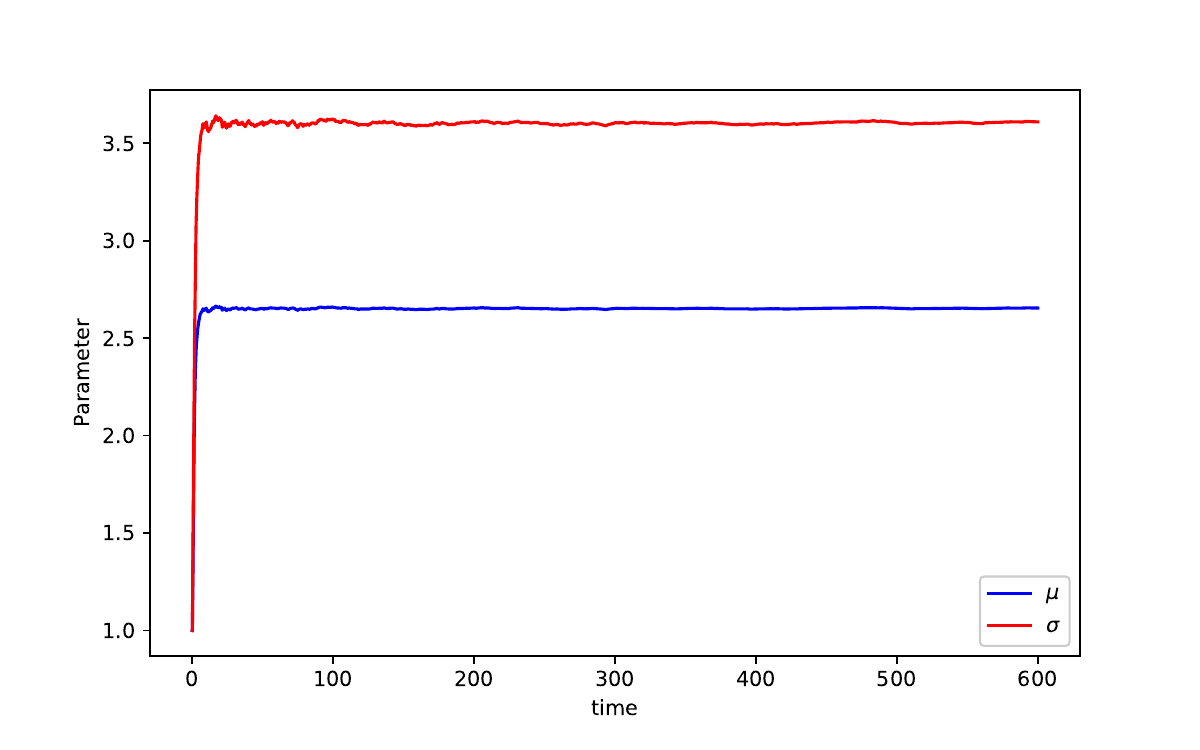}
\caption{ Parameters evolution for algorithm \eqref{nonlinear vol batch}.}
\label{nonlinear vol}
\end{minipage}
\begin{minipage}[t]{0.48\textwidth}
\centering
\includegraphics[width=6cm]{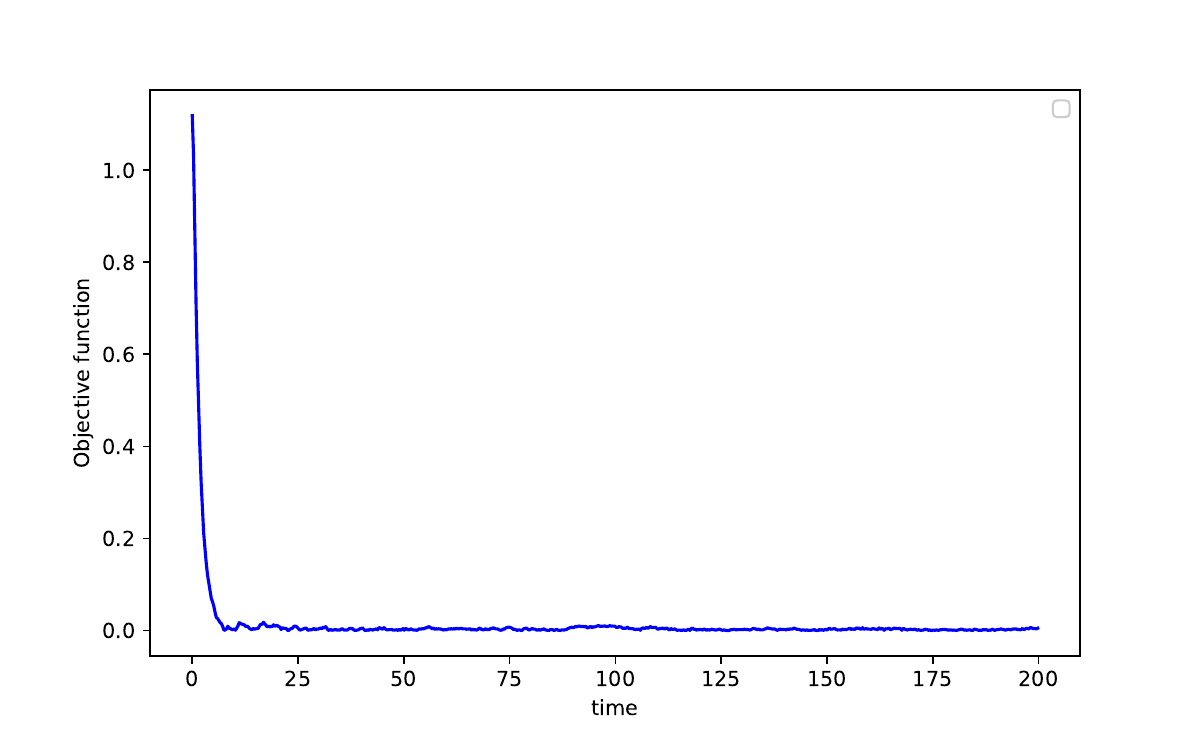}
\caption{ Objective function for  algorithm \eqref{nonlinear vol batch}.}
\label{nonlinear vol object}
\end{minipage}
\end{figure}

\subsection{Multi-Dimensional Independent Ornstein–Uhlenbeck Process}

\hspace{1.4em} We next consider a simple multi-dimensional Ornstein–Uhlenbeck process which consists of $m$ independent copies of \eqref{1 dim ou}. For the parameter $\theta = (\theta^1, \theta^2) \in R^{2m}$, let the m-dimensional Ornstein–Uhlenbeck process be 
\begin{eqnarray}
dX_t^\theta = \left(\theta^1 - \theta^2 \odot X_t^\theta\right) dt + dW_t,
\end{eqnarray}
where $X_t^\theta \in R^m$, $W_t \in R^m$, and $\odot$ is an element-wise product. The objective function is
\beq
\label{multi ou object}
J(\theta) := \left(\sum_{k=1}^m \e_{Y \sim\pi_\theta} |Y_k|^2 - 2m\right)^2.
\eeq
The online algorithm \eqref{nonlinear update} is
\bae
\label{multi ou update}
d\theta^{1}_t &= -4\alpha_t \left(|\bar X_t|^2 - 2\right) X_t \odot \tilde X^{1}_t dt, \\
d\theta^{2}_t &= -4\alpha_t \left(|\bar X_t|^2 - 2\right) X_t \odot \tilde X^{2}_t dt, \\
dX_t  &= \left( \theta^{1}_t - \theta_t^{2} \odot X_t \right)dt + dW^i_t, \\
d\tilde X^{1}_t &= \left(1- \theta^{2}_t \odot \tilde X^{1}_t \right) dt, \\
d\tilde X^{2}_t &= \left(- X_t - \theta^{2}_t \odot \tilde X^{2}_t \right) dt, \\
d\bar X_t  &= \left( \theta^{1}_t - \theta^{2}_t \odot \bar X_t \right)dt + d\bar W^i_t.
\eae
We implement the algorithm for $m =3$ and $m =10$. In Figures \ref{multi_ou_object1} and \ref{multi_ou_object2}, the objective functions $J(\theta_t) \to 0$ as $t$ becomes large. 
\begin{figure}[htbp]
\centering
\begin{minipage}[t]{0.48\textwidth}
\centering
\includegraphics[width=6cm]{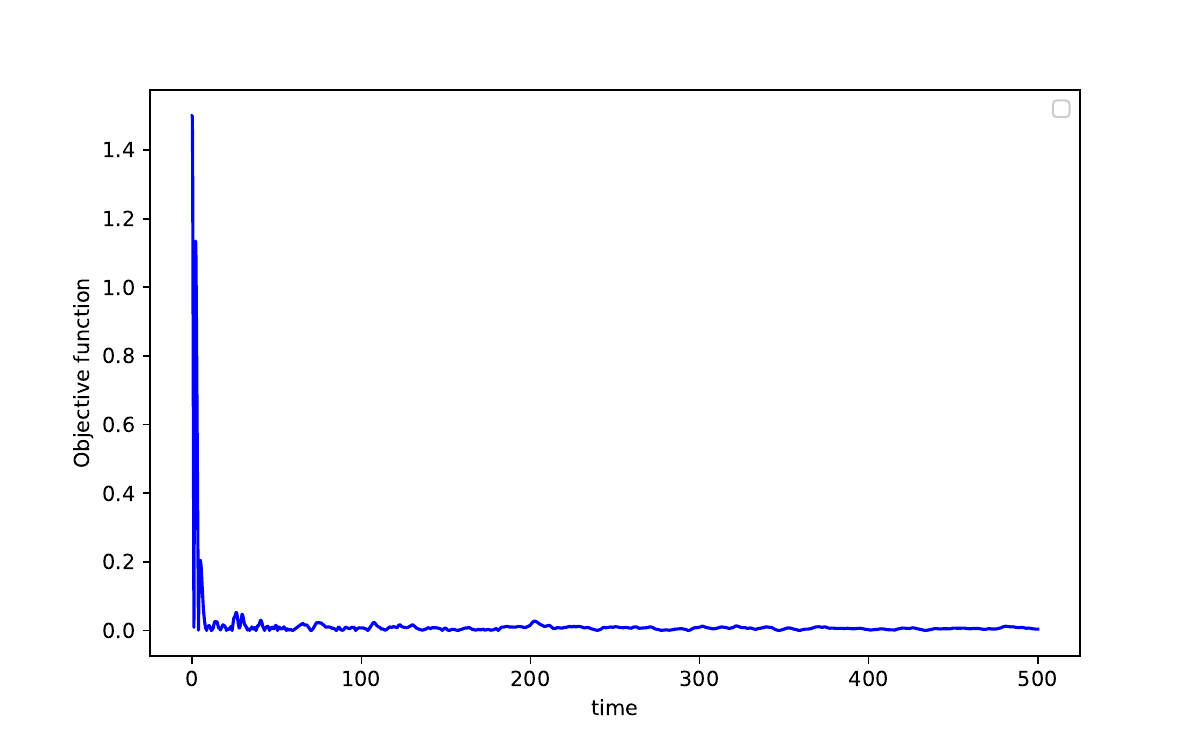}
\caption{ Objective function for \eqref{multi ou update} with $m=3$.}
\label{multi_ou_object1}
\end{minipage}
\begin{minipage}[t]{0.48\textwidth}
\centering
\includegraphics[width=6cm]{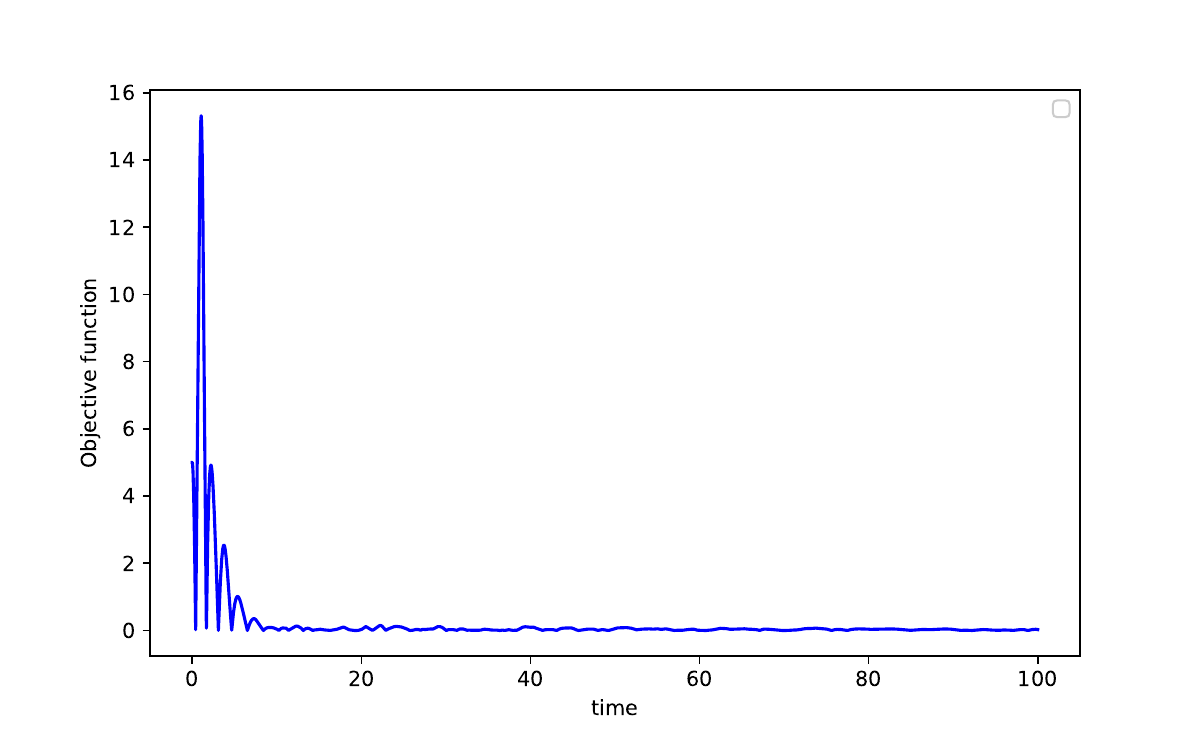}
\caption{Objective function for \eqref{multi ou update} with $m=10$.}
\label{multi_ou_object2}
\end{minipage}
\end{figure}

\subsection{Multi-Dimensional Correlated Ornstein–Uhlenbeck Process}

\hspace{1.4em} For the parameters $ \theta = (\mu, \sigma)$ with $\mu \in \mathbb{R}^{m}, \ \sigma \in \mathbb{R}^{m \times m}$, let the $m$-dimensional process $X_t^{\theta}$ satisfy
\beq
\label{multi correlated process}
dX_t^\theta = \left( \mu - X_t^\theta \right) dt + \sigma dW_t,
\eeq
where $W_t \in \mathbb{R}^m$. Let $X_t^{\theta, i}$ denote the $i$-th element of $X_t^\theta$ and define $\tilde X_t^\mu $ and $\tilde X_t^\sigma$ as the Jacobian matrices of $X_t^\theta$ with respect to $\mu$ and $\sigma$:
\bae
\tilde X_t^\mu &= \nabla_\mu X_t^\theta  \in \mathbb{R}^{m\times m}, \quad \tilde X_t^{\mu, i} = \nabla_\mu X_t^{\theta, i} \in \mathbb{R}^{m}, \\
\tilde X_t^\sigma &= \nabla_\sigma X_t^\theta  \in \mathbb{R}^{m\times m \times m}, \quad \tilde X_t^{\sigma, i} = \nabla_\sigma X_t^{\theta, i} \in \mathbb{R}^{m\times m}.
\eae
Noting that for $i \in \{1,2, \cdots, m\}$ 
$$
dX_t^{\theta, i} = \left( \mu_i - X_t^{\theta, i}\right) dt + \sum_j \sigma_{i, j} dW^j_t,
$$
now the algorithm \eqref{nonlinear update} becomes 
\bae
\label{multi correlated}
d\mu_t &= -4\alpha_t \left(|\bar X_t|^2 - 2m\right) \left(\sum\limits_{k=1}^m X^k_t \tilde X_t^{\mu, k} \right) dt \\
d\lambda_t &= -4\alpha_t \left(|\bar X_t|^2 - 2m \right) \left(\sum\limits_{k=1}^m X^k_t \tilde X_t^{\lambda, k} \right) dt \\
dX_t &= \left( \mu_t - X_t \right) dt + \sigma_t dW_t \\
d\bar X_t &= \left( \mu_t - \bar X_t \right) dt + \sigma_t d\bar W_t \\
d \tilde X_t^{\mu} &= \left( I_m - \tilde X_t^{\mu} \right) dt \\
d \tilde X_t^{\sigma, i} &= -\tilde X_t^{\sigma, i} dt + D_i\left(dW_t\right), \quad i \in \{1, \cdots, m\}
\eae
where $I_m$ is the $m \times m$ identity matrix and where $D_{i}(dW_t)$ is a $m \times m$ matrix with all elements equal to $0$ except $i$-th column being $dW_t$. We examine the algorithm's performance for dimensions $m =3, 10$. In Figures \ref{multi_correlated_mean_object1} and \ref{multi_correlated_mean_object2}, the objective function $J(\theta_t) \to 0$. 
\begin{figure}[htbp]
\centering
\begin{minipage}[t]{0.48\textwidth}
\centering
\includegraphics[width=6cm]{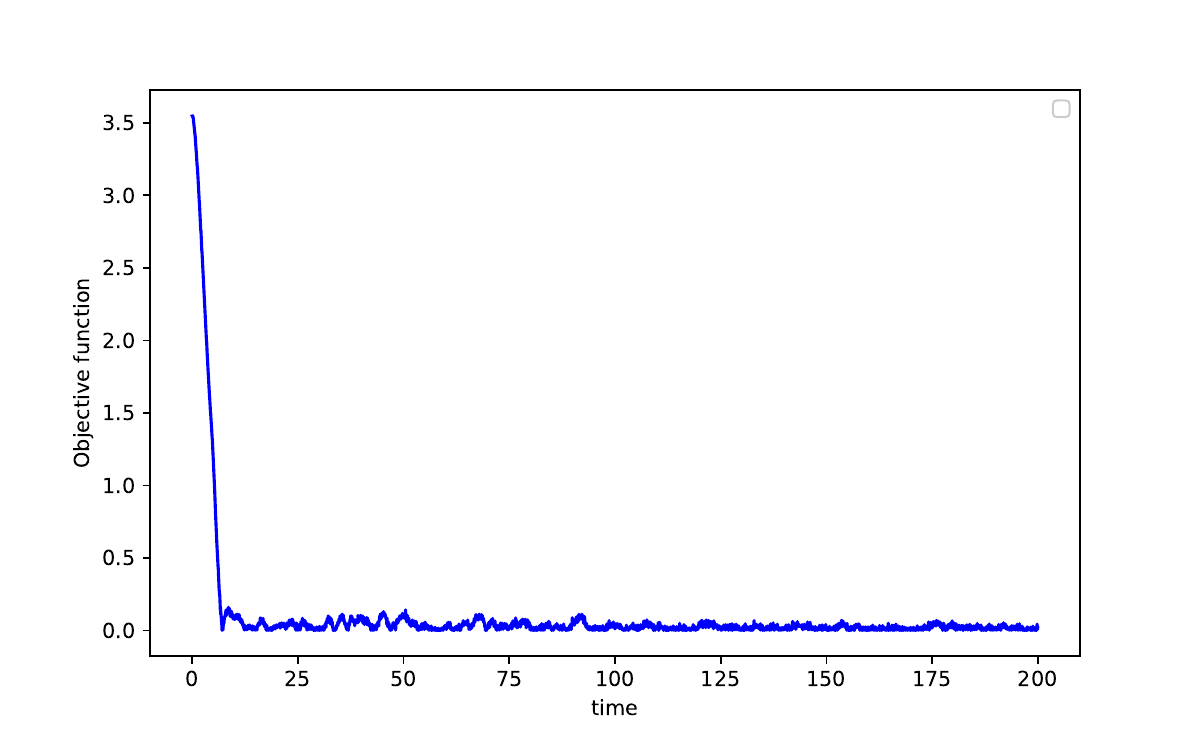}
\caption{Object function for \eqref{multi correlated} with $m =3$. }
\label{multi_correlated_mean_object1}
\end{minipage}
\begin{minipage}[t]{0.48\textwidth}
\centering
\includegraphics[width=6cm]{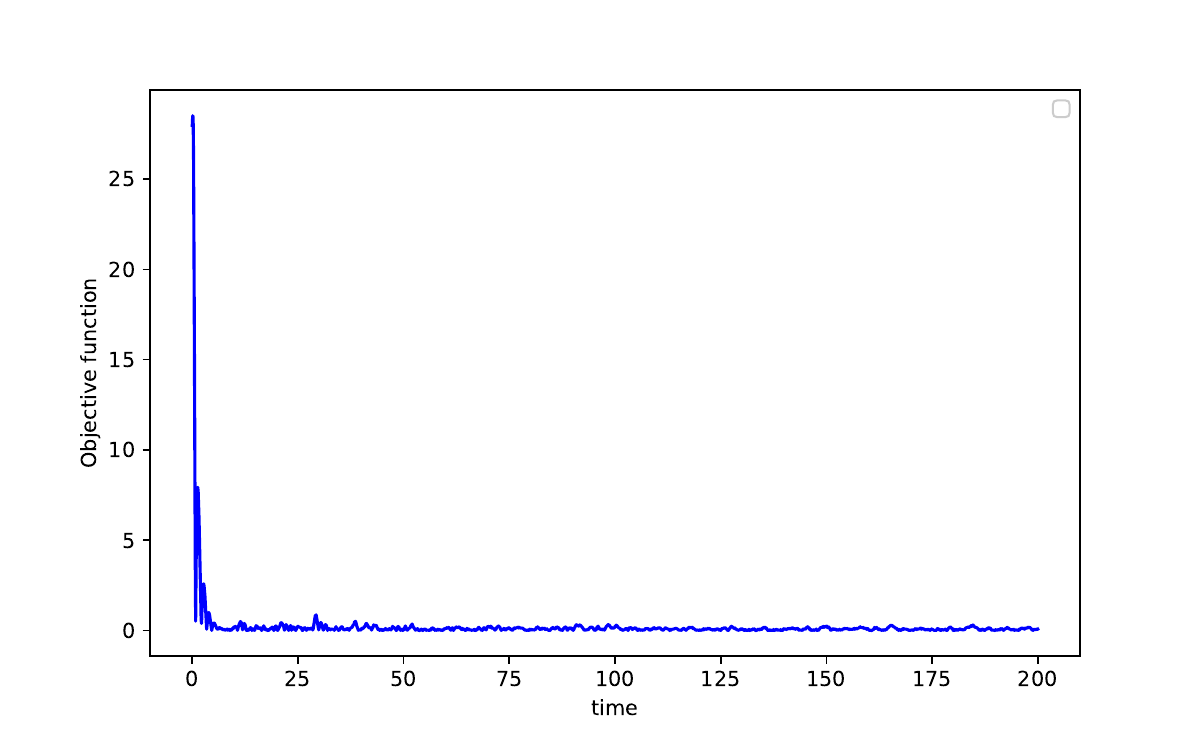}
\caption{ Objective function for \eqref{multi correlated} with $m =10$.}
\label{multi_correlated_mean_object2}
\end{minipage}
\end{figure}

\subsection{Multi-dimensional Nonlinear SDE}

\hspace{1.4em} In our next example, we optimize over the stationary distribution of a multi-dimensional nonlinear SDE:

\beq
\label{MV process}
dX_t^{\theta, i} = \left( \theta - \frac1N \sum_{j=1}^N X_t^{\theta, j} - (X_t^{\theta, i})^3 \right) dt + dW_t^i, \quad i = 1,2,\cdots, N,
\eeq
and now $N$ is the number of agents in the system \eqref{MV process} instead of mini-batch size as before. The objective function is 
\begin{eqnarray}
J(\theta) = \left( \frac1N \sum_{i=1}^N \e_{Y \sim \pi_\theta} Y_i^2 - 2 \right)^2.
\end{eqnarray}
The nonlinear SDE (\ref{MV process}) has a mean-field limit as $N \rightarrow \infty$. Thus, for large $N$, our algorithm could also be used to optimize over the mean-field limit equation (\cite{sznitman1991topics}) for (\ref{MV process}). The online algorithm for (\ref{MV process}) is
\bae
\label{MV}
d\theta_t &= -4\alpha_t \left(\frac1N \sum_{i=1}^N \left(\bar X^i_t\right)^2 - 2\right) \times \left( \frac1N \sum_{i=1}^N X^i_t \tilde X^i_t \right) dt \\
dX^i_t  &= \left( \theta_t - \frac1N \sum_{j=1}^N X^j_t - \left(X_t^i\right)^3 \right)dt + dW^i_t \\
d\tilde X^{i}_t &= \left(1 - \frac1N \sum_{j=1}^N \tilde X^{1,j}_t - 3\left(X_t^i\right)^2 \tilde X_t^i \right) dt \\
d\bar X^i_t  &= \left( \theta_t - \frac1N \sum_{j=1}^N \bar X^j_t - \left(\bar X_t^i\right)^3 \right)dt + d\bar W^i_t
\eae
for $i = 1, 2, \cdots, N$. We will select $N = 1,000$ for our numerical experiment. Therefore, this is an example of high-dimensional SDE model calibration where the dimension of the SDE is $N = 1,000$. Figure \ref{MV parameter} and \ref{MV object} shows the convergence of parameter and objective function.
\begin{figure}[htbp]
\centering
\begin{minipage}[t]{0.48\textwidth}
\centering
\includegraphics[width=6cm]{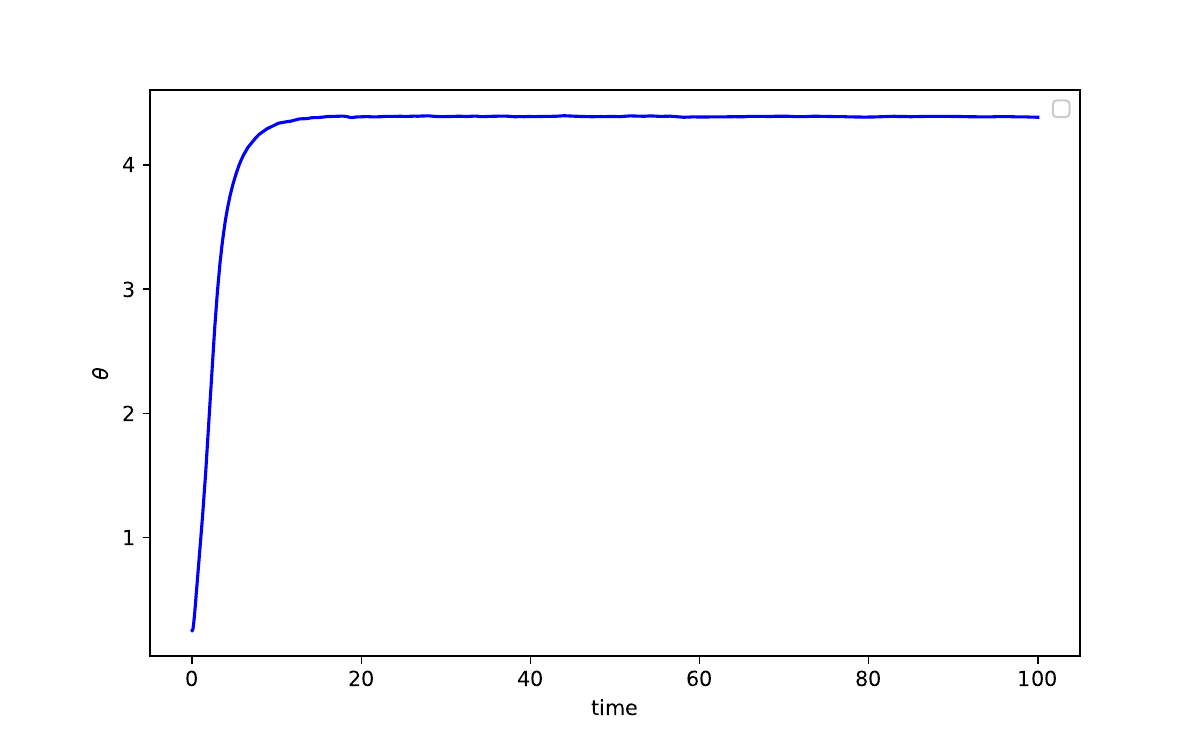}
\caption{ Parameter for algorithm \eqref{MV}.}
\label{MV parameter}
\end{minipage}
\begin{minipage}[t]{0.48\textwidth}
\centering
\includegraphics[width=6cm]{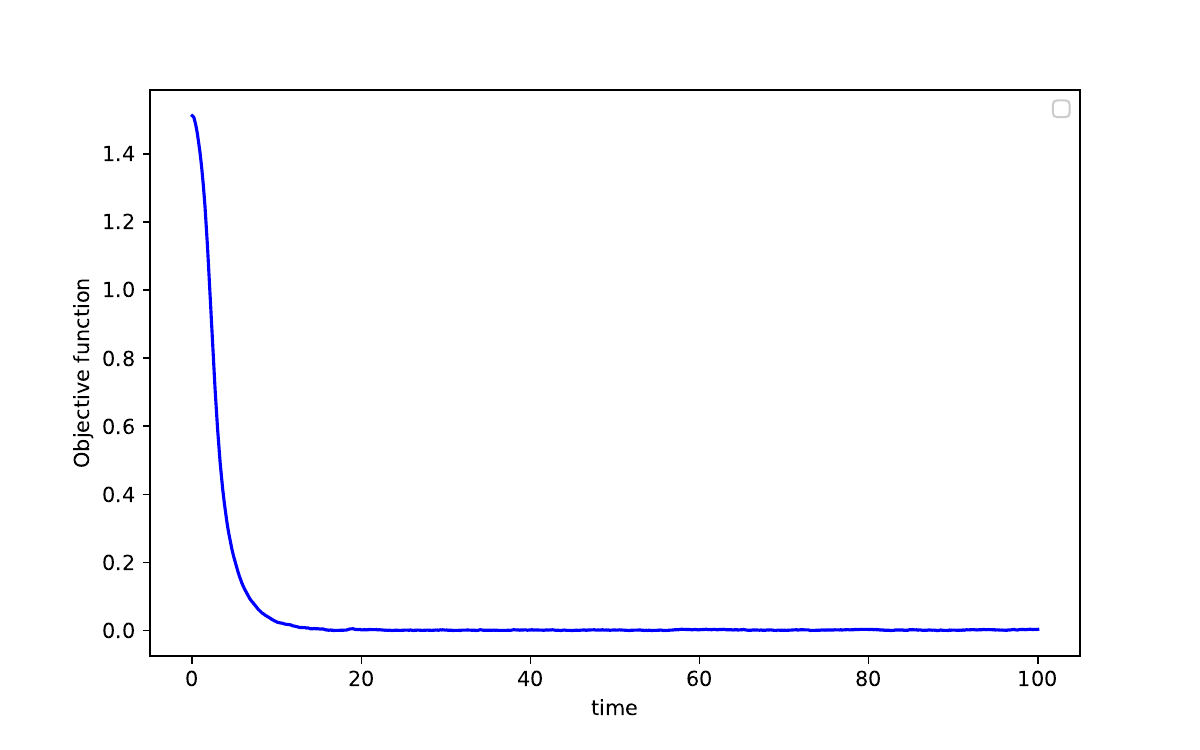}
\caption{Objective function for algorithm \eqref{MV}.}
\label{MV object}
\end{minipage}
\end{figure}

\subsection{Path-dependent SDE}

\hspace{1.4em} We consider the path-dependent SDE
\beq
\label{path}
dX_t^\theta = \left(\theta - X_t^\theta - \frac1t \int_0^t X_s^\theta ds \right) dt + dW_t,
\eeq
where $X_t^\theta, W_t \in \mathbb{R}$. Although path-dependent SDEs are not directly addressed by this article's convergence theory, this numerical example suggests that the online forward propagation
algorithm can also be applied to path-dependent stochastic processes.

For this numerical example, the objective function is 
\beq
J(\theta) = ( \e_{Y \sim \pi_\theta} Y - 2 )^2.
\eeq
The SDE \eqref{path} does not fit the problem described in 
\eqref{ergodic process} and \eqref{objective function}. However, our algorithm still can find the global optimum.

Now the online algorithm \eqref{nonlinear update} is:
\bae
\label{path dependent}
d\theta_t &= -4\alpha_t (\bar X_t - 2) \tilde X_t dt \\
dX_t  &= \left(\theta_t - X_t - \frac1t \int_0^t X_s ds \right) dt + dW_t \\
d\tilde X_t &= \left(1 - \tilde X_t - \frac1t \int_0^t \tilde X_s ds \right) dt \\
d\bar X_t  &= \left(\theta_t - \bar X_t - \frac1t \int_0^t \bar X_s ds \right)dt + d\bar W_t.
\eae

In Figure \ref{path dependent parameter}, the trained parameter converges. The objective function $J(\theta_t)$ is approximated using a time-average. In Figure \ref{path dependent object}, the objective function $J(\theta_t)$ converges to $0$ very quickly.

\begin{figure}[htbp]
\centering
\begin{minipage}[t]{0.48\textwidth}
\centering
\includegraphics[width=6cm]{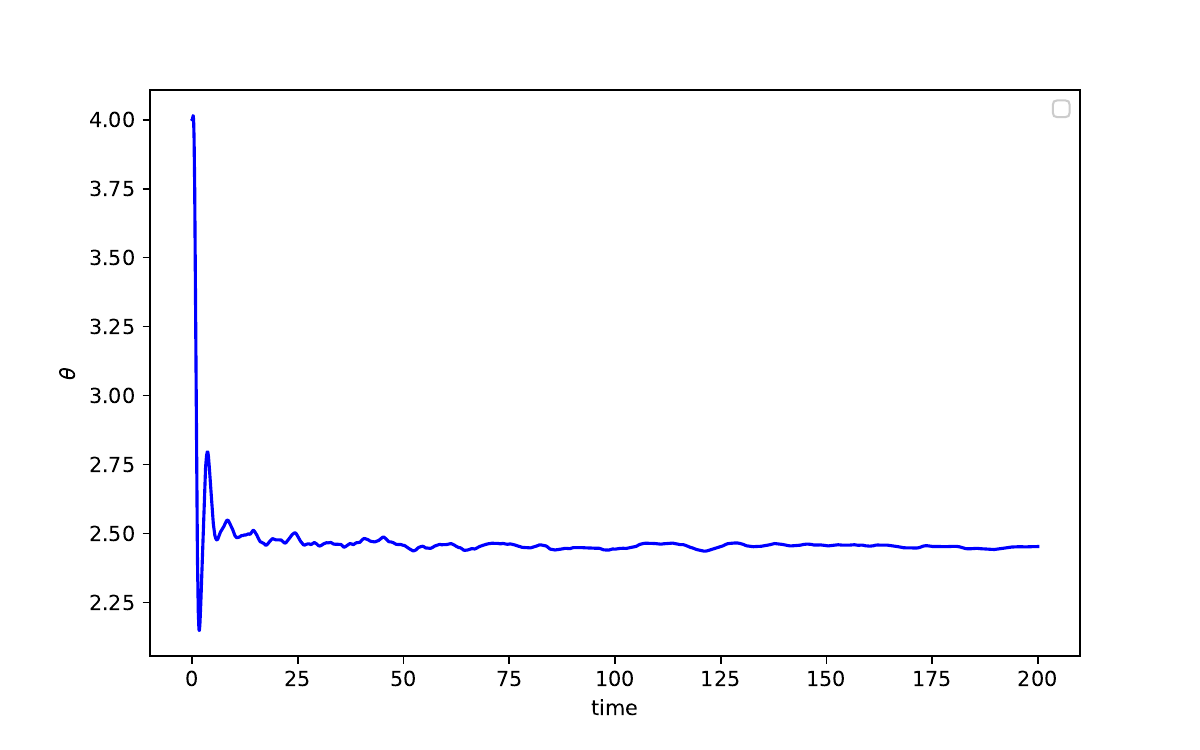}
\caption{Parameter for algorithm \eqref{path dependent}.}
\label{path dependent parameter}
\end{minipage}
\begin{minipage}[t]{0.48\textwidth}
\centering
\includegraphics[width=6cm]{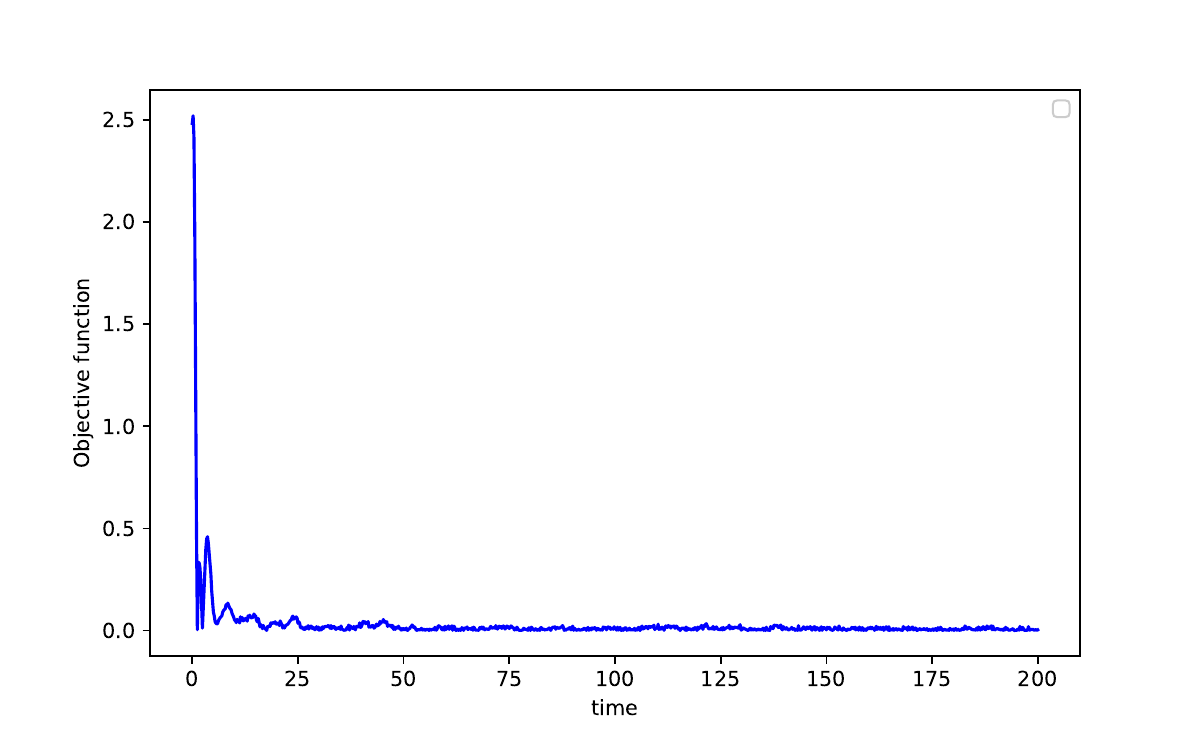}
\caption{Objective function for algorithm \eqref{path dependent}. }
\label{path dependent object}
\end{minipage}
\end{figure}

\subsection{Optimizing over the Auto-Covariance of the Ornstein-Uhlenbeck Process} \label{AutoCov}

\hspace{1.4em} As our final numerical example, consider the Ornstein-Uhlenbeck process
\beq
d X_t^{\theta} = (\mu - \lambda X_t^{\theta} ) dt + \sigma dW_t,
\eeq
where $\theta = (\mu, \lambda, \sigma)$. Define $\pi_{\theta}$ as the stationary distribution of $X_t^{\theta}$ and $\pi_{\theta, \tau}(dx, dx')$ as the stationary distribution of $(X_{t- \tau}^{\theta}, X_t^{\theta})$. The objective function is
\beq
J(\theta) = \left(\e_{Y \sim \pi_{\theta}} Y - 1\right)^2 + \left(\e_{Y \sim \pi_{\theta}} Y^2 - 2\right)^2 + \left(\e_{Y, Y' \sim \pi_{\theta, \tau}} YY' - 1.6\right)^2,
\eeq
where we will select $\tau = 0.1$ for our numerical experiment. 

The online algorithm is 
\bae
\label{auto}
d\mu_t &= -2\alpha_t \left[ \left(\bar X_t - 1 \right)\tilde X_t^1 + 2(\bar X^2_t - 2 )X_t \tilde X_t^1  + ( \bar X_{t-\tau} \bar X_t - 1.6) \left( \tilde X^1_{t-\tau} X_t + X_{t-\tau} \tilde X^1_t \right) \right]dt \\
d\lambda_t &= -2\alpha_t \left[ (\bar X_t - 1 )\tilde X_t^2 + 2(\bar X^2_t - 2 )X_t \tilde X_t^2  + ( \bar X_{t-\tau} \bar X_t - 1.6) \left( \tilde X^2_{t-\tau} X_t + X_{t-\tau} \tilde X^2_t \right) \right]dt \\
d\sigma_t &= -2\alpha_t \left[ (\bar X_t - 1 )\tilde X_t^3 + 2(\bar X^2_t - 2 )X_t \tilde X_t^3  + ( \bar X_{t-\tau} \bar X_t - 1.6) \left( \tilde X^3_{t-\tau} X_t + X_{t-\tau} \tilde X^3_t \right) \right]dt \\
dX_t  &= (\mu_t - \lambda_t X_t) dt + \sigma_t dW_t \\
d\tilde X^1_t &= ( 1 - \lambda_t \tilde X^1_t ) dt \\
d\tilde X^2_t &= (-X_t - \lambda_t \tilde X^2_t ) dt \\
d\tilde X^3_t &= - \lambda_t \tilde X^3_t dt + dW_t \\
d\bar X_t  &= (\mu_t - \lambda_t \bar X_t) dt + d\bar W_t.
\eae
Figures \ref{auto mu} - \ref{auto object} display the trained parameters and the objective function. The trained parameters have $\sim 0.1-0.3\%$ relative error compared to the global minimizers. The objective function $J(\theta_t)$ is computed from the exact formula
\beq
\label{auto objective}
J(\theta) =\left(\frac{\mu}{\lambda} -1 \right)^2 + \left( \left(\frac{\mu}{\lambda}\right)^2 + \frac{\sigma^2}{2\lambda} - 2 \right)^2 + \left( \left(\frac{\mu}{\lambda}\right)^2 + \frac{\sigma^2 e^{-\lambda\tau}}{2\lambda} - 1.6 \right)^2.
\eeq

\begin{figure}[htbp]
\centering
\begin{minipage}[t]{0.48\textwidth}
\centering
\includegraphics[width=6cm]{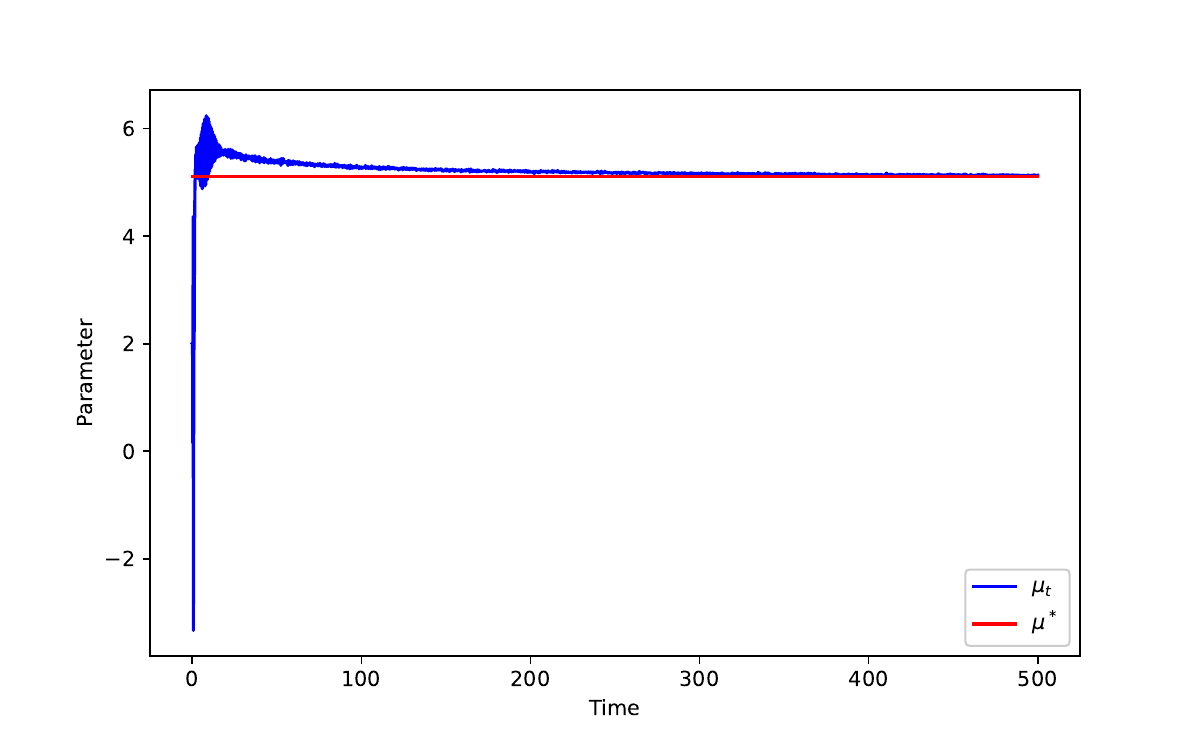}
\caption{$\mu_t$ evolution in \eqref{auto}.}
\label{auto mu}
\end{minipage}
\begin{minipage}[t]{0.48\textwidth}
\centering
\includegraphics[width=6cm]{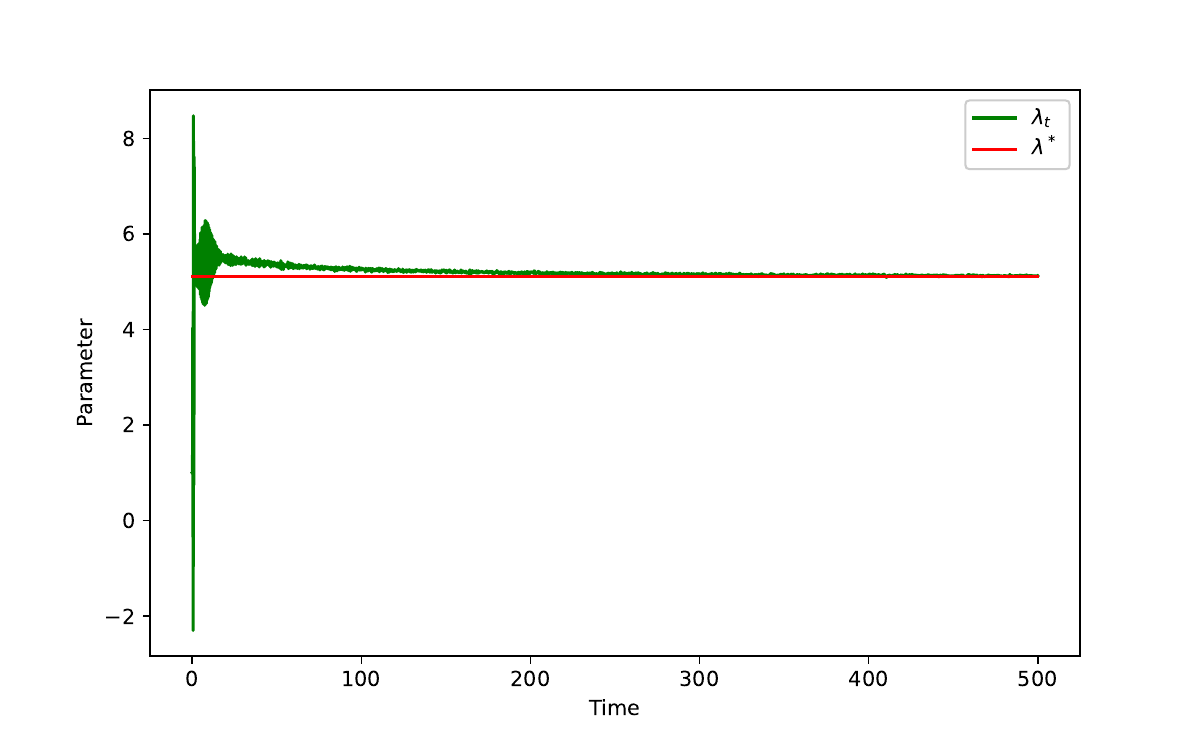}
\caption{$\lambda_t$ evolution in \eqref{auto}.}
\label{auto lambda}
\end{minipage}
\begin{minipage}[t]{0.48\textwidth}
\centering
\includegraphics[width=6cm]{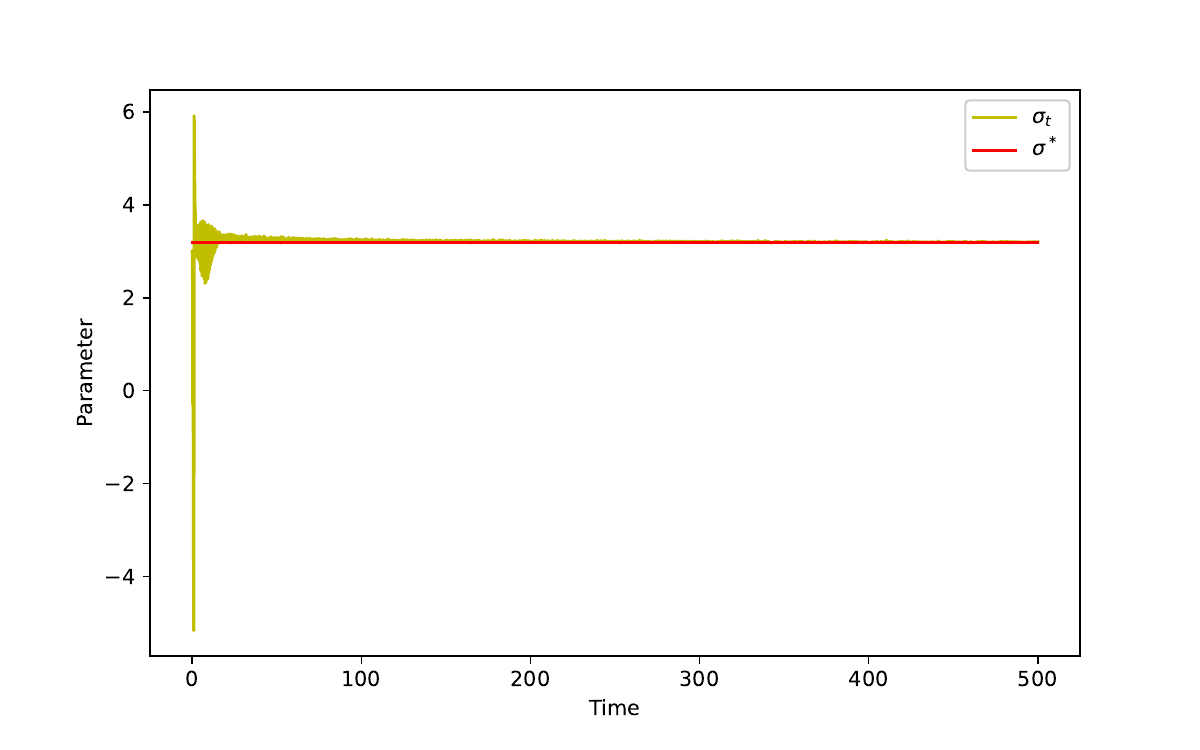}
\caption{$\sigma_t$ evolution in \eqref{auto}.}
\label{auto sigma}
\end{minipage}
\begin{minipage}[t]{0.48\textwidth}
\centering
\includegraphics[width=6cm]{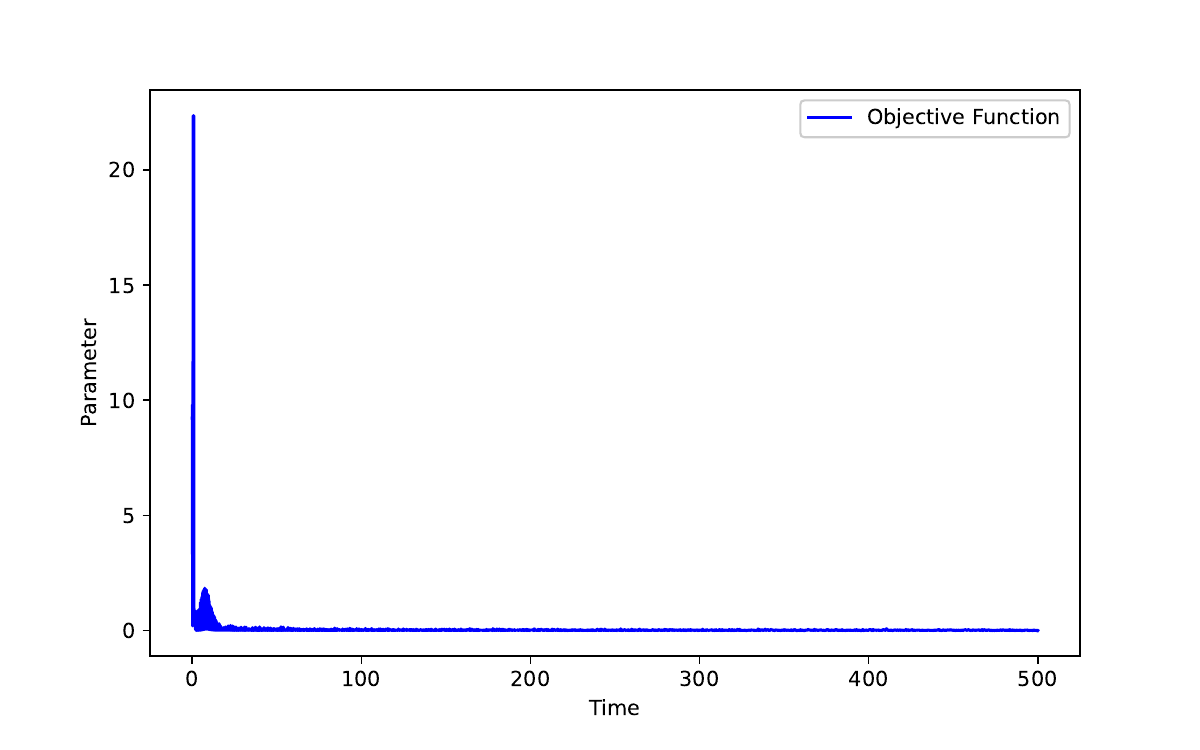}
\caption{Objective function for \eqref{auto}.}
\label{auto object}
\end{minipage}
\end{figure}

\subsection{Applications to Mathematical Finance} \label{finance application}

\hspace{1.4em} In this section, we discuss several potential applications of the forward propagation algorithm \eqref{nonlinear update} in mathematical finance. Our algorithm provides a new approach to estimate the parameters in SDE models in mathematical finance and financial econometrics \cite{ait2020maximum, cartea2016incorporating,  kitapbayev2018mean, lehalle2019incorporating, leung2016speculative, leung2015optimal, leung2015optimalbook, zhang2018mean}, including when the SDE is partially observed. Our algorithm is applicable for the calibration/estimation of SDE model parameters for long time series where ergodicity in the data is expected. In section \ref{partially observed model}, we discuss parameter estimation in partially-observed SDE models \cite{ait2020maximum, sharrock2022joint, surace2018online}, which are commonly used in financial econometrics \cite{bates1996jumps, christoffersen2009shape,  collin2002bonds, collin2004can, creal2015estimation, duffie2000transform, heston1993closed}.

In section \ref{LQR model}, we discuss the application of our algorithm to solving stochastic optimal control problems for long time horizons where the ergodic framework is suitable; stochastic optimal control is important in many areas of mathematical finance such as optimal order execution and portfolio optimization \cite{cartea2015algorithmic, hambly2021policy, pham2009continuous, yong1999stochastic, arapostathis2012ergodic, bardi2014linear}. High-dimensional stochastic optimal control problems are computationally intractable for traditional numerical methods. Although the optimal control satisfies a Hamilton-Jacobi-Bellman (HJB) equation, finite difference methods cannot solve high-dimensional PDEs. We demonstrate that our online optimization algorithm can efficiently solve high-dimensional stochastic optimal control problems (in the ergodic setting). In order to evaluate the accuracy of our algorithm for solving stochastic optimal control problems, we implement it for several high-dimensional stochastic linear quadratic regulator (LQR) problems \cite{fazel2018global, hambly2021policy, bertsekas2012dynamic, duncan1999adaptive, yong1999stochastic}. The LQR problem is selected since a closed-form solution is available (even in high dimensions) to evaluate the accuracy of our algorithm. (However, it should be highlighted that our online optimization algorithm can be used for the stochastic optimal control of any ergodic SDE, including nonlinear SDEs.) The online optimization algorithm learns a parametric control, either a linear function or a neural network (NN), to minimize the objective function. In both the linear and neural network cases, the algorithm can learn the optimal control. The optimal control functions appears in the drift of the SDE. In the case of the neural network optimal control, the SDE is therefore a ``neural network-SDE". Neural network-SDEs -- sometimes referred to as neural-SDEs -- are SDEs where the drift and/or volatility of the SDE is a neural network. Neural-SDEs have recently become of great interest in mathematical finance  \cite{Szpruch2, cohen2021arbitrage, cohen2022estimating, cohen2022hedging, Szpruch1, Szpruch3}.

The online optimization algorithm can also be used to solve multi-agent stochastic control problems -- e.g., mean-field games -- which is a widely-researched topic in mathematical finance  \cite{bardi2014linear, cao2022stationary, cardaliaguet2021ergodic, carmona2013mean, carmona2021convergence, carmona2021deep} in the ergodic setting. The finite multi-agent stochastic optimal control problem is typically computationally intractable since the corresponding HJB equation is very high-dimensional. It will be an $N\times d$ dimensions PDE, where $N$ is the number of agents and $d$ is the dimension of each agent's state (i.e., SDE) process. The limit mean-field game, which approximates the finite case, may be computationally tractable to solve. However, if the state space of each agent is high-dimensional (e.g., dimension $d > 4$), the limit mean-field game will also be computationally intractable since it will be a PDE in $d$ dimensions. In addition, the mean-field game limit may not be accurate for the finite-$N$ case if $N$ is not sufficiently large. Therefore, it is of interest to develop new methods for the computational solution of high-dimensional multi-agent stochastic optimal control problems in mathematical finance. As an example, we numerically implement the online optimization model for a simplified version of the multi-agent systemic risk model (\cite{carmona2013mean}) in Section \ref{ergodic MFC model}. There are $N$ agents where each agent is modeled by an SDE. As $N \rightarrow \infty$, the system converges to a mean-field game limit. In the numerical example, we use the online optimization algorithm to solve the the high-dimensional stochastic optimal control problem corresponding to a large number of $N$ SDEs ($N = 5,000$).

Finally, the online optimization algorithm can be used to train SDE models (including point process models) of limit order books \cite{Pakkanen} \cite{Pakkanen2} \cite{Cartlidge} \cite{Abergel} \cite{Kumar}. Order books involve large numbers of high-frequency events ($\sim 10^5 - 10^6$  events per day per stock) and high-dimensional dynamics (many price levels, each with limit order submissions and cancellations, as well as market orders, hidden orders, and transactions). The large amounts of high-frequency high-dimensional data for limit order books makes this a very promising application area for the online forward propagation algorithm, which is able to asymptotically optimize general classes of models over the \emph{entire history} of the order flow dataset (in contrast to standard methods can typically only optimize over much smaller sub-sequences).

\subsection{Optimizing parameters in partially-observed SDE models}\label{partially observed model}

\subsubsection{Two-dimensional Ornstein–Uhlenbeck Model}
\hspace{1.4em} In this section, we focus on the following partially observed two-dimensional Ornstein–Uhlenbeck process \cite{ait2020maximum} with parameters $\theta = \left(\alpha, \sigma_1, \sigma_2 \right)$:
\bae
dX_t &= \kappa^1\left( Y_t - X_t \right) dt + \sigma^1 dW_t^1 \\
dY_t &= \kappa^2\left( \alpha - Y_t \right) dt + \sigma^2 dW_t^2,
\eae
where the state process $X_t$ is observable and $Y_t$ is the latent (unobserved) process. As in Section \ref{numerical experiment}, we can estimate the parameters by calibrating the model to the moments of the stationary distribution. In our numerical example, the objective function is
\beq
J(\theta) = \left(\e_{Y \sim \pi_{\theta}} Y - 1\right)^2 + \left(\e_{Y \sim \pi_{\theta}} Y^2 - 2\right)^2 + \left(\e_{Y \sim \pi_\theta} Y^3 - 4 \right)^2.
\eeq
The algorithm \eqref{nonlinear update} becomes
\bae
\label{partially observed moment algo}
d \alpha_t &= -\alpha_t \left[ \left(\bar X_t - 1\right) \tilde X^1_t + 2\left(\bar X_t^2 - 2\right) X_t \tilde X_t^1 + 3\left(\bar X^3_t - 4\right) X^2_t \tilde X^1_t \right] dt \\
d \sigma^1_t &= -\alpha_t \left[ \left(\bar X_t - 1\right) \tilde X^2_t + 2\left(\bar X_t^2 - 2\right) X_t \tilde X_t^2 + 3\left(\bar X^3_t - 4\right) X^2_t \tilde X_t^2 \right] dt \\
d \sigma^2_t &= -\alpha_t \left[ \left(\bar X_t - 1\right) \tilde X^3_t + 2\left(\bar X_t^2 - 2\right) X_t \tilde X_t^3 + 3\left(\bar X^3_t - 4\right) X_t^2 \tilde X^3_t \right] dt \\
dX_t &= \kappa^1\left( Y_t - X_t \right) dt + \sigma^1_t dW_t^1 \\
dY_t &= \kappa^2\left( \alpha_t - Y_t \right) dt + \sigma^2_t dW_t^2 \\
d\tilde X^1_t &= \kappa^1\left( \tilde Y^1_t - \tilde X^1_t \right) dt \\
d\tilde Y^1_t &= \kappa^2\left( 1 - \tilde Y^1_t \right) dt \\
d\tilde X^2_t &= - \kappa^1 \tilde X^2_t dt + dW_t^1 \\
d\tilde X^3_t &= \kappa^1\left( \tilde Y^3_t - \tilde X^3_t \right) dt \\
d\tilde Y^3_t &= - \kappa^2 \tilde Y^3_t dt + dW_t^2 \\
d\bar X_t &= \kappa^1\left( \bar Y_t - \bar X_t \right) dt + \sigma^1_t d \bar W_t^1 \\
d\bar Y_t &= \kappa^2\left( \alpha_t - \bar Y_t \right) dt + \sigma^2_t d \bar W_t^2.
\eae
Figures \ref{partially observed moment parameter} and \ref{partially observed moment object} display the parameter convergence and the objective function.
\begin{figure}[htbp]
\centering
\begin{minipage}[t]{0.48\textwidth}
\centering
\includegraphics[width=6cm]{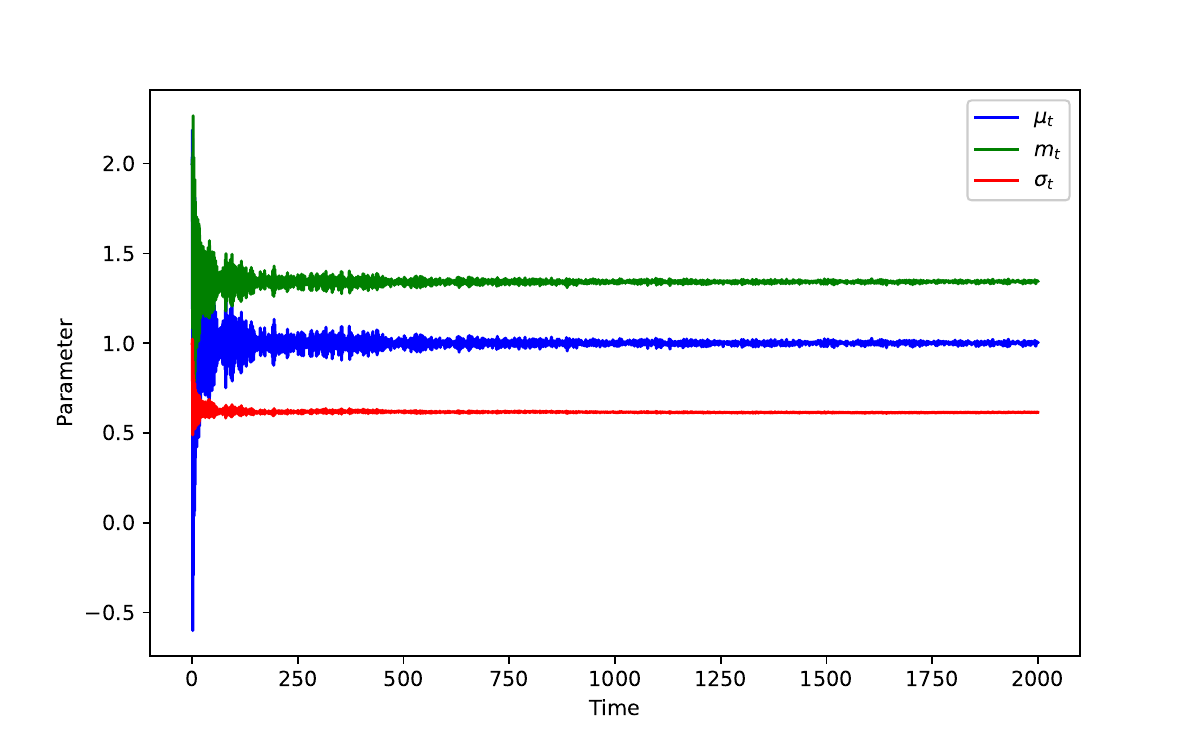}
\caption{Parameters for algorithm \eqref{partially observed moment algo}.}
\label{partially observed moment parameter}
\end{minipage}
\begin{minipage}[t]{0.48\textwidth}
\centering
\includegraphics[width=6cm]{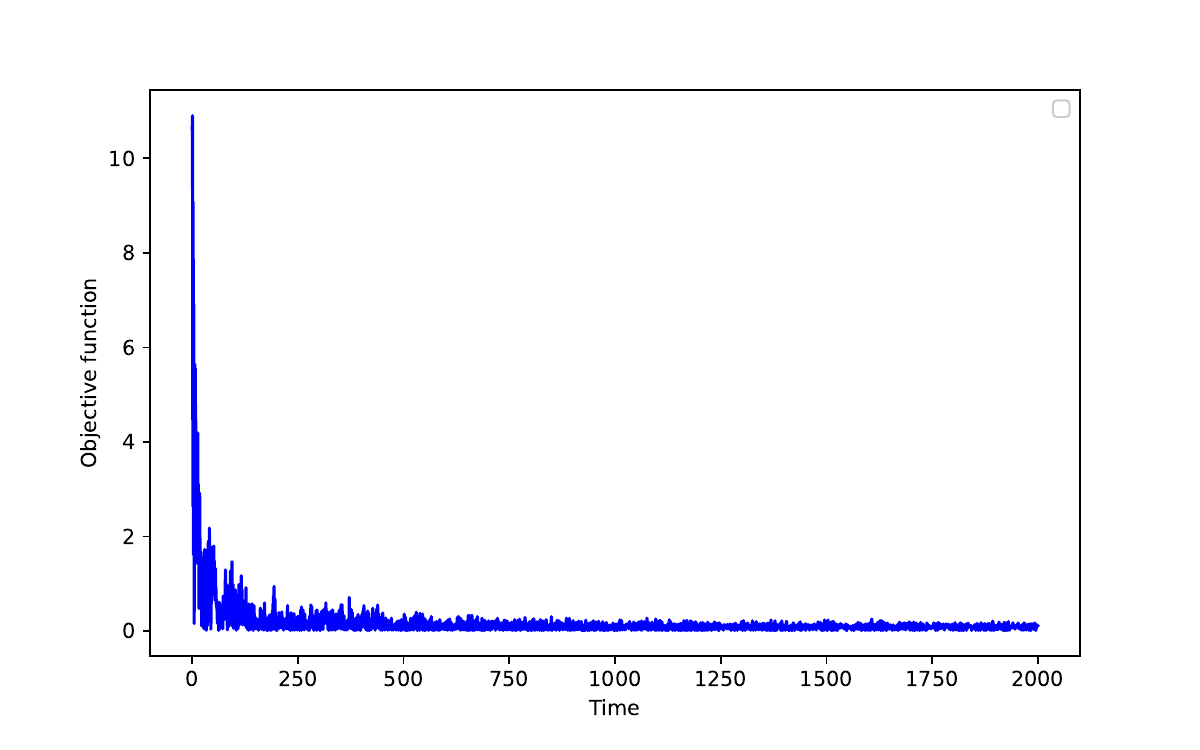}
\caption{Objective function for algorithm \eqref{partially observed moment algo}. }
\label{partially observed moment object}
\end{minipage}
\end{figure}

\subsection{Stochastic Optimal Control}\label{LQR model}

\hspace{1.4em} The online optimization algorithm can be used to solve stochastic optimal control problems, including high-dimensional problems for which traditional numerical methods (e.g., solving the HJB equation with finite difference methods) are computationally expensive or intractable. As a numerical example we consider the classic LQR problem \cite{anderson2007optimal, bertsekas2012dynamic, yong1999stochastic}, which itself has many financial applications such as optimal execution \cite{almgren2001optimal, cartea2015algorithmic, cartea2016incorporating, hambly2021policy}. Let $\{X_t\}_{t \geq 0}$ be the state process that satisfies the SDE
\beq
\label{LQR state}
dX_t = \left(A X_t + B U_t \right) dt + \sigma dW_t,
\eeq
where $X_0 = x_0, X_t \in \mathbb{R}^n$, matrix $A, \sigma \in \mathbb{R}^{n \times n},\ B \in \mathbb{R}^{n \times m}$, $\{W_t\}_{t \ge 0}$ is an $\mathbb{R}^n$-valued standard Wiener process, and $\{U_t\}_{t\ge 0} \in \mathbb{R}^m$ denotes the control. The objective is to learn a control process $u_\cdot$ to minimize the following ergodic cost functional for system \eqref{LQR state}:
\beq
\label{LQR object}
J(U_\cdot)=\lim\limits_{T \rightarrow \infty} \frac{1}{T} \int_0^T \left(X^{\mathrm{T}}_t Q X_t + U_t R U_t \right) dt,
\eeq
where $Q$ and $R$ are positive definite matrices. It is well-known that the optimal control is given by \cite{duncan1999adaptive}:
\beq
U = -R^{-1} B^\top K X,
\eeq
where $K$ is the unique solution of the following algebraic Riccati equation (ARE)
\beq
\label{ARE}
A^\top K + KA - KBR^{-1}B^{\top}K + Q = 0.
\eeq

In order to evaluate the accuracy of our algorithm for solving stochastic optimal control problems, we numerically implement it for several high-dimensional stochastic (LQR) problems. The LQR problem is selected since a closed-form solution is available (even in high dimensions) to evaluate the accuracy of our algorithm. We present a series of numerical examples where the online optimization algorithm learns parametric controls for various LQR problems. The parametric control is either a linear function or a neural network. 

\subsubsection{One-dimensional Linear Control}
\hspace{1.4em} As a first step, we implement the online optimization algorithm for the one-dimensional case with a linear control function. For simplicity, we assume that $A = -1, \ B = \sigma = Q = R = 1$ for \eqref{LQR state}:
\bae
dX_t^\theta &= \left(-X_t^\theta + \theta X_t^\theta \right) dt + dW_t, \\ 
J(\theta) &= \lim\limits_{T \to \infty} \frac1T \int_0^T \left(1+\theta^2\right) \left(X^\theta_t\right)^2 dt.
\eae
The coupled system \eqref{nonlinear update mini batch} becomes 
\bae
\label{1 dim lqr linear}
d\theta_t &= -\alpha_t \left[  \frac1N \sum_{i=1}^N \left( 2\theta_t \left(X^{(i)}_t\right)^2 + 2\left(1+\theta_t^2\right) X^{(i)}_t \tilde X^{(i)}_t \right) \right] dt, \\
dX^{(i)}_t &= (\theta_t - 1) X^{(i)}_t dt + dW^{(i)}_t, \\
d\tilde X^{(i)}_t &= ( X^{(i)}_t + (\theta_t - 1) \tilde X^{(i)}_t ) dt,
\eae
with $i = 1,2, \cdots, N$. Solving the ARE \eqref{ARE} yields the optimal control $\theta^* = -0.41421$. Figure \ref{linear_1_dim} shows that the parameter $\theta_t$ trained with the online optimization algorithm converges to $\theta^{\ast}$.
\begin{figure}[htbp]
\centering
\includegraphics[width=6cm]{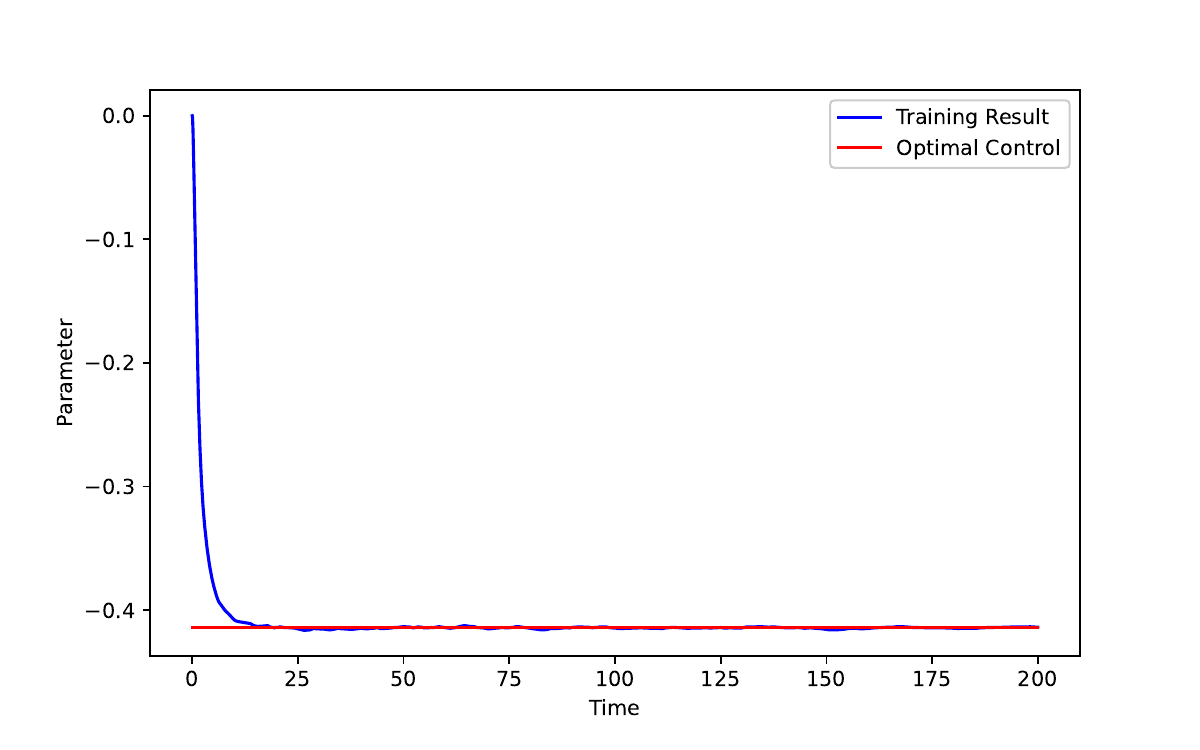}
\caption{ Parameter $\theta_t$ for algorithm \eqref{1 dim lqr linear}}
\label{linear_1_dim}
\end{figure}

\subsubsection{Multi-dimensional Linear Control}

\hspace{1.4em} We next solve a multi-dimensional LQR problem with a linear control function. For simplicity, we assume that $m = n, \ A = -I_n,\  B = \sigma = I_n$ in \eqref{LQR state} where $I_n$ is $n$ dimensional identity matrix. That is,
\bae
\label{multi-LQR}
dX_t^\theta &= \left(-X_t^\theta + \theta X_t^\theta \right) dt + dW_t, \\ 
J(\theta) &= \lim\limits_{T \to \infty} \frac1T \int_0^T \left( X^\theta_t \right)^\top \left(Q + \theta^\top R \theta \right) X^\theta_t dt,
\eae
where $\theta \in \mathbb{R}^{n \times n}$. Let $X_t^{\theta, i}$ denote the $i$-th elment of $X_t^\theta$ and define 
\beq
\tilde X_t^\theta = \nabla_\theta X_t^\theta, \quad \tilde X_t^{\theta, i} = \nabla_\theta X_t^{\theta, i} , \ \forall i \in \{1,2, \cdots, n\}.
\eeq
$\tilde X_t^\theta$ has dimensions $n \times n \times n$ and $\tilde X_t^{\theta, i}$ has dimensions $n \times n$. Note that when we are training over a mini-batch of size $N$, $\tilde X_t^\theta$ has dimensions $N \times n \times n \times n$. 

We first discuss the methods necessary for the computationally efficient simulation of the gradient $\nabla_\theta X_t^\theta$. The state process from \eqref{multi-LQR} satisfies
\bae
d X_t^{\theta, i} = \left( -X_t^{\theta, i} + \sum_{j=1}^n \theta_{i, j} X_t^{\theta, j} \right) dt + dW^i_t,
\eae
and therefore 
\bae
d \tilde X_t^{\theta, i} = \left( - \tilde X_t^{\theta, i} + \sum_{j=1}^n \theta_{i, j} \tilde X_t^{\theta, j} + D_i(X_t^\theta) \right) dt, \label{lqrXtilde}
\eae
where $D_i(X^\theta_t)$ is an $n \times n$ matrix whose elements are all zeros except for the $i$-th row, which has values $X^\theta_t$. The gradient of the objective function in \eqref{multi-LQR} is:
\bae
\nabla_\theta \left[ \left( X^\theta_t \right)^\top \left(Q + \theta^\top R \theta \right) X^\theta_t \right] &= \sum_{i, j} \nabla_\theta \left(\delta_{i, j} + \sum_{k=1}^n \theta_{k, i} \theta_{k, j} \right) X_t^{\theta, i} X_t^{\theta, j} + 2 \sum_{i, j} \nabla_\theta X_t^{\theta, i}\left(q_{i, j} + \theta_{:, i}^\top R \theta_{:, j} \right) X_t^{\theta, j} \\ 
&= \sum_{i, j} \left( (R\theta)_{:, j} \mathbbm{1}_{\{i = n\}} + (R\theta)_{:, i} \mathbbm{1}_{\{j = n\}} \right) X_t^{\theta, i} X_t^{\theta, j} + 2 \sum_{i, j} \tilde X_t^{\theta, i}\left(q_{i, j} + \theta_{:, i}^\top R \theta_{:, j} \right) X_t^{\theta, j}.
\eae

We now present the method for computationally efficient evaluation of the gradient process $\tilde X_t^{\theta}$. For notational simplicity, we only discuss below the case without using a mini-batch. The method can be easily extended to the mini-batch case though. Let $\odot$ indicate element-wise multiplication \emph{with broadcasting} \cite{mckinney2012python}. The RHS of (\ref{lqrXtilde}) can be evaluated using the following operations:
\begin{itemize}
\item To vectorize the term $\sum_{j=1}^n \theta_{i, j} \tilde X_t^{\theta, j}$ for $i \in \{1,2, \cdots, n\}$, we need to perform an inner-product of the \emph{second dimension} of the $n \times n \times 1 \times 1$ matrix $\theta$ with the $1 \times n \times n \times n$ matrix $\tilde X_t^\theta$.
\item Note that the final output $w$ is a tensor with dimensions $n \times n \times n$. 
\item To vectorize the term $D_i(X^\theta_t)$, consider the $n \times n \times n$ tensor $E$ where $E_{i,j,:} = \delta_{ij}$. Then $p = E \odot X_t^\theta$. 
\item Add $w$ and $p$. 
\end{itemize}
The objective function can be evaluated using a similar method: 
\begin{itemize}
\item First vectorize the $(R\theta)_{:, j} \mathbbm{1}_{\{i = n\}} + (R\theta)_{:, i} \mathbbm{1}_{\{j = n\}}$ to be an $n \times n \times n \times n$ matrix, which can be achieved by broadcasting, and denote the output as $D$. Similarly, the matrix multiplication of $X_t^{\theta, i} X_t^{\theta, j}$ produces a $n \times n \times 1 \times 1$ which we denote $X$. 
\item Perform an inner-product of the \emph{first and second dimension} of the $1 \times 1 \times n \times n$ matrix $D$ with the $n \times n \times 1 \times 1$ matrix $X$. Call this output $z$, which will be a tensor with dimensions $n \times n$. 
\item Perform the inner-product of the first dimension of the $n\times n \times n$ matrix $\tilde X_t^\theta$ and $n \times 1 \times 1$ matrix $F$, where $F_{i, :, :} = \sum_{j} \left(q_{i, j} + \theta_{:, i}^\top R \theta_{:, j} \right) X_t^{\theta, j}$. The output $q$ is a tensor with dimensions $n \times n$.
\item Add $z$ and $q$. 
\end{itemize}

Table \ref{linear res} presents the numerical results for the online optimization algorithm for learning the optimal control to the LQR problem. The online optimization algorithm performs well even in high dimensions.  Figure \ref{5dim} and Figure \ref{20dim} display the maximum and average errors for dimension $5$ and $20$ during training.  
\begin{table}[h]
\centering
\caption{Training Result for Linear Control}
\begin{tabular}{*{4}{|l|}}
\hline
\textbf{Dimension} &\textbf{Ave Error} &\textbf{Max Error} &\textbf{Cost Error} \\
\hline
1  & $0.1\%$ & $0.1\%$ & $0.01\%$ \\  \hline
5  & $0.2\%$ & $0.5\%$ & $0.05\%$\\  \hline
20 & $0.5\%$ & $1\%$ & $0.05\%$\\  \hline
\end{tabular}
\label{linear res}
\end{table}

\begin{figure}[htbp]
\centering
\begin{minipage}[t]{0.48\textwidth}
\centering
\includegraphics[width=6cm]{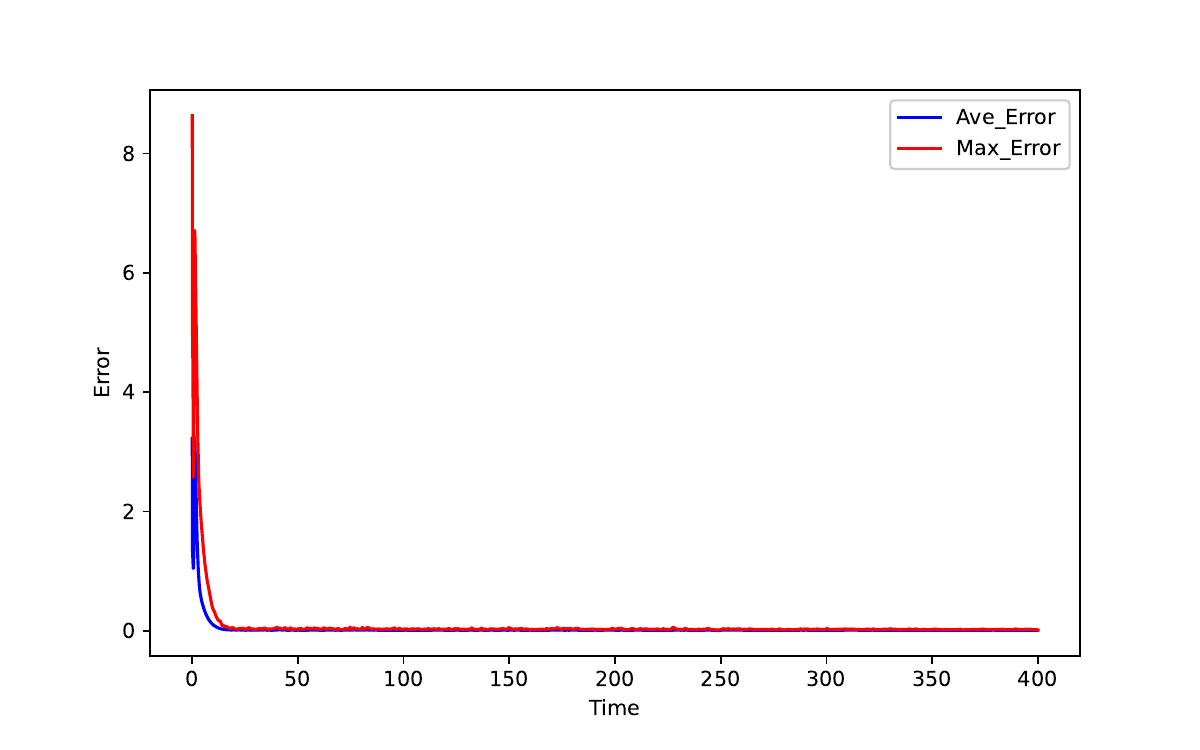}
\caption{Training result for dim = 5}
\label{5dim}
\end{minipage}
\begin{minipage}[t]{0.48\textwidth}
\centering
\includegraphics[width=6cm]{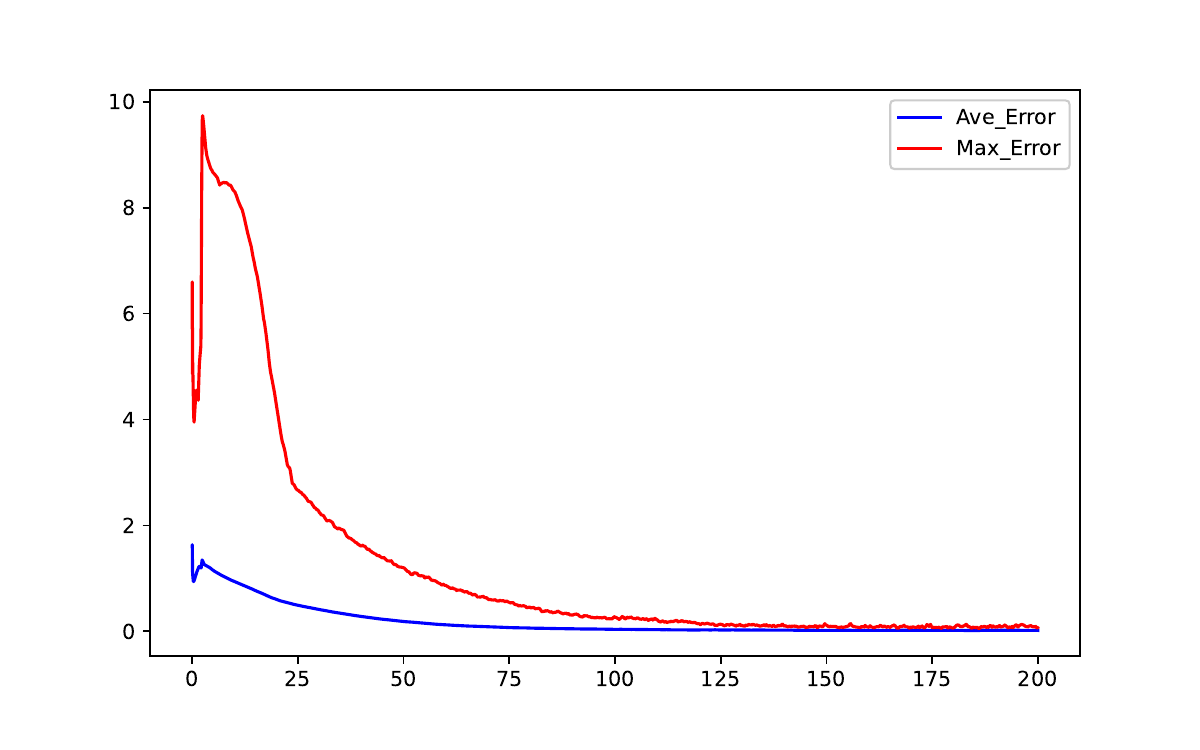}
\caption{Training result for dim = 20}
\label{20dim}
\end{minipage}
\end{figure}

The error metrics in Table \ref{linear res} are defined as:
\bae
\text{Ave Error} &= \frac{ \sum\limits_{i, j = 1}^n \left| \theta_{t, i, j} - \theta^*_{i, j} \right| }{ \sum\limits_{i, j = 1}^n \left| \theta^*_{i, j} \right|} \\ 
\text{Max Error} &= \frac{ \max\limits_{i, j \in \{1, 2, \cdots, n\}} \left| \theta_{t, i, j} - \theta^*_{i, j} \right|}{ \frac{1}{n^2} \sum\limits_{i, j = 1}^n \left| \theta^*_{i, j} \right|} \\ 
\text{Cost Error} &= \frac{\left| J(\theta_T) - J(\theta^*) \right|}{ \left| J(\theta^*) \right|},
\eae
where $\theta^*$ is the optimal control and $\theta_t$ is the parameter during training. $J(\theta_T)$ and $J(\theta^*)$ denote the objective function $J(\theta)$ in \eqref{multi-LQR} with the parameters $\theta_T$ and $\theta^*$, respectively.

\subsubsection{One-dimensional Neural Network Control}

\hspace{1.4em} We will now train a single-layer neural network control using the online optimization algorithm. The state process is:
\beq
dX^\theta_t = \left( -X^\theta_t + f_\theta\left(X_t^\theta\right) \right) dt + dW_t,
\eeq
where the control $f_\theta(\cdot)$ is a single-layer neural network
\beq
f_\theta(x) = \sum_{i=1}^m c^i \sigma\left(w^ix + b^i\right),
\eeq 
with parameters $\theta = (c^i, w^i, b_i)_{i=1}^m$. The objective function is
\beq
\label{NN LQR object}
J(\theta) = \lim\limits_{T \to \infty} \frac1T \int_0^T \left(X_t^\theta\right)^2 + \left(f_{\theta}(X_t^\theta)\right)^2 dt.
\eeq

Define the gradient of $X_t$ with respect to the parameters as:
\beq
\label{NN derivatives 1 dim}
\tilde X_t^w = \nabla_w X_t^\theta \in \mathbb{R}^m, \quad
\tilde X_t^b = \nabla_b X_t^\theta \in \mathbb{R}^m, \quad
\tilde X_t^c = \nabla_c X_t^\theta \in \mathbb{R}^m.
\eeq
The coupled system \eqref{nonlinear update} becomes 
\bae
dw_t &= -\alpha_t \left( 2 X_t \tilde X^{w}_t + 2f_{\theta_t}(X_t) \left( c_t \odot \sigma'(w_t X_t + b_t) X_t + f'_{\theta_t}(X_t) \tilde X^{w}_t \right) \right) dt, \\
db_t &= -\alpha_t \left( 2 X_t \tilde X^{b}_t + 2f_{\theta_t}(X_t) \left( c_t \odot \sigma'(W_t X_t + B_t) + f'_{\theta_t}(X_t) \tilde X^{b}_t \right) \right) dt, \\
dc_t &= -\alpha_t \left( 2 X_t \tilde X^{c}_t + 2f_{\theta_t}(X_t)
\left(\sigma(w_t X_t + b_t) + f'_{\theta_t}(X_t) \tilde X^{c}_t \right) \right) dt, \\
dX_t &= ( -X_t + f_{\theta_t}(X_t)) dt + dW_t, \\
d\tilde X^{w}_t &= ( - \tilde X^{w}_t + c_t \odot \sigma'(w_t X^i_t + b_t) X_t + f'_{\theta_t}(X_t) \tilde X^{w}_t ) dt, \\
d\tilde X^{b}_t &= ( - \tilde X^{b}_t + c_t \odot \sigma'(w_t X^i_t + b_t) + f'_{\theta_t}(X_t) \tilde X^{b}_t ) dt, \\
d\tilde X^{c}_t &= ( - \tilde X^{c}_t + \sigma(w_t X_t + b_t) + f'_{\theta_t}(X_t) \tilde X^{c}_t ) dt, \\
d\bar X_t &= ( - \bar X_t + f_{\theta_t}(\bar X_t)) dt + d\bar W_t.
\eae

\begin{figure}[htbp]
\centering

\begin{minipage}[t]{0.48\textwidth}
\centering
\includegraphics[width=6cm]{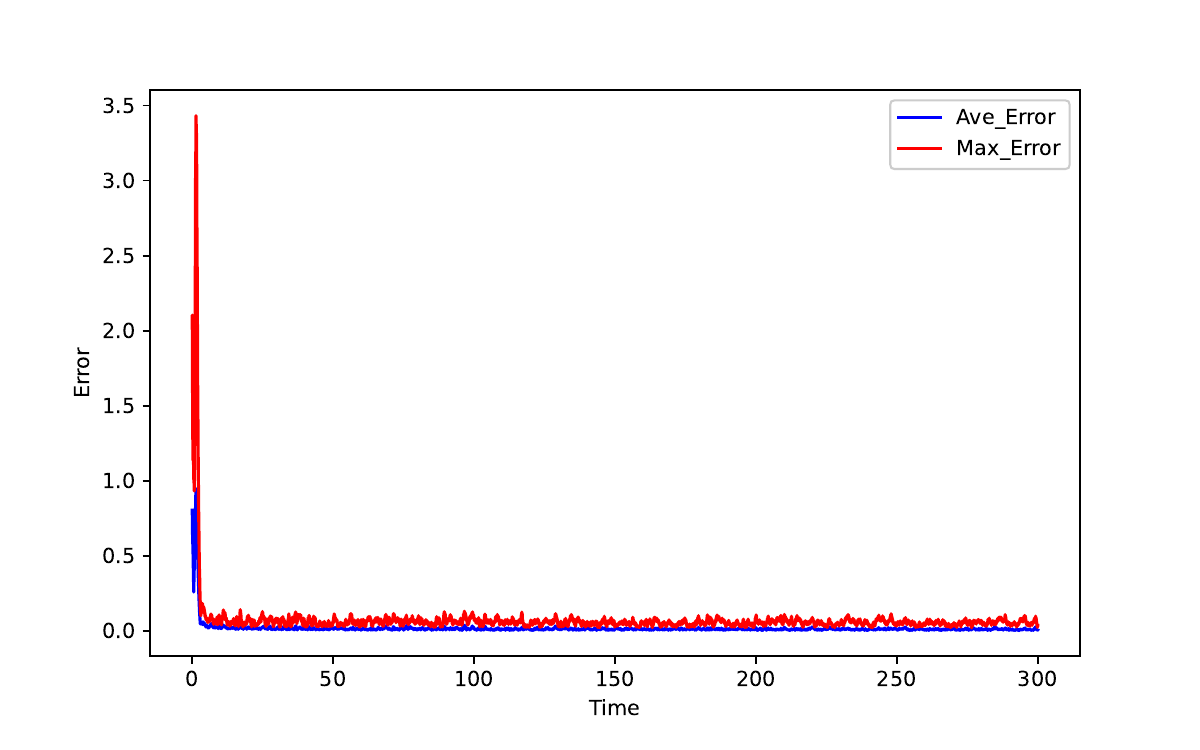}
\caption{Training result for dim = 1}
\label{Before NN}
\end{minipage}
\begin{minipage}[t]{0.48\textwidth}
\centering
\includegraphics[width=6cm]{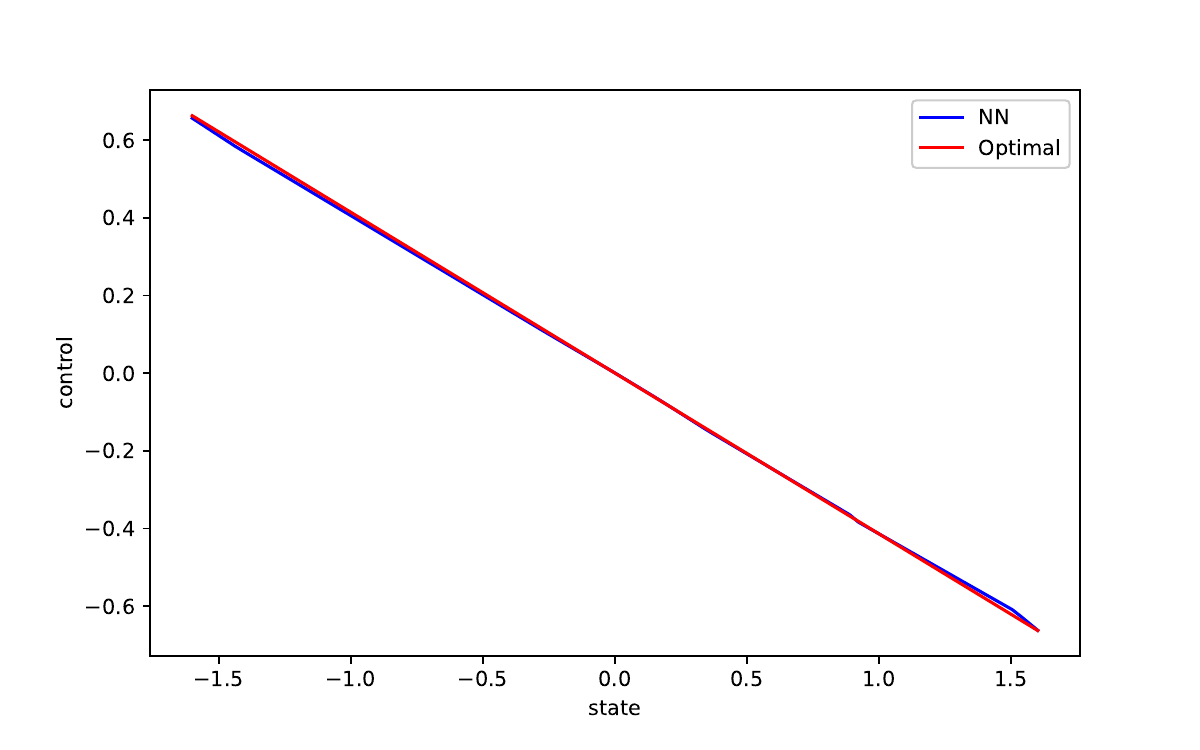}
\caption{Neural Network output after training}
\label{after NN}
\end{minipage}
\end{figure}

The training result for $1$ dimensional LQR with network network control 
is presented in Figure \ref{Before NN}, Figure \ref{after NN}, and Table \ref{NN res}. The error metrics are defined as:\footnote{Here the norm $\left\| \cdot \right\|$ denotes the $L^1$ norm, i.e. for a vector $Y = \left(y_1, y_2, \cdots, y_d \right) \in \mathbb{R}^d$,
$\left\| Y \right\| = \sum_{i=1}^d \left| y_i \right|$.}  
\bae
\text{Ave Error} &= \frac{\sum\limits_{i=1}^n\left\| f_{\theta_t}(X^i) - \theta^* X^i \right\|}{ \sum\limits_{i=1}^n \left\|\theta^* X^i\right\|}
\\
\text{Max Error} &= \frac{\max\limits_{i \in \left\{1,2,\cdots, n\right\}} \left\| f_{\theta_t}(X^i) - \theta^* X^i \right\|}{ \sum\limits_{i=1}^n \left\|\theta^* X^i\right\|} \\ 
\text{Cost Error} &= \frac{\left| J(\theta_T) - J(\theta^*) \right|}{ \left| C^* \right|},
\eae
where $\theta^*$ is the optimal control and $\theta_t$ is the trained parameter. $J(\theta_T)$ and $J(\theta^*)$ denote the objective function $J(\theta)$ in \eqref{NN LQR object} with the parameters $\theta_T$ and $\theta^*$, respectively. The points $\{X^i\}_{i=1}^n$ are uniformly sampled from $[-L, L]$ with $L$ chosen such that $[-L, L]$ contains the optimally controlled process $99\%$ of the time.

\subsubsection{Multi-dimensional Neural Network Control}

\hspace{1.4em} We now optimize a single-layer neural network control for a high-dimensional state process:
\bae
dX^\theta_t &= \left( -X^\theta_t + f_\theta\left(X_t^\theta\right) \right) dt + dW_t, \\ 
J(\theta) &= \lim\limits_{T \to \infty} \frac1T \int_0^T \left(X_t^\theta\right)^\top Q X_t^\theta + \left(f_{\theta}\left(X_t^\theta\right)\right)^\top R f_{\theta}\left(X_t^\theta\right) dt,
\eae
where $X_t^\theta \in \mathbb{R}^n$ and the single-layer neural network with $m$ hidden units is:
\beq
f_\theta(x) = c \sigma\left(wx + b\right),
\eeq 
where $w \in \mathbb{R}^{m \times n}, \ b \in \mathbb{R}^m,$ and $\ c \in \mathbb{R}^{n \times m}$. As in \eqref{NN derivatives 1 dim}, define 
\bae
\label{NN derivatives n dim}
\tilde X_t^w &= \nabla_w X_t^\theta \in \mathbb{R}^{n \times m \times n}, \quad \tilde X_t^{w, i} = \nabla_w X_t^{\theta, i} \in \mathbb{R}^{m \times n}, \\
\tilde X_t^b &= \nabla_b X_t^\theta \in \mathbb{R}^{n \times m}, \quad \tilde X_t^{b, i} = \nabla_b X_t^{\theta, i} \in \mathbb{R}^{m}, \\
\tilde X_t^c &= \nabla_c X_t^\theta \in \mathbb{R}^{n \times n \times m}, \quad \tilde X_t^{c, i} = \nabla_c X_t^{\theta, i} \in \mathbb{R}^{n \times m},
\eae
for $i = 1,2, \cdots, n$.

The online algorithm \eqref{nonlinear update} becomes:
\bae
\label{multi lqr nn dynamic}
dw_t &= -\alpha_t \left[ \nabla_w \left( f_{\theta_t}(X_t)^\top R f_{\theta_t}(X_t) \right) + \sum_{i=1}^n \frac{\partial}{\partial x_i} \left( (X_t)^\top Q X_t + f_{\theta_t}(X_t)^\top R f_{\theta_t}(X_t) \right) \tilde X^{w, i}_t \right] dt \\
db_t &= -\alpha_t \left[ \nabla_b \left( f_{\theta_t}(X_t)^\top R f_{\theta_t}(X_t) \right) + \sum_{i=1}^n \frac{\partial}{\partial x_i} \left( (X_t)^\top Q X_t + f_{\theta_t}(X_t)^\top R f_{\theta_t}(X_t) \right) \tilde X^{b, i}_t \right] dt \\
dc_t &= -\alpha_t \left[ \nabla_c \left( f_{\theta_t}(X_t)^\top R f_{\theta_t}(X_t) \right) + \sum_{i=1}^n \frac{\partial}{\partial x_i} \left( (X_t)^\top Q X_t + f_{\theta_t}(X_t)^\top R f_{\theta_t}(X_t) \right) \tilde X^{c, i}_t \right] dt \\
dX_t &= ( -X_t + f_{\theta_t}(X_t)) dt + dW_t, \\
d\tilde X_t^{w, i} &= \left( -\tilde X_t^{w, i} + \sum_{k} c_{t, i, k} \sigma'\left( w_t X_t + b_t\right)_k \left( \sum_{\ell} w_{t, k, \ell} \tilde X_t^{w, \ell} \right) + (c_{t, i, :})^\top \odot \sigma'\left(w_t X_t + b_t\right) \left(X_t\right)^\top \right) dt \\
d\tilde X_t^{b, i} &= \left( -\tilde X_t^{b, i} + \sum_{k} c_{t, i, k} \sigma'\left( w_t X_t + b_t\right)_k \left( \sum_{\ell} w_{t, k, \ell} \tilde X_t^{b, \ell} \right) + (c_{t, i, :})^\top \odot \sigma'\left(w_t X_t + b_t\right) \right) dt\\
d\tilde X_t^{c, i} &= \left( -\tilde X_t^{c, i} + \sum_{k} c_{t, i, k} \sigma'\left( w_t X_t + b_t\right)_k \left( \sum_{\ell} w_{t, k, \ell} \tilde X_t^{c, \ell} \right) + D_{i} \left( \sigma\left(w_t X_t + b_t\right)\right) \right) dt  \\
d\bar X_t &= ( - \bar X_t + f_{\theta_t}(\bar X_t)) dt + d\bar W_t
\eae
for $i = 1, 2, \cdots, N$. In \eqref{multi lqr nn dynamic}, $C_{t, i, :} \in \mathbb{R}^n$ denotes the $i$-th row of the matrix $C_t$ and $D_i(X_t)$ is an $n \times n$ matrix whose elements are all zeros except for the $i$-th row, which has the vector value $\sigma\left(w_t X_t + b_t \right)$. 

The numerical results for training the neural network SDE control with the online optimization algorithm are presented in Figure \ref{NN 5 dim}, Figure \ref{NN 20 dim}, and Table \ref{NN res}. In general, the trained neural network control performs well, even in high dimensions. 

\begin{figure}[htbp]
\centering
\begin{minipage}[t]{0.48\textwidth}
\centering
\includegraphics[width=6cm]{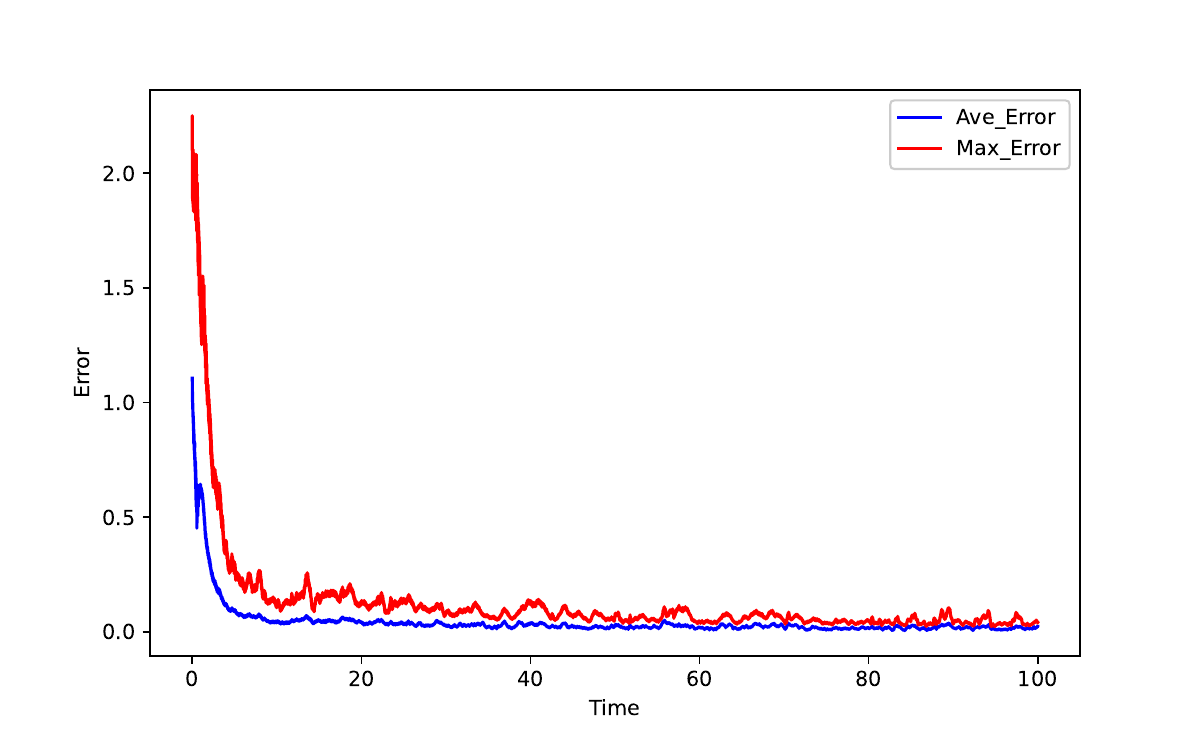}
\caption{Training result for dim = 5}
\label{NN 5 dim}
\end{minipage}
\begin{minipage}[t]{0.48\textwidth}
\centering
\includegraphics[width=6cm]{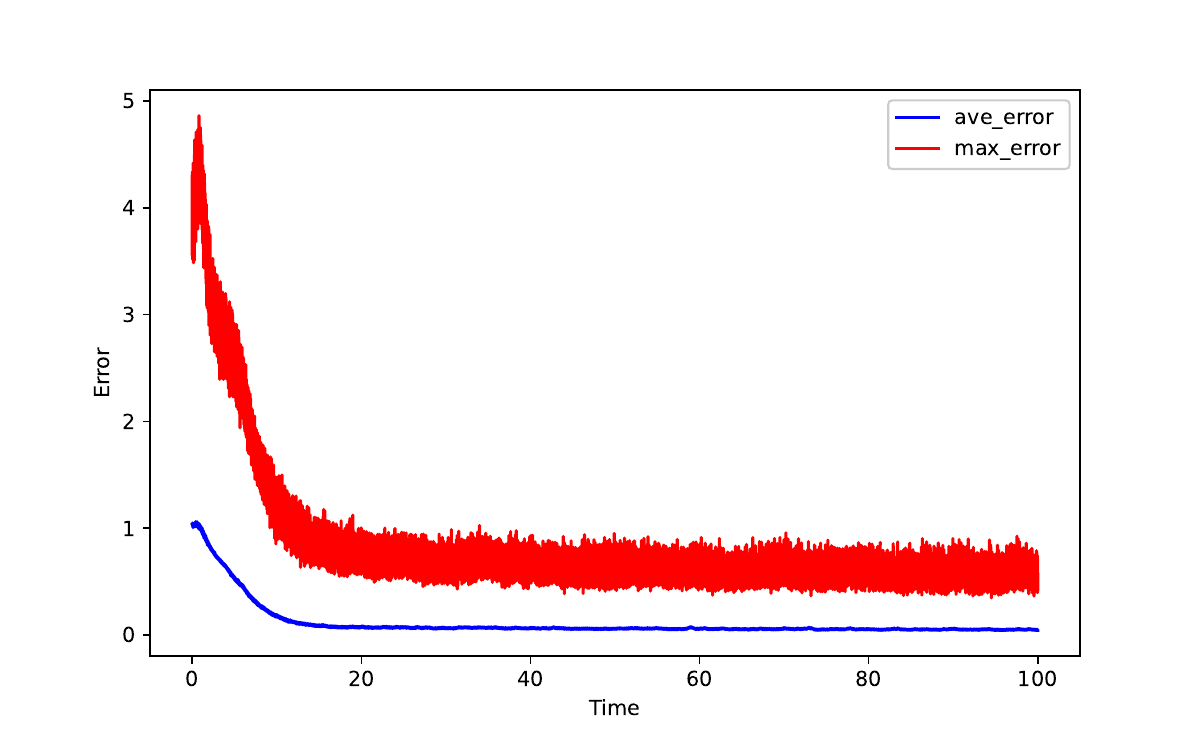}
\caption{Training result for dim = 20}
\label{NN 20 dim}
\end{minipage}
\end{figure}

\begin{table}[h]
\centering
\caption{Training Result for NN Control}
\begin{tabular}{*{4}{|l|}}
\hline
\textbf{Dimension} &\textbf{Ave Error} &\textbf{Max Error} &\textbf{Cost Error} \\
\hline
1  & $0.1\%$ & $0.1\%$ & $0.01\%$ \\  \hline
5  & $0.6\%$ & $1\%$ & $0.02\%$\\  \hline
20 & $1\%$ & $10\%$ & $0.1\%$\\  \hline
\end{tabular}
\label{NN res}
\end{table}

\subsection{Applications to Multi-Agent and Mean-Field System Control} \label{ergodic MFC model}

\hspace{1.4em} Finally, the online optimization algorithm can be used to solve multi-agent stochastic control problems -- e.g., mean-field control and mean-field games, which are important topics in mathematical finance  \cite{bardi2014linear, cao2022stationary, cardaliaguet2021ergodic, carmona2013mean, carmona2021convergence, carmona2021deep} -- in the ergodic setting. As an example, we numerically implement the online optimization model for a simplified version of the multi-agent systemic risk model (\cite{carmona2013mean}) in Section \ref{ergodic MFC model}. There are $N$ agents where each agent is modeled by an SDE. As $N \rightarrow \infty$, the system converges to a mean-field game limit. In the numerical example, we use the online optimization algorithm to solve the the high-dimensional stochastic optimal control problem corresponding to a large number of $N$ SDEs ($N = 5,000$). 

We consider the following multi-agent control problem, which is a simplified version of the systemic risk model in \cite{carmona2013mean}:
\beq
dX_t^{\theta, i} = \left[ a\left( \frac1N \sum_{j=1}^N X_t^{\theta, j} - X_t^{\theta, i} \right) + f_\theta(X_t^{\theta, i}) \right] dt + \sigma dW^i_t 
\eeq
for $i = 1, 2, \cdots N$ with the objective function 
\beq
J^N(\theta) = \frac1N \sum_{i=1}^N \lim_{T\to \infty} \frac1T \int_0^T  \left(X_t^{\theta, i}\right)^2 + f^2\left(X_t^{\theta, i}\right) dt.
\eeq
This mean-field system has the following mean-field limit:
\bae
\label{MFC problem}
dX_t^\theta &= a\left[ \left( EX_t^\theta - X_t^\theta \right) + f_\theta \left(X_t^\theta\right) \right] dt + \sigma dW_t \\
J(\theta) &= \lim_{T \to \infty} \frac1T \int_0^T \left(X_t^\theta\right)^2 + f^2_\theta(X_t^\theta) dt.
\eae

We describe how the online optimization algorithm can train both linear and neural network controls for this mean-field system. The algorithm \eqref{nonlinear update} to train the linear model becomes:
\bae
\label{risk linear}
d\theta_t &= -\alpha_t \left[ \frac1N \sum_{i=1}^N \left( 2 \theta_t \left( X_t^i \right)^2 + 2 \left(1+\theta_t^2\right) X_t^i \tilde X_t^i \right) \right] dt \\
dX^i_t &= \left[ a \left( \frac1N \sum_j X_t^{j} - X^{i}_t \right) + \theta_t X^i_t \right] dt + dW^i_t \\
d\tilde X^i_t &= \left[ a \left( \frac1N \sum_j \tilde X_t^j - \tilde X^i_t \right) + X_t^i + \theta_t \tilde X_t^i \right] dt. 
\eae
The training result for the linear control is displayed in Figure \ref{risk}.

We next train a neural network for the control function $f_{\theta}(x) = c \sigma(wx + b)$ where $\theta = (c, w, b)$. The online optimization algorithm becomes:
\bae
\label{risk NN}
dw_t &= -\alpha_t \left[ \frac1N \sum_{i=1}^N \left( 2 X^i_t \tilde X^{w, i}_t + 2f_{\theta_t}(X^i_t) 
\left( C_t \odot \sigma'(W_t X^i_t + B_t) X^i_t + f'_{\theta_t}(X^i_t) \tilde X^{w, i}_t \right) \right) \right] dt \\
db_t &= -\alpha_t \left[ \frac1N \sum_{i=1}^N \left( 2 X^i_t \tilde X^{b, i}_t + 2f_{\theta_t}(X^i_t) 
\left( C_t \odot \sigma'(W_t X^i_t + B_t) + f'_{\theta_t}(X^i_t) \tilde X^{b, i}_t \right) \right) \right] dt \\
dc_t &= -\alpha_t \left[ \frac1N \sum_{i=1}^N \left( 2 X^i_t \tilde X^{c, i}_t + 2f_{\theta_t}(X^i_t) 
\left( \sigma(W_t X^i_t + B_t) + f'_{\theta_t}(X^i_t) \tilde X^{c, i}_t \right) \right) \right] dt \\
dX^i_t &= \left[ a \left( \frac1N \sum_j X_t^{j} - X^{i}_t \right) + f_{\theta_t}(X^i_t) \right] dt + dW^i_t \\
d\tilde X^{w, i}_t &= \left[ a \left( \frac1N \sum_j \tilde X_t^{w, j} - \tilde X^{w, i}_t \right) + C_t \odot \sigma'(W_t X^i_t + B_t) X^i_t + f'_{\theta_t}(X^i_t) \tilde X^{w, i}_t \right] dt \\
d\tilde X^{b, i}_t &= \left[ a \left( \frac1N \sum_j \tilde X_t^{b, j} - \tilde X^{b, i}_t \right) + C_t \odot \sigma'(W_t X^i_t + B_t) + f'_{\theta_t}(X^i_t) \tilde X^{b, i}_t \right] dt \\
d\tilde X^{c, i}_t &= \left[ a \left( \frac1N \sum_j \tilde X_t^{c, j} - \tilde X^{c, i}_t \right) + \sigma(W_t X^i_t + B_t) + f'_{\theta_t}(X^i_t) \tilde X^{c, i}_t \right] dt.
\eae

The trained neural network control is also displayed in Figure \ref{risk}; the controls learned by the linear model and neural network are similar.

\begin{figure}[htbp]
\centering
\begin{minipage}[t]{0.48\textwidth}
\centering
\includegraphics[width=6cm]{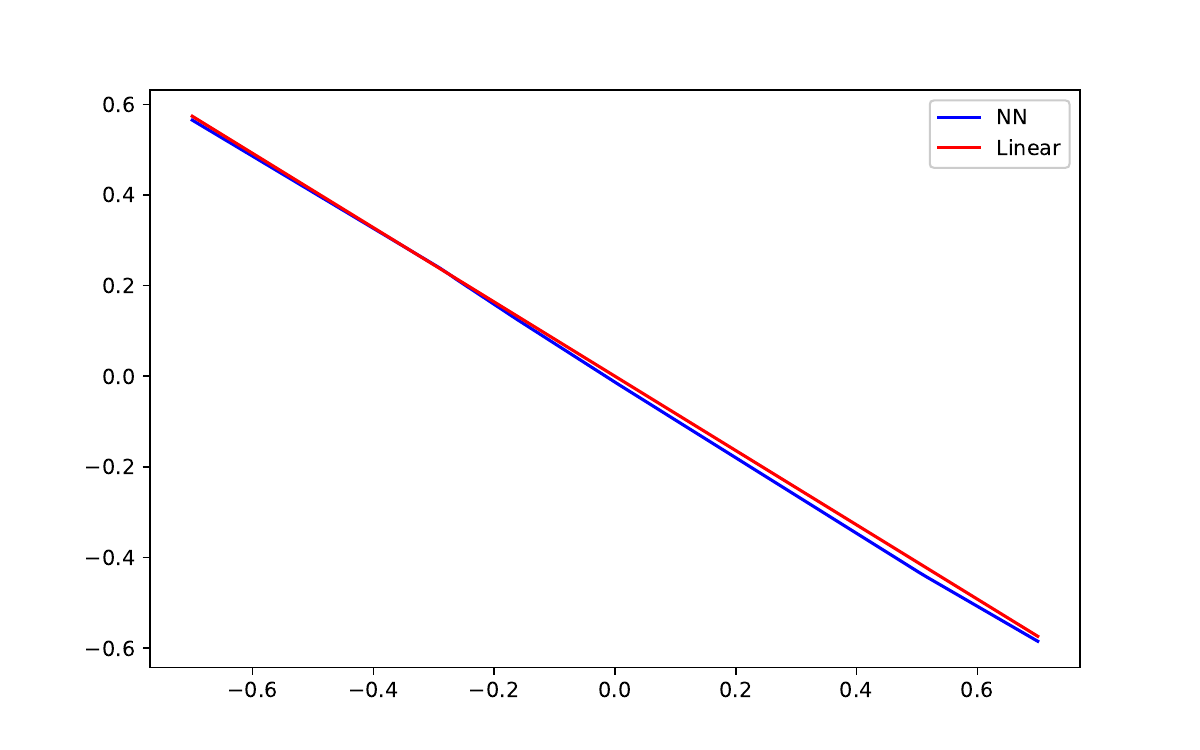}
\caption{Training result for \eqref{risk linear} and \eqref{risk NN}.}
\label{risk}
\end{minipage}
\end{figure}

\color{black}

\subsection{Models of Order Book Dynamics} \label{OrderBook}

Order books involve large numbers of high-frequency events ($\sim 10^5 - 10^6$  events per day per stock) and high-dimensional dynamics (many price levels, each with limit order submissions and cancellations, as well as market orders, hidden orders,
and transactions). Due to the size of the datasets and the high-dimensionality, calibrating simulation models of order book dynamics to data is computationally challenging. Recent examples of such model frameworks for the simulation of  
the order books include \cite{Pakkanen} \cite{Pakkanen2} \cite{Cartlidge} \cite{Abergel} \cite{Kumar}. \cite{Pakkanen} \cite{Pakkanen2} \cite{Cartlidge} \cite{Abergel} \cite{Kumar} develop stochastic point process models to model the event-by-event dynamics in order books.

For more complex stochastic models, it is computationally intractable for many traditional calibration methods to optimize over the entire order flow history (even for a few days of events) to estimate the model parameters from the data. The online forward propagation optimization algorithm proposed in this paper provides a tractable computational method to optimize over the entire order flow history. In particular, the online forward propagation optimization algorithm asymptotically minimizes the objective function over the stationary distribution of the entire order flow process (instead of optimizing over only small subsets of the data, which can lead to a sub-optimal model parameter calibration). In principle, our online optimization algorithm could be used to calibrate a general class of point process models to event-by-event order book data. Such a large-scale data project is outside of the scope of this paper, which is focused on developing a convergence theory. However, in order to demonstrate the applicability of our method to point process models, we present two simple numerical examples below. Synthetic data is simulated from a standard Hawkes process with stochastic intensity
\begin{eqnarray}
d \lambda_t^{\ast} = - \alpha^{\ast} ( \mu^{\ast} - \lambda_t^{\ast}) dt + \kappa^{\ast} d N_t^{\ast},
\label{Hawkes}
\end{eqnarray}
where $N_t^{\ast}$ is the number of events that have occurred by time $t$. Events arrive with stochastic intensity $\lambda_t$, i.e. $ \lim_{\Delta \rightarrow 0} \frac{ \mathbb{P}[N_{t+ \Delta}^{\ast} - N_t^{\ast} = 1 | \mathcal{F}_t ] }{\Delta} = \lambda_t^{\ast}$. For example, $N_t^{\ast}$ could be the number of limit orders submitted to the order book by time $t$. Multi-dimensional point process models can model the dynamics of the entire order book (e.g., limit order submissions, cancellations, market orders, hidden orders, and transactions) \cite{Pakkanen} \cite{Pakkanen2}.

Model parameters for point process models can be calibrated from event data. The data consists of only the observed process $N_t^{\ast}$; the stochastic intensity $\lambda_t^{\ast}$ is unobserved. Note that (\ref{Hawkes}) is an ergodic process with a stationary distribution. Hawkes process models have been widely used in the financial literature for modeling order book events (for example, see \cite{Pakkanen}). Using the event data $N_t^{\ast}$ simulated from (\ref{Hawkes}), we will calibrate point process models using the online forward propagation optimization algorithm.

First, we consider calibrating a standard Hawkes model using the online optimization algorithm. The model is
\begin{eqnarray}
d \lambda_t^{\theta} = - \alpha ( \mu - \lambda_t^{\theta}) dt + \kappa d N_t^{\theta},
\label{HawkesTrained}
\end{eqnarray}
where $\theta = (\alpha, \mu, \kappa)$ are the parameters that must be trained and the time-averaged log-likelihood objective function is
\begin{eqnarray}
L_T(\theta) = - \frac{1}{T} \int_0^T \hat{\lambda}_t^{\theta} dt + \frac{1}{T} \int_0^T \log( \hat{\lambda}_t^{\theta} ) d N^{\ast}_t,
\label{LikelihoodPointProcess}
\end{eqnarray}
where $\hat{\lambda}_t^{\theta}$ is the intensity process (\ref{HawkesTrained}) conditioned on the event observations $(N^{\ast}_{t'})_{t' \leq t}$, i.e. $d \hat{\lambda}_t^{\theta} = - \alpha ( \mu - \hat{\lambda}_t^{\theta}) dt + \kappa d N_t^{\ast}$.  Using our online optimization algorithm, we train the parameters $\theta_t$ to maximize the objective function $L_T(\theta)$. Figure \ref{HawkesTraining} displays the results from the training and demonstrate the numerical convergence of the method. The training converges to a global minimizer; the objective function evaluated at the trained parameters matches the objective function evaluated at the true parameters $\theta^{\ast} = ( \alpha^{\ast}, \mu^{\ast}, \kappa^{\ast} ) = (\frac{1}{10}, 1, \frac{1}{10} )$.

\begin{figure}[h!]
\begin{center}
\includegraphics[width=.4\textwidth, height=50mm]{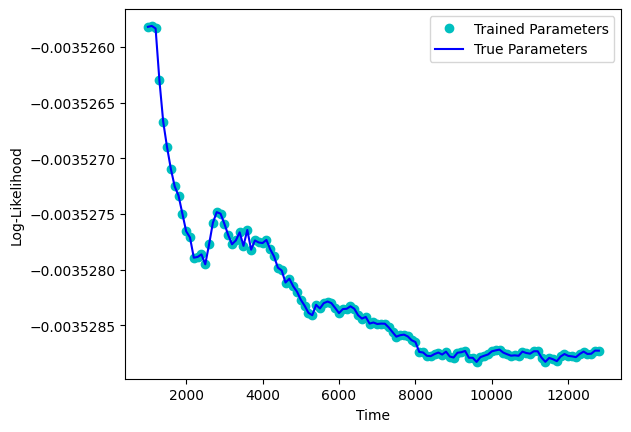}
\includegraphics[width=.4\textwidth, height=50mm]{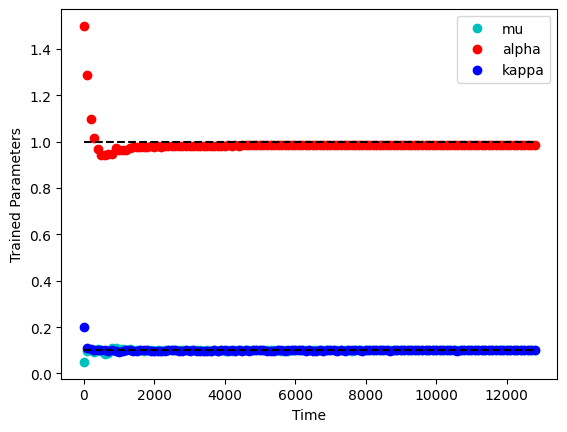}
\end{center}
\caption{Objective function (left) and trained parameters (right).}
\label{HawkesTraining}
\end{figure}

We now consider a slightly more complex model where the intensity dynamics are given by a neural network. Neural network (or ``neural SDEs") have been widely studied in the financial mathematics literature  \cite{Szpruch2, cohen2021arbitrage, cohen2022estimating, cohen2022hedging, Szpruch1, Szpruch3}. Neural network Hawkes processes (or ``neural Hawkes processes") have also been recently studied and implemented in a number of papers for modeling order book data \cite{Cartlidge} \cite{Abergel} \cite{Kumar}. We consider the following neural SDE:
\begin{eqnarray}
d \bar{ \lambda_t}^{\theta} = f( \bar{\lambda_t}^{\theta}; \theta) dt + \kappa d N_t^{\theta},
\label{NeuralHawkesTrained}
\end{eqnarray}
where, for this simplified numerical experiment, we set $\kappa = \kappa^{\ast}$ and $\lambda_t^{\theta} = | \bar{\lambda}_t^{\theta} | + \epsilon$ where $\epsilon > 0$. $f(  \lambda ; \theta) $ is a single-layer neural network with $25$ hidden units. The neural network parameters $\theta$ are trained with the online forward propagation optimization algorithm:
\begin{eqnarray}
d \tilde \lambda_t &=& \bigg{(} \frac{\partial f}{\partial \lambda}( \bar \lambda_t; \theta_t) \tilde{\lambda}_t  +  \frac{\partial f}{\partial \theta}( \bar \lambda_t; \theta_t)  \bigg{)}dt, \notag \\
d \bar{ \lambda_t} &=& f( \bar{\lambda_t}; \theta_t) dt + \kappa d N_t^{\ast}, \notag \\
d \theta_t &=& \alpha_t \bigg{(} - \frac{\partial \lambda_t}{\partial \bar \lambda_t } \tilde \lambda_t  dt+ ( \lambda_t)^{-1}  \frac{\partial \lambda_t}{\partial \bar \lambda_t } \tilde \lambda_t  d N^{\ast}_t  \bigg{)},
\end{eqnarray}
where $\lambda_t = | \bar{\lambda}_t | + \epsilon$ and $\alpha_t$ is the learning rate. The data $N_t^{\ast}$ which the model (\ref{NeuralHawkesTrained}) is trained on is generated using (\ref{Hawkes}) with the ``true parameters" $\theta^{\ast} = (\frac{1}{10}, 1, \frac{1}{10} )$. The training and out-of-sample test results are displayed in Figure \ref{HawkesTrainingNN}. The plots display the value of the objective function (\ref{LikelihoodPointProcess}) evaluated using the ``true" process (\ref{Hawkes}) with the true parameters $\theta^{\ast}$ (which is the global minimum) as compared to the value of the objective function (\ref{LikelihoodPointProcess}) for the trained model (\ref{NeuralHawkesTrained}). The neural network point process model (\ref{NeuralHawkesTrained}), trained with the online forward propagation algorithm, is able to achieve a nearly identical value for the objective function as the exact global minimizer (with $\sim 10^{-4}$ relative error), indicating that the trained model converges to a global minimizer.

\begin{figure}[h!]
\begin{center}
\includegraphics[width=.4\textwidth, height=50mm]{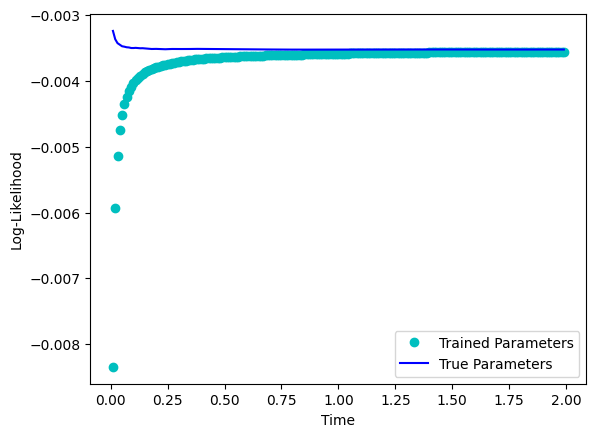}
\includegraphics[width=.4\textwidth, height=50mm]{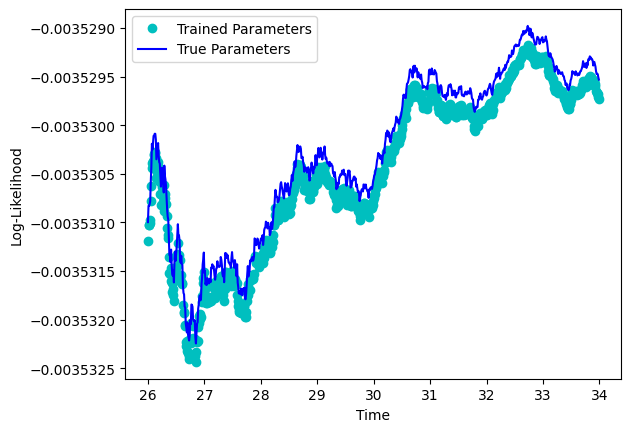}
\end{center}
\caption{Objective function during the initial time period of training (left) and out-of-sample objective function (right).}
\label{HawkesTrainingNN}
\end{figure}

We conclude by highlighting that -- although outside of the scope of this paper -- a more general multi-dimensional model for the entire order book (see \cite{Pakkanen}) could also be calibrated to real order book data using the online forward propagation algorithm. General classes of multi-dimensional neural SDE models can be optimized using our method. For example, ``recurrent neural SDEs", where the dynamics (\ref{HawkesTrained}) depend upon the evolution of a ``hidden" neural SDE, can also be calibrated using the online forward propagation method, such as:
\begin{eqnarray}
d \tilde \lambda_t^{\theta} = f( \tilde \lambda_t^{\theta}, S_t^{\theta}; \theta) dt + \kappa(  \tilde \lambda_t^{\theta}, S_t^{\theta}; \theta )d N_t^{\theta}, \notag \\
d S_t^{\theta} = g( \tilde \lambda_t^{\theta}, S_t^{\theta}; \theta) dt + h(  \tilde \lambda_t^{\theta}, S_t^{\theta}; \theta )d N_t^{\theta}, 
\label{HawkesTrainedGeneral}
\end{eqnarray}
where $f, g, h, \kappa$ are neural networks with collective parameters $\theta$ and where $\tilde \lambda_t^{\theta}, N_t^{\theta},$ and $S_t^{\theta}$ can be multi-dimensional. Recurrent neural networks Hawkes models for order books have been investigated
in \cite{Kumar} \cite{Cartlidge}. Recurrent neural network Hawkes processes have recently received significant interest in the broader machine learning community \cite{Eisner}. General classes of continuous-time recurrent network SDEs have also been
proposed in \cite{Natarajan}. A more general class of continuous-time recurrent network point processes has also been developed in \cite{Nickel}; (\ref{HawkesTrainedGeneral}) is an example from the general framework in \cite{Nickel}. The unique capability provided by the algorithm is to asymptotically optimize such models over the \emph{entire history} of the order flow dataset, while standard methods can typically only optimize over much smaller sub-sequences.

\section{Conclusion}

\hspace{1.4em}In this paper, we proposed a new online algorithm for computationally efficient optimization over the stationary distribution of ergodic SDEs. In particular, the online forward propagation algorithm can optimize over parameterized SDEs in order to minimize the distance between their stationary distribution and target statistics. By proving bounds for a new class of Poisson PDEs, we can analyze the parameters' fluctuations during training and rigorously prove convergence to a stationary point for linear SDE models. We also study the numerical performance of our algorithm for nonlinear examples. In the nonlinear cases which we present in this paper, the algorithm performs well and the parameters converge to a minimizer. 

Our algorithm can be used for applications where optimizing over the stationary distribution of an SDE model is of interest. In many applications, the stationary distribution $\pi_{\theta}$ is unknown and the dimension of the stochastic process may be large. The online algorithm developed in this paper is well-suited for such problems. 

Finally, there are several future research directions which should be explored. First, a convergence analysis for nonlinear SDEs would be an important next step. The focus of our paper is a convergence analysis for linear SDEs; this required addressing several non-trivial mathematical challenges, in particular the development and rigorous analysis of a new class of Poisson PDEs. Our results in this paper provide the building blocks for a future nonlinear analysis. The convergence of our online algorithm for discrete-time stochastic processes would also be interesting to study.

\section*{Acknowledgement}

This research has been supported by the EPSRC Centre for Doctoral Training in Mathematics of Random Systems: Analysis, Modelling and Simulation (EP/S023925/1).

\section*{Appendix}

\appendix
\renewcommand{\appendixname}{Appendix~\Alph{section}}

\section{Proof of Proposition \ref{ergodic estimation}} \label{ergodic appendix}

We first present a useful lemma before proving Proposition \ref{ergodic estimation}. The bound \eqref{key} will be frequently used in the proof of Proposition \ref{ergodic estimation}.

\begin{lemma}
\label{key lemma}
For any $m', k \in \mathbb{R}_+$, there exist constants $C, m > 0$ such that for any $x, x' \in \mathbb{R}^d$, 
\beq
\label{key}
e^{-\left| x' - x \right|^2} \cdot \left| x' - x \right|^k \le C \frac{1+|x|^m}{1+|x'|^{m'}}. 
\eeq
\end{lemma}
\begin{proof}
For any fixed $x \in \mathbb{R}^d$, when $\frac{\left| x' \right|}{2} \ge \left| x \right|$ we have 
\begin{eqnarray}
\left| x' - x \right| &\ge& \left| x' \right| - \left| x \right| \ge \frac{\left| x' \right|}{2}, \notag \\
\left| x' - x \right| &\le& \left| x' \right| + \left| x \right| \le \frac{3\left| x' \right|}{2}.
\label{xminusxprimeIneq}
\end{eqnarray}
Therefore, we have that for any $m', k>0$ there exists a constant $C_1>0$ such that
\beq
e^{-\left| x' - x \right|^2} \cdot \left| x' - x \right|^k \le e^{ -\frac{\left| x' \right|^2}{4}} \left( \frac{3}{2} \left| x' \right|\right)^k \overset{(a)}{\le} \frac{C_1}{1+|x'|^{m'}},
\eeq
where the first inequality is due (\ref{xminusxprimeIneq}) and step (a) uses the fact that
\beq
\label{order}
\lim\limits_{s \to +\infty} \frac{s^m}{e^s} = 0, \quad \forall m > 0.
\eeq
When $\frac{\left| x' \right|}{2} < \left| x \right|$, for any $m',k>0$ there exist constants $C_2, m > 0$ such that
\beq
e^{-\left| x' - x \right|^2} \cdot \left| x' - x \right|^k \le \left( 3|x| \right)^k \le \frac{\left( 3|x| \right)^k \cdot \left( 1+|2x|^{m'}\right) }{1+|x'|^{m'}} \le C_2 \frac{1+|x|^m}{1+|x'|^{m'}}.
\eeq
Let us now choose $C= C_1 + C_2 $ and then \eqref{key} holds.
\end{proof}

\begin{proof}[Proof of Proposition \ref{ergodic estimation}:]
The proof for the convergence results leverages the closed-form formula for the distribution. Let
\beq
f(t,x,\theta) = e^{-h(\theta)t} x + h(\theta)^{-1} \left( I_d - e^{-h(\theta)t} \right)g(\theta), \quad \Sigma_t(\theta) = \sigma^2(2h(\theta))^{-1}\left( I_d - e^{-2h(\theta)t} \right),
\eeq
and from \eqref{solution} we know that
\beq
\label{multi dim normal}
X_t^\theta \sim N\left( f(t,x,\theta),\ \Sigma_t(\theta) \right).
\eeq
Thus, the stationary distribution for $X_t^\theta$ is $N\left( h^{-1}(\theta) g(\theta),\  \sigma^2(2h(\theta))^{-1} \right)$.
Since $h(\theta)$ is positive definite, there exists orthogonal matrix $Q(\theta)$ such that 
$$
h(\theta) = Q(\theta)^\top \Lambda(\theta) Q(\theta)
$$
where $\Lambda(\theta) = \text{diag}(\lambda_1(\theta), \cdots, \lambda_d(\theta))$ is a diagonal and all its eigenvalues are positive. Thus for $t>0$
\beq
\label{spectrum}
\Sigma_t(\theta) = \frac{\sigma^2}{2} Q(\theta)^T \Lambda^{-1}(\theta) \left(I_d -  e^{-2\Lambda(\theta)t}\right)  Q(\theta),
\eeq
and the eigenvalues of $\Sigma_t(\theta)$ are $\left(\frac{\sigma^2 \left( 1-e^{-2\lambda_1(\theta)t}\right)}{2\lambda_1(\theta)} , \cdots, \frac{\sigma^2 \left( 1-e^{-2\lambda_d(\theta)t}\right)}{2\lambda_d(\theta)}  \right)$.
Then we know the covariance matrix $\Sigma_t(\theta)$ is also positive definite for any $t>0$ and the density is 
\bae
\label{close formula}
p_t(x, x', \theta) &= \frac{1}{\sqrt{(2\pi)^{d} \left| \Sigma_t(\theta) \right|}} \exp\left\{ -\frac12 (x' -f(t,x,\theta))^\top \Sigma^{-1}_t(\theta)  (x' -f(t,x,\theta))  \right\}, \ t>0 \\
p_\infty(x',\theta) &= \frac{1}{\sqrt{(2\pi)^d \left| \sigma^2 (2h(\theta))^{-1} \right| }} \exp\left\{ - (x' -h(\theta)^{-1}g(\theta))^\top \frac{h(\theta)}{\sigma^2} (x' - h(\theta)^{-1} g(\theta)) \right\}.
\eae

\begin{itshape}Proof of (\romannumeral1).\end{itshape} Recall that (by assumption) $h(\theta)$ is uniformly positive definite and thus 
\beq
\label{invariant bound}
p_\infty(x', \theta) \le C \sqrt{ \left| h(\theta) \right| } \exp\left\{ -c|x' - h(\theta)^{-1}g(\theta)|^2 \right\} \overset{(a)}{\le} \frac{C}{1+|x'|^m}  
\eeq
where step $(a)$ uses the bounds for $g, h$ in Assumption \ref{condition} and \eqref{order}. Due to \eqref{close formula}, we have for any $k \in \{ 1, 2, \cdots, \ell\} $ that
\bae
\label{stationary first theta}
&\left| \frac{\partial}{\partial \theta_k} p_\infty(x', \theta) \right| \\
\le& C \left( \frac{\partial}{\partial \theta_k} \sqrt{|h(\theta)|} \right) \cdot \exp\left\{ -c\left|x' - h(\theta)^{-1} g(\theta)\right|^2 \right\} + C \sqrt{|h(\theta)|} \exp\left\{ -c\left|x' -h(\theta)^{-1} g(\theta)\right|^2 \right\}\\
\cdot& \left[ \left| \left(x'- h(\theta)^{-1} g(\theta)\right)^\top \frac{\partial h(\theta)}{\partial \theta_k} (x' - h(\theta)^{-1} g(\theta)) \right| + 2\left| \left( \frac{\partial \left( h(\theta)^{-1} g(\theta)\right) }{\partial \theta_k} \right)^\top h(\theta) \left(x' - h(\theta)^{-1} g(\theta) \right) \right| \right] \\
\overset{(a)}{\le}& C \exp\left\{ -c\left|x' -h(\theta)^{-1} g(\theta)\right|^2 \right\} + C \exp\left\{ -c\left|x' - h(\theta)^{-1} g(\theta)\right|^2 \right\} \left( \left|x' - h(\theta)^{-1} g(\theta)\right|^2 +  \left|x' - h(\theta)^{-1} g(\theta)\right| \right) \\
\overset{(b)}{\le}& \frac{C}{1+|x'|^m},
\eae
where step $(a)$ is by the boundedness of $g(\theta), \frac{\partial g(\theta)}{ \partial \theta_k}, h(\theta), \frac{\partial h(\theta)}{\partial \theta_k}$ and since $h(\theta)$ is positive definite due to Assumption \ref{condition}. Step $(b)$ is due to equation \eqref{key} with $x = h(\theta)^{-1}g(\theta)$ and equation (\ref{order}). Using the same method as in \eqref{stationary first theta}, we can obtain the bound for $\nabla^2_\theta p_\infty(x, \theta)$. 

\begin{itshape}Proof of (\romannumeral2) and (\romannumeral3).\end{itshape}
We now prove \eqref{x prime decay}. First let 
$$
X:= x' - f(t,x,\theta), \quad Y:= x' -h(\theta)^{-1}g(\theta),
$$
and then since $h$ is uniformly positive definite:
\bae
\label{difference}
|X - Y| =  \left|e^{-h(\theta)t}x - e^{-2h(\theta)t}h(\theta)^{-1}g(\theta) \right| \le Ce^{-ct}(1+|x|).
\eae
We will use the following decomposition:
\begin{equation}
\begin{aligned}
\label{decomposition0}
&|p_t(x,x',\theta) - p_\infty(x',\theta)| \\
\le& C \left( \frac{1}{\sqrt{ \left| \left( I_d - e^{-2h(\theta)t} \right) \right|}} - 1 \right) + C \left| \exp\left\{ -X^\top \frac{h(\theta)}{\sigma^2}\left( I_d - e^{-2h(\theta)t}\right)^{-1} X \right\} - \exp\left\{ -Y^\top \frac{h(\theta)}{\sigma^2}\left( I_d - e^{-2h(\theta)t}\right)^{-1} Y \right\} \right| \\
+& \left|\exp\left\{ -Y^\top \frac{h(\theta)}{\sigma^2} \left( I_d - e^{-2h(\theta)t}\right)^{-1} Y \right\} - \exp\left\{ -Y^\top \frac{h(\theta)}{\sigma^2} Y \right\}  \right| \\
=&: I_1 + I_2 + I_3
\end{aligned}
\end{equation}
For $I_1$, note that when $t>1$
\beq
\label{dif1}
\frac{1}{\sqrt{ \left| \left( I_d - e^{-2h(\theta)t} \right) \right|}} - 1 = \frac{1}{ \sqrt{ \prod\limits_{k=1}^d \left( 1 -  e^{-2\lambda_k(\theta)t} \right)} } - 1 \le C\left[ 1 - \prod_{k=1}^d \left( 1 -  e^{-2\lambda_k(\theta)t} \right) \right] \le C e^{-2\lambda_1(\theta)t} \le C e^{-2ct},
\eeq
where $\lambda_1(\theta) \le \lambda_2(\theta) \le \cdots \le \lambda_d(\theta)$ are the eigenvalues of the matrix $h(\theta)$. For $I_3$, similar to  \eqref{spectrum}, we know the eigenvalues of $h(\theta) \left( \left( I_d - e^{-2h(\theta)t}\right)^{-1} -I_d \right)$
are $\frac{\lambda_i(\theta) e^{-2\lambda_i(\theta)t}}{1-e^{-2\lambda_i(\theta)t}},\ i=1,\cdots,d$, which implies that $h(\theta)\left(\left( I_d - e^{-2h(\theta)t}\right)^{-1} -I_d\right)$
is also a positive definite matrix. When $t>1$, since $h(\theta)$ is uniformly positive definite, the eigenvalues will have a uniform upper bound:
\beq
\label{eigen bound}
\frac{\lambda_i(\theta) e^{-2\lambda_i(\theta)t}}{1-e^{-2\lambda_i(\theta)t}} \le \frac{C}{1-e^{-c}} \le C , \quad t>1, \ \forall i \in \{1, \cdots, d\}.
\eeq
Thus for any $m',k>0$, there exists a constant $C>0$ such that when $t>1$
\begin{equation}
\begin{aligned}
\label{dif2}
&\left| \exp\left\{ -Y^\top \frac{h(\theta)}{\sigma^2}\left( I_d - e^{-2h(\theta)t}\right)^{-1} Y  \right\} - \exp\left\{ -Y^\top \frac{h(\theta)}{\sigma^2} Y \right\} \right| \\
=& \exp\left\{ -Y^\top \frac{h(\theta)}{\sigma^2} Y \right\} \left| \exp\left\{ -Y^\top \frac{h(\theta)}{\sigma^2}\left( \left( I_d - e^{-2h(\theta)t}\right)^{-1} -I_d\right) Y  \right\} - 1 \right| \\
\overset{(a)}{\le}& C \exp\left\{ -Y^\top \frac{h(\theta)}{\sigma^2} Y \right\} \left|Y^\top \frac{h(\theta)}{\sigma^2}\left( \left( I_d - e^{-2h(\theta)t}\right)^{-1} -I_d\right) Y  \right|
\\
\overset{(b)}{\le}& C \exp\left\{ -Y^\top \frac{h(\theta)}{\sigma^2} Y \right\} |Y|^2 \cdot \lambda_{\max} \left( h(\theta) \left( \left( I_d -e^{-2h(\theta)t}\right)^{-1} -I_d\right)\right) \\
\overset{(c)}{\le}& C \exp\left\{ -c\left|x' - h(\theta)^{-1} g(\theta)\right|^2 \right\} \cdot \left|x' - h(\theta)^{-1} g(\theta)\right|^2 
\\ 
\overset{(d)}{\le}& C e^{-ct} \frac{1}{1+|x'|^{m'}}\\
\le& C\frac{1}{(1+|x'|^{m'})(1+t)^k},
\end{aligned}
\end{equation}
where step $(a)$ is by the positive definiteness of $\frac{h(\theta)}{\sigma^2}\left( \left( I_d - e^{-2h(\theta)t}\right)^{-1} -I_d\right)$, which means 
$$
Y^\top\frac{h(\theta)}{\sigma^2}\left( \left( I_d - e^{-2h(\theta)t}\right)^{-1} -I_d\right) Y \ge 0,
$$ 
and the fact $0\le 1-e^{-s} \le s,\ \forall s\ge0$. In step $(b)$, $\lambda_{\max}$ denotes the largest eigenvalue and step $(c)$ uses \eqref{eigen bound}. Step $(d)$ follows from \eqref{key} with $x = h(\theta)^{-1}g(\theta)$ and the boundedness of $g,h$.

For $I_2$, define the function on $F_t: \mathbb{R}^d \rightarrow \mathbb{R}$ for $t>0$
$$
F_t(x) := \exp \left \{ -x^\top \frac{h(\theta)}{\sigma^2}\left( I_d - e^{-2h(\theta)t}\right)^{-1} x \right\}.
$$
By mean value theorem,
\beq
F_t(x) - F_t(y) = \nabla F_t(x_0)^\top (x-y) = -\frac{2h(\theta)}{\sigma^2} \left( I_d - e^{-2h(\theta)t}\right)^{-1} F_t(x_0)x_0^\top (x-y),
\eeq
where $x_0 = t_0 x + (1-t_0)y$ for some $t_0 \in [0, 1]$. Thus for any $m', k>0$ there exist constants $C, m>0$ such that when $t>1$
\begin{equation}
\begin{aligned}
\label{dif3}
\left| F_t(x) - F_t(Y) \right| &\overset{(a)}{=} \frac{2}{\sigma^2} \left| \exp\left\{-(X_0)^\top \frac{h(\theta)}{\sigma^2}\left( I_d - e^{-2h(\theta)t}\right)^{-1} X_0 \right\} X_0^\top (X-Y) \right|\\
&\overset{(b)}{\le} e^{-c |X_0|^2} |X_0| Ce^{-ct}(1+|x|) \\ &\overset{(c)}{\le} C\frac{1+|x|^m}{(1+|x'|^{m'})(1+t)^k},
\end{aligned}
\end{equation}
where in step $(a)$ 
\beq
\label{mean value}
X_0 = t_0 X+ (1-t_0)Y = x' - t_0 f(t,x,\theta) - (1-t_0) h(\theta)^{-1}g(\theta),
\eeq
for some $t_0 \in [0,1]$. Step $(b)$ uses \eqref{difference} and \eqref{eigen bound} and step $(c)$ is by substituting in $x$ in \eqref{key} to be the $X_0$ in \eqref{mean value}. Combining \eqref{decomposition0}, \eqref{dif1}, \eqref{dif2}, and \eqref{dif3}, we have for $t>1$
\beq
\label{exp conv}
|p_t(x,x',\theta) - p_\infty(x',\theta)| \le C\frac{1+|x|^m}{(1+|x'|^{m'})(1+t)^k}.
\eeq

The proof of \eqref{x prime decay} for the case $i=1,2$ and \eqref{x decay} is the same as the proof for $|p_t(x,x',\theta) - p_\infty(x',\theta)|$ above (i.e., one uses the decomposition in \eqref{decomposition0} and \eqref{key} with different choices of $x$). The only challenge is establishing a bound for $\nabla_\theta e^{-h(\theta)t}$. $e^{-h(\theta)t}$ satisfies the ODE
\beq
\label{ode}
\frac{d}{dt} e^{-h(\theta)t} = -h(\theta) e^{-h(\theta)t}
\eeq
with initial value $I_d$.\footnote{Here we use the fact that $\frac{\partial}{\partial y} e^{A y} = A e^{Ay} = e^{A y} A$.} Differentiating \eqref{ode} with respect to $\theta_i, i \in \{1, \cdots, d\}$ yields an ODE for $\frac{\partial}{\partial \theta_i} e^{-h(\theta)t}$:
\beq
\frac{d}{dt} \frac{\partial}{\partial \theta_i} e^{-h(\theta)t} = -\frac{\partial h(\theta) }{\partial \theta_i} e^{-h(\theta)t} - h(\theta) \frac{\partial}{\partial \theta_i}e^{-h(\theta)t},
\eeq
with initial value $0$. Using an integrating factor yields 
$$
\frac{d}{dt} \left( e^{h(\theta) t} \frac{\partial}{\partial \theta_i} e^{-h(\theta)t} \right) = -e^{h(\theta) t} \frac{\partial h(\theta)}{\partial \theta_i} e^{-h(\theta)t},
$$
and thus 
\beq
\frac{\partial}{\partial \theta_i} e^{-h(\theta)t} = e^{-h(\theta)t} \int_0^t e^{h(\theta)s} \frac{\partial h(\theta)}{\partial \theta_i} e^{-h(\theta)s} ds.
\eeq
Since $e^{h(\theta)t}$ is invertible for any $t$, we know the matrices $e^{h(\theta)s} \frac{\partial h(\theta)}{\partial \theta_i} e^{-h(\theta)s}$ and $\frac{\partial h(\theta)}{\partial \theta_i}$ are similar and thus their eigenvalues are the same, which implies that their spectral norm are also the same. We therefore can show that
\beq
\label{norm}
\left| \frac{\partial}{\partial \theta_i} e^{-h(\theta)t} \right| \le C \left| e^{-h(\theta)t} \right| \int_0^t \left| \frac{\partial h(\theta)}{\partial \theta_i} \right| ds \overset{(a)}{\le} Ce^{-ct} t,
\eeq
where step $(a)$ is by the bound for $\nabla_\theta h(\theta)$ in Assumption \ref{condition}. Using the same method, we also can show that
\beq
\label{norm 2}
\left| \frac{\partial^2}{\partial \theta_i \partial \theta_j} e^{-h(\theta)t} \right| \le C \left| e^{-h(\theta)t} \right| \int_0^t \left| \frac{\partial h(\theta)}{\partial \theta_i \partial \theta_j} \right| ds \overset{(a)}{\le} Ce^{-ct} t, \quad \forall i,j \in \{1, \cdots, d\}.
\eeq

\begin{itshape} Proof of (\romannumeral4).\end{itshape}
The first part of \eqref{normal bound} follows from the fact that $X_t^\theta$ has a multivariate normal distribution whose mean and variance are uniformly bounded. \eqref{normal bound} is obvious when $t=0$. For $t>0$, as we know $\Sigma_t(\theta)$ is positive definite for $t>0$, thus the random variable
\beq
\label{standard normal}
Y:=  \Sigma_t^{-\frac12}(\theta) \left(  X_t^\theta - f(t,x,\theta) \right)  
\eeq
has a $d$-dimensional standard normal distribution, where $\Sigma_t^{\frac12}(\theta)$ denotes the square root matrix of $\Sigma_t(\theta)$. Since for any $m>0$ there exists a $C_m>0$ such that $\e\left|Y\right|^m = C_m <\infty$.
\bae
\label{X theta normal bound}
\e_x\left|X_t^\theta\right|^m = \e_x\left|\Sigma_t^{\frac12}(\theta) Y + f(t,x,\theta) \right|^m \le C \left( \left|\Sigma_t^{\frac12}(\theta) \right|^m \e_x \left| Y\right|^m + \left|f(t,x,\theta)\right|^m \right)
\overset{(a)}{\le} C(1+|x|^m),
\eae
where step $(a)$ is by the uniform bound for $g(\theta)$ and $h(\theta)$ in Assumption \ref{condition}. For the second part of \eqref{normal bound}, we use \eqref{X theta normal bound} to develop the following bound:
\bae
\label{tilde thete bound}
\e_{x,\tilde x}|\tilde X^\theta_t|^m &\le 2|\tilde x|^m + 2\e_{x, \tilde x}\left| \int_0^t \left| e^{-h(\theta)(t-s)}\right| \cdot \left| \nabla_\theta g(\theta) - \nabla_\theta h(\theta)X^\theta_s \right| ds \right|^m\\
&\overset{(a)}{\le} 2|\tilde x|^m + C_m \e_{x, \tilde x}\left| \int_0^t e^{-c(t-s)} \left( 1 + |X^\theta_s| \right) ds \right|^m\\
&\le 2|\tilde x|^m + C_m \e_{x}\left| \int_0^t \frac{e^{cs}}{e^{ct} - 1} \left( 1 + |X^\theta_s| \right) ds \right|^m e^{-cm t}\left( e^{ct} - 1 \right)^m \\
&\overset{(b)}{\le} 2|\tilde x|^m + C_m \e_x\left| \int_0^t \frac{e^{cs}}{e^{ct} - 1} \left( 1 + |X^\theta_s| \right)^m ds \right|\\
&\le C_m \left( 1 + |x|^m + |\tilde x|^m \right),
\eae		
where step $(a)$ is by Assumption \ref{condition} and the fact
\beq
\lambda_{\max} \left( e^{-h(\theta)(t-s)} \right) = e^{ -\lambda_{\min}\left( h(\theta)(t-s) \right)} \le e^{-c(t-s)}
\eeq
and step (b) is by Jensen's inequality. In particular, let $p(s) = \frac1c \frac{e^{cs}}{e^{ct}-1}$ and we have $\int_0^t p(s) ds = 1$, and therefore $p(s)$ is a probability density function on $[0,t]$. By Jensen's inequality,
\beq
\left| \int_0^t \left( 1 + |X^\theta_s| \right) p(s) ds \right|^m \le \int_0^t \left( 1 + |X^\theta_s| \right)^m p(s)ds,
\eeq
which we have used in step $(b)$ of equation (\ref{tilde thete bound}).

\begin{itshape} Proof of (\romannumeral5).\end{itshape}
For \eqref{expectation bound}, the conclusion for $t=0$ is trivial. When $t>0$, by \eqref{poly} and \eqref{normal bound}, we have for any polynomial bounded function $f$ that
\beq
\label{exp bound}
|\e_x f\left(X_t^\theta\right) | \le \e_x| f\left(X_t^\theta\right) | \le C \e_x (1 + \left|X_t^\theta\right|^m) \le C(1+|x|^m).
\eeq
For the derivatives, we will use the dominated convergence theorem. By \eqref{close formula}, we have 
\beq
\label{expectation z}
\e_x f\left(X_t^\theta\right) = \int_{\mathbb{R}^d} f\left(f(t,x,\theta) + x'\right) \frac{1}{\sqrt{(2\pi)^d \left| \Sigma_t(\theta) \right|}} \exp\left\{ -\frac12 x'^\top \Sigma^{-1}_t(\theta) x' \right\}dx'.
\eeq
Let $Z^\theta$ denote a normal distribution  
$$
Z_t^\theta \sim N\left( 0, \Sigma_t(\theta) \right)
$$
and then 
\beq
\e_x f(X_t^\theta) =  \e f \left( f(t,x,\theta) + Z_t^\theta \right) =  \e f \left( e^{-h(\theta)t} x + h(\theta)^{-1} \left( I_d - e^{-h(\theta)t} \right)g(\theta) + Z_t^\theta \right).
\eeq
For $\nabla_x \e_x f(X_t^\theta)$, we change the order of $\nabla_x$ and $\e_x$ and obtain for $t \in (0,1]$
\bae
& \e \left| \nabla_x f \left( e^{-h(\theta)t} x + h(\theta)^{-1} \left( I_d - e^{-h(\theta)t} \right)g(\theta) + Z_t^\theta \right) \right| \\
=& \e \left| e^{-h(\theta)t} \nabla f \left( e^{-h(\theta)t} x + h(\theta)^{-1} \left( I_d - e^{-h(\theta)t} \right)g(\theta) + Z_t^\theta \right)\right| \\
\le& e^{-ct} \e_x \left| \nabla f(X_t^\theta)\right| \\
\le& C(1+|x|^m).
\eae
Therefore, by DCT we have that
\beq
\left| \nabla_x \e_x f(X_t^\theta) \right| = \left| \e \nabla_x f \left( e^{-h(\theta)t} x + h(\theta)^{-1} \left( I_d - e^{-h(\theta)t} \right)g(\theta) + Z_t^\theta \right) \right| = \left| e^{-h(\theta)t} \e_x \nabla f(X_t^\theta) \right| \le C(1+|x|^m).
\eeq
Similarly for $\nabla^2_x \e_x f(X_t^\theta)$, we have for $t \in (0,1]$
\bae
\left| \nabla^2_x \e_x f(X_t^\theta) \right| =& \left| \e e^{-h(\theta)t} \nabla^2 f \left( e^{-h(\theta)t} x + h(\theta)^{-1} \left( I_d - e^{-h(\theta)t} \right)g(\theta) + Z_t^\theta \right) e^{-h(\theta)t} \right| \\
=& \left| e^{-h(\theta)t} \e_x \nabla^2 f(X_t^\theta) e^{-h(\theta)t} \right| \\
\le& C(1+|x|^m).
\eae
Finally, for $\nabla_\theta \nabla^2 _x\e_x f(X_t^\theta)$, by  \eqref{norm} we have for $t\in(0,1]$ that
\bae
\label{theta x 1}
\left| \frac{\partial}{\partial \theta_i} \nabla^2_x \e_x f(X_t^\theta) \right| &\le 2\left| \frac{\partial}{\partial \theta_i} e^{-h(\theta)t} \right| \left| \e_x \nabla^2 f(X_t^\theta) \right| \left| e^{-h(\theta)t} \right| + \left| e^{-h(\theta)t} \frac{\partial}{\partial \theta_i} \e_x \nabla^2 f(X_t^\theta) e^{-h(\theta)t} \right| \\
&\le C(1+|x|^m) + e^{-ct} \left| \frac{\partial}{\partial \theta_i} \e_x \nabla^2 f(X_t^\theta) \right|.
\eae
Thus, it remains to prove a bound for $\frac{\partial}{\partial \theta_i} \e_x f_0(X_t^\theta)$, where $f_0$ is any polynomial bounded function such that
$$
|f_0(x)| + |\nabla f_0(x)| \le C(1+|x|^m),\quad \forall x\in \mathbb{R}^d.
$$
In order to establish this result, we need a bound for $\nabla_\theta \Sigma^{-1}_t(\theta)$ when $t \in [0,1]$. For $t \in (0,1]$,
\bae
\frac{\partial}{\partial \theta_i} \Sigma^{-1}_t(\theta) &= 2\sigma^2	\frac{\partial}{\partial \theta_i} \left[h(\theta)\left( I_d - e^{-2h(\theta)t} \right)^{-1} \right]\\
&\overset{(a)}{=}  2\sigma^2	\left( \frac{\partial}{\partial \theta_i}h(\theta) \right) \left( I_d - e^{-2h(\theta)t} \right)^{-1} + 2\sigma^2 \left( I_d - e^{-2h(\theta)t} \right)^{-1} h(\theta) \left(\frac{\partial}{\partial \theta_i}e^{-2h(\theta)t}\right) \left( I_d - e^{-2h(\theta)t} \right)^{-1} \\
&=  2\sigma^2  \left( I_d - e^{-2h(\theta)t} \right)^{-1} \left[ I_d - e^{-2h(\theta)t} - 2 e^{-2h(\theta)t} h(\theta)t\right] \left( \frac{\partial}{\partial \theta_i} h(\theta)t \right) \left( I_d - e^{-2h(\theta)t} \right)^{-1},
\eae
where in step $(a)$ we change the order of $\left( I_d - e^{-2h(\theta)t} \right)^{-1} $ and $h(\theta)$ since $h(\theta) e^{h(\theta)t} = e^{h(\theta) t} h(\theta)$. For $t \in [0,1]$, $0$ is the only singular point for $\nabla_\theta \Sigma^{-1}_t(\theta)$. Therefore to prove the uniform bound, it suffices to prove the limit exists when $t \to 0^+$. As $t \to 0+$,
\beq
I_d - e^{-2h(\theta)t} - 2 e^{-2h(\theta)t} h(\theta)t = I_d - 2h(\theta)t - \left(I_d -2h(\theta)t + 2h^2(\theta)t^2 + o(t^2) \right) = -2h^2(\theta)t^2 + o(t^2). 
\eeq
Therefore, 
\bae
\lim_{t \to 0+} \frac{\partial}{\partial \theta_i} \Sigma^{-1}_t(\theta) &= 2\sigma^2 \lim_{t \to 0+} \left( 2h(\theta)t +o(t) \right)^{-1} \left( -2h^2(\theta)t^2 + o(t^2)\right)\left( \frac{\partial}{\partial \theta_i} h(\theta) \right) \left( 2h(\theta)t +o(t) \right)^{-1}  \\
&= 2\sigma^2 \lim_{t \to 0+} \left( 2h(\theta) +o(1) \right)^{-1} \left( -2h^2(\theta) + o(1)\right)\left( \frac{\partial}{\partial \theta_i} h(\theta) \right) \left( 2h(\theta) +o(1) \right)^{-1} \\
&= -16 h(\theta) \left( \frac{\partial}{\partial \theta_i} h(\theta) \right) h^{-1}(\theta),
\eae
which together with the bound for $h(\theta)$ from Assumption \ref{condition} yields
\beq
\label{inverse bound}
\left| \nabla_\theta \Sigma^{-1}_t(\theta) \right| \le C, \quad t\in[0, 1].
\eeq

We will now analyze $\frac{\partial}{\partial \theta_i} \e_x f_0(X_t^\theta)$ for $t \in (0,1]$ using formula \eqref{expectation z} and changing the order of $\frac{\partial}{\partial \theta_i}$ and $\e_x$.
\begin{equation}
\begin{aligned}
\label{theta x 2}
&\int_{\mathbb{R}^d} \left| \frac{\partial}{\partial \theta_i} \left( f_0\left(f(t,x,\theta) + x'\right) \frac{1}{\sqrt{(2\pi)^d \left| \Sigma_t(\theta) \right|}} \exp\left\{ -x'^\top \Sigma^{-1}_t(\theta) x' \right\} \right) \right| dx' \\
\le& \int_{\mathbb{R}^d}  \left| \left( \frac{\partial}{\partial \theta_i} f(t,x,\theta) \right)^\top \nabla f_0\left(f(t,x,\theta) + x'\right) \frac{1}{\sqrt{(2\pi)^d \left| \Sigma_t(\theta) \right|}} \exp\left\{ -x'^\top \Sigma^{-1}_t(\theta) x' \right\} \right| dx'\\
+& \int_{\mathbb{R}^d} \left| f_0\left(f(t,x,\theta) + x'\right) \frac{\partial}{\partial \theta_i} \left( \frac{1}{\sqrt{(2\pi)^d \left| \Sigma_t(\theta) \right|}} \right) \exp\left\{ -x'^\top \Sigma^{-1}_t(\theta) x' \right\}  \right| dx'\\
+& 2\int_{\mathbb{R}^d} \left| f_0\left(f(t,x,\theta) + x'\right) \frac{1}{\sqrt{(2\pi)^d \left| \Sigma_t(\theta) \right|}} \exp\left\{ -x'^\top \Sigma^{-1}_t(\theta) x' \right\}  x'^\top \frac{\partial}{\partial \theta_i} \left( \Sigma^{-1}_t(\theta) \right) x' \right| dx'  \\
\overset{(a)}{\le}& C \e \left| \nabla f_0\left(f(t,x,\theta) + Z_t^\theta\right) \right| +  C \e \left| f_0\left(f(t,x,\theta) + Z_t^\theta \right) \right| + C \e \left| f_0\left(f(t,x,\theta) + Z_t^\theta \right) (Z_t^\theta)^\top Z_t^\theta \right| \\
\le& C \e_x \left| \nabla f_0\left( X_t^\theta \right) \right| + \e_x \left| f_0\left(X_t^\theta \right) \right| + \e_x \left| f_0^2\left(X_t^\theta \right)\right|  + \e \left| Z_t^\theta \right|^4\\
\overset{(b)}{\le} & C(1+|x|^m),
\end{aligned}
\end{equation}
where step $(a)$ is by \eqref{inverse bound} and the uniform bounds for $g(\theta), h(\theta)$ and step $(b)$ is by \eqref{exp bound} and the polynomial boudnds for $f_0^2, \nabla f_0$. Then, by the dominated convergence theorem,
\beq
\label{theta x 3}
\left|\nabla_\theta \e_x f_0(X_t^\theta)\right| \le C(1+|x|^m), \quad t\in(0,1].
\eeq
Combining \eqref{theta x 1} and \eqref{theta x 3}, we obtain the bound for $\nabla_\theta \nabla^2 _x\e_x f(X_t^\theta)$. The bound $\nabla^2_\theta \nabla^2_x\e_x f(X_t^\theta)$ can be obtained using similar calculations, which concludes the proof of the proposition.
\end{proof}

\section{Poisson PDEs}\label{Poisson appendix}

In this section we give the detailed proof of the regularities for the solutions of Poisson PDEs. We first show the proof of Lemma \ref{poisson eq}.

\begin{proof}[Proof of Lemma \ref{poisson eq}:]
We begin by proving that the integral \eqref{representation} is finite. We divide \eqref{representation} into two terms:
\begin{eqnarray}
v^1(x,\tilde x, \theta) &=& ( \e_{Y \sim \pi_\theta} f(Y) - \beta) \int_0^\infty \left( \nabla_\theta \e_{Y \sim \pi_\theta}f(Y) - \e_{x,\tilde x} \nabla f(X_t^\theta)\tilde X_t^\theta \right)^\top dt \notag \\
&=& (\e_{Y \sim \pi_\theta} f(Y) - \beta) \left[ \int_0^\infty \left( \nabla_\theta \e_{Y \sim \pi_\theta}f(Y) - \nabla_\theta \e_{x} f(X_t^\theta) \right)^\top dt + \int_0^\infty \left( \nabla_\theta \e_{x} f(X_t^\theta) - \e_{x,\tilde x} \nabla f(X_t^\theta) \tilde X_t^\theta \right)^\top dt \right] \notag \\
&=:& v^{1,1}(x, \theta) + v^{1,2}(x, \tilde x, \theta).
\end{eqnarray}
By Assumption \ref{condition} and \eqref{invariant density},
\beq
\left|\int_{\mathbb{R}^d} f(x') \nabla^i_\theta p_\infty(x',\theta) dx' \right| \le C\int \frac{ 1 + |x'|^m }{1 + |x'|^{m'} } dx' \overset{(a)}{\le} C, \quad i = 0, 1, 2,
\eeq
where step $(a)$ is by choosing $m'>m+d$. Thus by dominated convergence theorem (DCT):
\beq
\label{theta bound}
\left| \nabla^i_\theta \e_{Y \sim \pi_\theta} f(Y) \right| = \left|\int_{\mathbb{R}^d} f(x') \nabla^i_\theta p_\infty(x',\theta) dx' \right| \le C, \quad i =0,1,2.
\eeq
Similarly, we can bound $v^{1,1}$ as follows:
\bae
\label{control v11}
\left| v^{1,1}(x,\theta) \right| &\overset{(a)}{\le} C \int_0^1 \left( 1 + \left| \nabla_\theta \e_{x} f(X_t^\theta) \right| \right)dt + C \int_1^\infty \int_R \left( 1 + |x'|^m \right) \left| \nabla_\theta p_\infty(x',\theta)  - \nabla_\theta p_t(x, x', \theta) \right|dx'dt \\
&\overset{(b)}{\le} C + C \int_0^\infty \int_R \left( 1 + |x'|^m \right) \frac{1+|x|^{m'}}{(1+|x'|^{m''})(1+t)^2} dx'dt \\
&\overset{(c)}{\le} C \left( 1 + |x|^{m'} \right),
\eae
where steps $(a)$ is by Assumption \ref{condition} and \eqref{theta bound}, step $(b)$ by \eqref{x prime decay} and \eqref{expectation bound}, and step $(c)$ follows from selecting $m'' > m+d$. For $v^{1,2}$, by Assumption \ref{condition} and \eqref{normal bound} we have 
\beq
\label{pre DCT}
\left| \e_{x, 0} \nabla f(X_t^\theta) \tilde X_t^\theta \right| \le \e_{x, 0} \left|\tilde X_t^\theta\right|^2 + \e_x \left|\nabla f(X_t^\theta)\right|^2 \le C < \infty.
\eeq
Thus by DCT we have 
\beq
\label{change order gradient theta}
\nabla_\theta \e_{x} f(X_t^\theta) = \e_{x} \nabla_\theta f(X_t^\theta) = \e_{x, 0} \nabla f(X_t^\theta) \tilde X_t^\theta,
\eeq
which together with \eqref{tilde} derives
\begin{eqnarray}
\label{tilde x cal}
\nabla_\theta \e_{x} f(X_t^\theta) - \e_{x,\tilde x} \nabla f(X_t^\theta) \tilde X_t^\theta &=& \e_{x, 0} \nabla f(X_t^\theta) \tilde X_t^\theta - \e_{x,\tilde x}  \nabla f(X_t^\theta) \tilde X_t^\theta \notag \\
&=& -\e_{x} \nabla f\left(X_t^\theta\right) e^{-h(\theta)t} \tilde x.
\end{eqnarray}

Thus, $v^{1,2}$ satisfies the bound
\bae
\label{control v12}
\left| v^{1,2}(x, \tilde x, \theta) \right| &= \left| (\e_{Y \sim \pi_\theta} f(Y) -\beta) \int_0^\infty \e_{x} \nabla f\left(X_t^\theta\right) e^{-h(\theta)t} \tilde x dt \right| \\
&\overset{(a)}{\le} C \int_0^\infty \left( 1 + \e_{x} \left|X_t^\theta\right|^m \right) e^{-ct} dt \cdot \left| \tilde x  \right| \\
&\overset{(b)}{\le} C \int_0^\infty \left( 1 + |x|^{m'} \right) e^{-ct} dt \cdot \left| \tilde x  \right| \\
&\le C \left(1 + |x|^{m'} + |x'|^{m'} \right),
\eae
where step $(a)$ is by Assumption \ref{condition}, \eqref{theta bound} and $\lambda_{\max} \left( e^{-h(\theta)t} \right) \le e^{-ct}$. Step $(b)$ is by \eqref{normal bound}.

Next we show $v^1(x,\tilde x, \theta)$ is differentiable with respect to $x, \tilde x,$ and $\theta$. We can prove this using a version of the dominated convergence theorem (see Theorem 2.27 in \cite{folland1999real}), where it suffices to show that the derivative of the integrand is bounded by an integrable function. Using the same analysis as in \eqref{control v12}, we can show that
\beq
\label{control vxtilde}
\left| \int_0^\infty e^{-h(\theta)t} \e_{x} \nabla f(X_t^\theta) dt \right| \le C \int_0^\infty \left( 1 + \e_{x} \left|X_t^\theta\right|^m \right) e^{-ct} dt  \le C \left(1 + |x|^{m'} \right).
\eeq
Therefore, by the dominated convergence theorem, we know  $v^1$ is differentiable with respect to $\tilde x$. Furthermore, we can change the order of $\nabla_{\tilde x}$ and the integral in $v^1$ and obtain 
\beq
\left| \nabla_{\tilde x} v^1(x, \tilde x, \theta) \right| = \left| (\e_{Y \sim \pi_\theta} f(Y) -\beta) \int_0^\infty \e_{x} \nabla f\left(X_t^\theta\right) e^{-h(\theta)t} dt \right| \overset{(a)}{\le} C \left(1 + |x|^{m'} \right),
\eeq
where step $(a)$ is by \eqref{theta bound} and \eqref{control vxtilde}.

By \eqref{x prime decay}, \eqref{theta bound}, and the same approach as in \eqref{control v11}, we have  
\begin{equation}
\begin{aligned}
\label{cal 1}
&\left| \nabla_\theta \e_{Y \sim \pi_\theta} f(Y) \int_1^\infty \left( \nabla_\theta \e_{Y \sim \pi_\theta}f(Y) - \nabla_\theta \e_{x} f(X_t^\theta) \right) dt + \left( \e_{Y \sim \pi_\theta} f(Y) - \beta \right) \int_1^\infty \left( \nabla^2_\theta \e_{Y \sim \pi_\theta}f(Y) - \nabla^2_\theta \e_{x} f(X_t^\theta) \right) dt \right| \\
\le& C \int_1^\infty \int_{\mathbb{R}^d} \left| f(x') \nabla_\theta \left[ p_\infty(x',\theta) - p_t(x, x', \theta) \right] \right| dx'dt + C \int_1^\infty \int_{\mathbb{R}^d} \left| f(x') \nabla_\theta^2 \left[ p_\infty(x',\theta) - p_t(x, x', \theta) \right] \right| dx' dt \\
\le& C \left( 1 + |x|^{m'} \right).
\end{aligned}
\end{equation}
By \eqref{expectation bound} and \eqref{theta bound}, 
\bae
\label{cal 2}
&\left| \nabla_\theta \e_{Y \sim \pi_\theta} f(Y) \int_0^1 \nabla_\theta \e_{Y \sim \pi_\theta}f(Y) - \nabla_\theta \e_{x} f(X_t^\theta) dt + \left( \e_{Y \sim \pi_\theta} f(Y) - \beta \right) \int_0^1 \nabla^2_\theta \e_{Y \sim \pi_\theta}f(Y) - \nabla^2_\theta \e_{x} f(X_t^\theta) dt \right| \\ 
\le& C \left( 1 + |x|^{m'}\right). 
\eae
By \eqref{cal 1}, \eqref{cal 2}, and DCT we know $v^{1,1}$ is differentiable with respect to $\theta$ and 
\beq
\label{control v11 theta}
\left| \nabla_\theta v^{1,1}(x,\theta) \right|
\le C \left( 1 + |x|^{m'}\right) .
\eeq
For $\nabla_\theta v^{1,2}$, by \eqref{tilde x cal} we have for any $i \in \{1,2,\cdots, \ell \}$ 
\bae
\label{change order v12}
&\left| \int_0^\infty \frac{\partial}{\partial \theta_i} \left( \nabla_\theta \e_{x} f(X_t^\theta) - \e_{x,\tilde x} \nabla f(X_t^\theta) \tilde X_t^\theta \right) dt \right| \\
=&  \left| \int_0^\infty \e_{x} \nabla f\left(X_t^\theta\right) \left( \frac{\partial}{\partial \theta_i} e^{-h(\theta)t} \right) \tilde x dt + \int_0^\infty \left( \frac{\partial}{\partial \theta_i} \e_{x} \nabla f\left(X_t^\theta\right) \right) e^{-h(\theta)t} \tilde x dt \right| \\
\overset{(a)}{\le}& \left|\tilde x\right| \cdot \int_0^\infty \left|\frac{\partial}{\partial \theta_i} e^{-h(\theta)t}\right| \cdot \left|\e_x \nabla f(X_t^\theta)\right| dt + \left|\tilde x\right| \cdot \int_0^\infty \left|e^{-h(\theta)t}\right| \cdot \left|\e_x \frac{\partial}{\partial \theta_i} \nabla f(X_t^\theta)\right| dt\\
=&: I_4 + I_5.
\eae
where in step $(a)$ we use 
\beq
\frac{\partial}{\partial \theta_i} \e_{x} \nabla f\left(X_t^\theta\right) = \e_x \frac{\partial}{\partial \theta_i} \nabla f(X_t^\theta),
\eeq
which is due to \eqref{pre DCT} and \eqref{change order gradient theta}.

By \eqref{norm},
\bae
\label{I4}
& I_4 \le C\left| \tilde x\right| \left( \left| \int_0^1 \e_x \nabla f(X_t^\theta) e^{-ct} t dt \right| + \left| \int_1^\infty \e_x \nabla f(X_t^\theta) e^{-ct} t dt \right| \right) \\
\overset{(a)}{\le}&  C |\tilde x| \left( 1 + \int_1^\infty e^{-ct}t \int_{\mathbb{R}^d} \left( 1 + |x'|^{m} \right)\left| p_t(x, x', \theta) - p_\infty(x',\theta) \right| dx'dt + \int_1^\infty e^{-ct}t \int_{\mathbb{R}^d} \left( 1 + |x'|^{m} \right) \left| p_\infty(x',\theta) \right| dx'dt \right) \\
\overset{(b)}{\le}& C \left( 1 + |x|^{m'} + |\tilde x|^{m'} \right),
\eae
where in step (a) we used \eqref{expectation bound} and step $(b)$ is by \eqref{invariant density}, \eqref{x prime decay}, and the same analysis as in \eqref{theta bound} and \eqref{control v11}. Similarly, 
\bae
\label{I5}
I_5 \le& C\left| \tilde x\right| \left( \int_0^1 \left|\e_x \frac{\partial}{\partial \theta_i} \nabla f(X_t^\theta)\right| e^{-ct} dt + \int_1^\infty \left|\e_x \frac{\partial}{\partial \theta_i} \nabla f(X_t^\theta)\right| e^{-ct} dt \right)  \\
\le& C |\tilde x| + C |\tilde x| \cdot \int_1^\infty e^{-ct} \int_{\mathbb{R}^d} \left( 1 + |x'|^{m} \right)\left| \frac{\partial}{ \partial \theta_i} p_t(x, x', \theta) - \frac{\partial}{ \partial \theta_i} p_\infty(x',\theta) \right| dx'dt \\
+& C |\tilde x| \cdot \int_1^\infty e^{-ct} \int_{\mathbb{R}^d} \left( 1 + |x'|^{m} \right) \left| \frac{\partial}{ \partial \theta_i} p_\infty(x',\theta) \right| dx'dt \\
\le& C \left( 1 + |x|^{m'} + |\tilde x|^{m'} \right).
\eae
Combining \eqref{control v12}, \eqref{change order v12}, \eqref{I4}, \eqref{I5}, and DCT, we know 
$v^{1,2}$ is differentiable with respect to $\theta$ and for any $i \in \{1,2,\cdots, \ell\}$

\beq 
\label{control v12 theta}
\left|\frac{\partial v^{1,2}}{\partial \theta_i}(x, \tilde x, \theta)\right| \le \left| \frac{\partial}{\partial \theta_i} \e_{Y \sim \pi_\theta} f(Y) \int_0^\infty (e^{-h(\theta)t} \tilde x)^\top \e_{x} \nabla f\left(X_t^\theta\right) dt\right| 
+ \left|\e_{Y \sim \pi_\theta} f(Y) -\beta \right| \cdot (I_4 + I_5) \le C\left(1+ |x|^{m'} + |\tilde x|^{m^{\prime}}\right), 
\eeq
which together with \eqref{control v11 theta} yields
\beq
\left|\nabla_\theta v^1(x, \tilde x, \theta)\right| \le C\left(1+ |x|^{m'} + |\tilde x|^{m^{\prime}}\right).
\eeq

Similarly, by \eqref{x decay}, \eqref{normal bound}, and \eqref{expectation bound}, 
$$
\begin{aligned}
&\left| \int_0^\infty \nabla_x \left( \nabla_\theta \e_{Y \sim \pi_\theta}f(Y) - \nabla_\theta \e_{x} f(X_t^\theta) \right) dt \right| \\
=& \left| \int_0^1 \nabla_x \nabla_\theta \e_{x} f(X_t^\theta) dx'dt \right| + \left| \int_1^{+\infty} \int_{\mathbb{R}} f(x') \nabla_x \nabla_\theta p_t(x,x',\theta) dx'dt \right| \\
\le&  C \left( 1 + |x|^{m'} \right)
\end{aligned}
$$
and 
$$
\left| \int_0^\infty \nabla_x\left( \nabla_\theta \e_{x} f(X_t^\theta) - \e_{x,\tilde x} \nabla f(X_t^\theta) \tilde X_t^\theta \right) dt \right| = \left|\int_0^\infty \nabla_x \left( \e_{x} \nabla f\left(X_t^\theta\right) \right) e^{-h(\theta)t} \tilde x dt \right| \le C \left(1 + |x|^{m'} + |\tilde x|^{m'} \right).
$$
By DCT and \eqref{theta bound},
\bae
\label{control v1 x}
\left| \nabla_x v^{1,1}(x, \theta) \right| &= \left| \left( \e_{Y \sim \pi_\theta} f(Y) - \beta \right) \int_0^\infty \int_{\mathbb{R}^d} f(x') \nabla_x \nabla_ \theta p_t(x, x', \theta) dx'dt \right| \le C \left( 1 + |x|^{m'} \right), \\
\left| \nabla_x v^{1,2}(x, \tilde x, \theta) \right| &= \left| (\e_{Y \sim \pi_\theta} f(Y) -\beta) \int_0^\infty \nabla_x \left( \e_{x} \nabla f\left(X_t^\theta\right) \right) e^{-h(\theta)t} \tilde x dt \right|  \le C \left(1 + |x|^{m'} + |\tilde x|^{m'} \right).
\eae
Then, for $\nabla_x^2 v^1(x,\tilde x, \theta)$, we have 
$$
\begin{aligned}
&\left| \int_0^\infty \nabla^2_x \left( \nabla_\theta \e_{Y \sim \pi_\theta}f(Y) - \nabla_\theta \e_{x} f(X_t^\theta) \right) dt \right| \\
=& \left| \int_0^1 \nabla^2_x \nabla_\theta \e_{x} f(X_t^\theta) dx'dt \right| +  \left|\int_0^\infty \int_{\mathbb{R}^d} f(x') \nabla^2_x \nabla_ \theta p_t(x, x', \theta) dx'dt \right| \\
\le& C \left( 1 + |x|^{m'} \right),\\
\end{aligned}
$$
and
$$
\begin{aligned}
\left| \int_0^\infty \nabla^2_x\left( \nabla_\theta \e_{x} f(X_t^\theta) - \e_{x,\tilde x} \nabla_x f(X_t^\theta) \tilde X_t^\theta \right) dt \right| &= \left| \int_0^\infty \nabla^2_x \left( \e_{x} \nabla f\left(X_t^\theta\right) \right)  e^{-h(\theta)t} \tilde x dt \right| \le C \left(1 + |x|^{m'} + |\tilde x|^{m'} \right).
\end{aligned}
$$
By DCT and \eqref{theta bound},
\bae
\label{control v1 xx}
\left| \nabla_x^2 v^{1,1}(x, \theta) \right| &=\left| \left( \e_{Y \sim \pi_\theta} f(Y) - \beta \right) \int_0^\infty \int_{\mathbb{R}^d} f(x') \nabla^2_x \nabla_ \theta p_t(x, x', \theta) dx'dt \right| \le  C \left( 1 + |x|^{m'} \right), \\
\left| \nabla_x^2 v^{1,2}(x, \tilde x, \theta) \right| &=\left| (\e_{Y \sim \pi_\theta} f(Y) -\beta) \int_0^\infty \nabla^2_x \left( \e_{x} \nabla f\left(X_t^\theta\right) \right) e^{-h(\theta)t} \tilde x dt \right| \le C \left(1 + |x|^{m'} + |\tilde x|^{m'} \right).
\eae

Finally, we verify that $v^1$ is a solution to the PDE \eqref{PDE}. Note that 
\beq
\label{fubini 1}
\int_0^\infty \e_{x,\tilde x}  \e_{X^\theta_s,\tilde X_s^\theta} \left| G^1(X_t^\theta,\tilde X_t^\theta, \theta) \right| dt \overset{(a)}{=} \int_0^\infty \e_{x,\tilde x}  \left| G^1(X_{t+s}^\theta,\tilde X_{t+s}^\theta, \theta) \right| dt \overset{(b)}{=} \int_s^\infty \e_{x,\tilde x}  \left| G^1(X_{t}^\theta,\tilde X_{t}^\theta, \theta) \right| dt \overset{(c)}{<} \infty,
\eeq
where step $(a)$ is by the Markov property of the process $(X_\cdot^\theta, \tilde X_\cdot^\theta)$, step $(b)$ by change of variables and step $(c)$ is by 
the convergence of $v^1$. By Fubini's theorem,
\beq
\label{fubini 2}
\e_{x,\tilde x} v^1(X_s^\theta, \tilde X_s^\theta, \theta) = \e_{x,\tilde x}  \int_0^\infty \e_{X^\theta_s,\tilde X_s^\theta} G^1(X_t^\theta,\tilde X_t^\theta, \theta) dt = \int_0^\infty  \e_{x,\tilde x} \e_{X^\theta_s,\tilde X_s^\theta} G^1(X_t^\theta,\tilde X_t^\theta, \theta) dt.
\eeq
Combining \eqref{fubini 1} and \eqref{fubini 2}, we have that 
\bae
\label{weak solution}
\frac{1}{s} \left[ \e_{x,\tilde x} v^1(X_s^\theta, \tilde X_s^\theta, \theta) - v^1(x, \tilde x, \theta) \right] &= \frac{1}{s} \left[ - \int_0^\infty \e_{x,\tilde x} \e_{X^\theta_s,\tilde X_s^\theta} G^1(X_t^\theta,\tilde X_t^\theta, \theta) dt + \int_0^\infty \e_{x,\tilde x} G^1(X_t^\theta,\tilde X_t^\theta, \theta) dt \right]\\
&= \frac{1}{s} \left[ - \int_0^\infty \e_{x,\tilde x}  G^1(X_{t+s}^\theta,\tilde X_{t+s}^\theta, \theta) dt + \int_0^\infty \e_{x,\tilde x} G^1(X_t^\theta,\tilde X_t^\theta, \theta) dt \right]\\
&= \frac{1}{s} \left[ - \int_s^\infty \e_{x,\tilde x}  G^1(X_{t}^\theta,\tilde X_{t}^\theta, \theta) dt + \int_0^\infty \e_{x,\tilde x} G^1(X_t^\theta,\tilde X_t^\theta, \theta) dt \right]\\
&= \frac{1}{s} \int_0^s\e_{x,\tilde x} G^1(X_t^\theta,\tilde X_t^\theta, \theta) dt,
\eae
Let $s \to 0^+$. By the definition of the infinitesimal generator and since $v^1(x,\tilde x, \theta)$ is twice differentiable with respect to $x$ and once differentiable with respect to $\tilde x$, $v^1(x,\tilde x,\theta)$ is the classical solution of the Poisson PDE \eqref{PDE}.
\end{proof}

Now we show the proof of Lemma \ref{poisson eq 2}.

\begin{proof}[Proof of Lemma \ref{poisson eq 2}:]
The proof is exactly the same as in Lemma \ref{poisson eq} except for the presence of the dimension $\bar x$ and $\mathcal{L}_{\bar x}$. We first show that the integral in \eqref{representation2} converges. Note that 
\begin{eqnarray}
v^2(x,\tilde x, \bar x, \theta) &=& \int_0^\infty  
\e_{x,\tilde x, \bar x} \left[  \left(\e_{Y \sim \pi_{\theta}}f(Y) - f(\bar X_t^\theta) \right) \cdot  \left(\nabla f(X_t^\theta) \tilde X_t^\theta \right)^\top \right] dt \notag \\
&\overset{(a)}{=}& \int_0^\infty  \left( \e_{Y \sim \pi_{\theta}}f(Y) - \e_{\bar x}f(\bar X_t^\theta) \right) \cdot \e_{x,\tilde x} \left( \nabla f(X_t^\theta) \tilde X_t^\theta \right)^\top dt,
\end{eqnarray}
where step $(a)$ is by the independence of $\bar X^\theta_\cdot$ and $(X^\theta_\cdot, \tilde X^\theta_\cdot)$. 

We now prove a uniform bound for $ \e_{x,\tilde x} \nabla f(X_t^\theta) \tilde X_t^\theta$ and then by the ergodicity of $\bar X^\theta_\cdot$ in Lemma \ref{ergodic estimation} we can show that the integrals converge. 
\bae
\left| \e_{x,\tilde x} \nabla f(X_t^\theta) \tilde X_t^\theta \right| &= \left| \e_{x,\tilde x} \nabla f(X_t^\theta) \tilde X_t^\theta - \nabla_\theta \e_x f(X_t^\theta) + \nabla_\theta \e_x f(X_t^\theta) \right|\\
&\overset{(a)}{\le} \left| \e_{x} \nabla f(X_t^\theta) e^{-h(\theta)t} \tilde x \right| + \left| \nabla_\theta \e_x f(X_t^\theta) \right|,
\eae
where step $(a)$ is by \eqref{tilde x cal}. Therefore, for any $t\in[0,1]$, we can conclude 
\beq
\label{control v2 half 1}
\left| \e_{x} \nabla f(X_t^\theta) e^{-h(\theta)t} \tilde x \right| + \left| \nabla_\theta \e_x f(X_t^\theta) \right| \le C \left(1 + \left| x \right|^{m'} + \left| \tilde x \right|^{m'} \right),
\eeq
where we have used Assumption \ref{condition} and equation \eqref{expectation bound}. For $t>1$, we have 
\begin{equation}
\begin{aligned}
\label{control v2 half 2}
&\left| \e_{x} \nabla f(X_t^\theta) e^{-h(\theta)t} \tilde x \right| + \left| \nabla_\theta \e_x f(X_t^\theta) \right|\\
\overset{(a)}{\le}& C \left(1 + \e_x \left|X_t^\theta\right|^m \right) \cdot \left| \tilde x \right| + C \int_{\mathbb{R}^d} \left( 1 + |x'|^m \right) \left| \nabla_\theta p_t(x, x', \theta) -  \nabla_\theta p_\infty(x', \theta) \right| dx' + C \int_{\mathbb{R}^d} \left( 1 + |x'|^m \right) \left| \nabla_\theta p_\infty(x', \theta) \right| dx' \\
\overset{(b)}{\le}& C \left(1 + \left| x \right|^{m'} + \left| \tilde x \right|^{m'} \right),
\end{aligned}
\end{equation}
where step $(a)$ uses Assumption \ref{condition} and step $(b)$ uses Proposition \ref{ergodic estimation} and the same calculations as in \eqref{theta bound} and \eqref{control v11}. Combining \eqref{control v2 half 1} and \eqref{control v2 half 2}, we have for any $t\ge0$
\beq
\label{control v2 half}
\left| \e_{x,\tilde x} \nabla f(X_t^\theta) \tilde X_t^\theta \right| \le C \left(1 + \left| x \right|^{m'} + \left| \tilde x \right|^{m'} \right).
\eeq
Thus, by \eqref{control v2 half} and the same derivation as in \eqref{control v11}, we have 
\begin{eqnarray}
\left| v^2(x,\tilde x, \bar x, \theta) \right| &\le& C \left(1 + \left| x \right|^{m'} + \left| \tilde x \right|^{m'} \right) \cdot \int_0^\infty  \left| \e_{\bar x} f(\bar X_t^\theta) - \e_{Y \sim \pi_{\theta}}f(Y) \right| dt \notag \\
&\le& C\left(1+ |x|^{m'} + |\tilde x|^{m'} + |\bar x|^{m^{\prime}}\right).
\end{eqnarray}

We next show that $v^2(x,\tilde x, \bar x, \theta)$ is differentiable with respect to $x, \tilde x, \bar x, \theta$. Similar to Lemma \ref{poisson eq}, we first change the order of differentiation and integration and show the corresponding integral exists. Then, we apply DCT to prove that the differentiation and integration can be interchanged. For the ergodic process $\bar X_\cdot^\theta$, by \eqref{control v2 half}, \eqref{x decay}, and \eqref{expectation bound}, we have the bounds
\bae 
\int_0^\infty \int_{\mathbb{R}^d} \left| f(\bar x') \nabla_{\bar x} p_t(\bar x, \bar x', \theta) \right| d\bar x' \cdot \left| \e_{x,\tilde x} \nabla f(X_t^\theta) \tilde X_t^\theta  \right|dt &\le  C\left(1+ |x|^{m'} + |\tilde x|^{m'} + |\bar x|^{m^{\prime}}\right), \\
\int_0^\infty \int_{\mathbb{R}^d} \left| f(\bar x') \nabla^2_{\bar x} p_t(\bar x, \bar x', \theta) \right| d\bar x' \cdot \left| \e_{x,\tilde x} \nabla f(X_t^\theta) \tilde X_t^\theta  \right| dt &\le C\left(1+ |x|^{m'} + |\tilde x|^{m'} + |\bar x|^{m^{\prime}}\right),
\eae
and thus by the DCT
\beq
\sum_{i=1}^2 \left| \nabla^i_{\bar x} v^2(x,\tilde x, \bar x, \theta) \right| \le C\left(1+ |x|^{m'} + |\tilde x|^{m'} + |\bar x|^{m^{\prime}}\right). 
\eeq

To address $\nabla_x v^2, \nabla_x^2 v^2$, we first note that for any $i,j \in \{1,2,\cdots, d\}$
\bae
\left| \nabla_x \e_{x,\tilde x} \nabla f(X_t^\theta) \tilde X_t^\theta \right| &\le \left|  \nabla_x \e_{x} \nabla f(X_t^\theta) e^{-h(\theta)t} \tilde x \right| + \left| \nabla_x \nabla_\theta \e_x f(X_t^\theta) \right| \overset{(a)}{\le} C \left(1 + \left| x \right|^{m'} + \left| \tilde x \right|^{m'} \right),\\
\left| \frac{\partial^2 }{\partial x_i \partial x_j} \e_{x,\tilde x} \nabla f(X_t^\theta) \tilde X_t^\theta \right| &\le \left| \frac{\partial^2 }{\partial x_i \partial x_j} \e_{x} \nabla f(X_t^\theta) e^{-h(\theta)t} \tilde x \right| + \left| \frac{\partial^2 }{\partial x_i \partial x_j} \nabla_\theta \e_x f(X_t^\theta) \right| \overset{(a)}{\le} C \left(1 + \left| x \right|^{m'} + \left| \tilde x \right|^{m'} \right),
\eae
where in step $(a)$ we use \eqref{x decay} when $t>1$ and \eqref{expectation bound} for $t\in[0,1]$. Thus we have $ \forall i,j \in \{1,2,\cdots, d\}$
\bae
\int_0^\infty \left| \left[ \e_{Y \sim \pi_{\theta}}f(Y) - \e_{\bar x}f(\bar X_t^\theta) \right] \right| \cdot \left| \nabla_x \e_{x,\tilde x} \nabla f(X_t^\theta) \tilde X_t^\theta \right| dt &\le C\left(1+ |x|^{m'} + |\tilde x|^{m'} + |\bar x|^{m^{\prime}}\right),\\
\int_0^\infty  \left| \left[ \e_{Y \sim \pi_{\theta}}f(Y) - \e_{\bar x}f(\bar X_t^\theta) \right] \right| \cdot \left| \frac{\partial^2}{\partial x_i \partial x_j} \e_{x,\tilde x} \nabla f(X_t^\theta) \tilde X_t^\theta \right| dt &\le C\left(1+ |x|^{m'} + |\tilde x|^{m'} + |\bar x|^{m^{\prime}}\right).
\eae
Then by DCT,
\beq
\sum_{i=1}^2 \left| \nabla^i_x v^2(x,\tilde x,\bar x, \theta) \right| \le C\left(1+ |x|^{m'} + |\tilde x|^{m'} + |\bar x|^{m^{\prime}}\right).
\eeq

Then for $\nabla_\theta v^2$, first we have for any $i \in \{1,2,\cdots, \ell\}$
\bae
\left| \frac{\partial}{\partial \theta_i} \e_{x,\tilde x} \nabla f(X_t^\theta) \tilde X_t^\theta \right| &\le  \left| \left( \frac{\partial}{\partial \theta_i} \e_{x} \nabla f(X_t^\theta) \right) e^{-h(\theta)t} \tilde x \right| + \left| \e_{x} \nabla f(X_t^\theta) \left(\frac{\partial }{\partial \theta_i} e^{-h(\theta)t}\right) \tilde x \right| + \left| \frac{\partial}{\partial \theta_i} \nabla_\theta \e_x f(X_t^\theta) \right| \\
&\overset{(a)}{\le} C \left(1 + \left| x \right|^{m'} + \left| \tilde x \right|^{m'} \right),
\eae
where in step $(a)$ we use \eqref{norm} and the same analysis as in \eqref{I5}. Thus 
\begin{equation}
\begin{aligned}
&\left| \int_0^\infty \frac{\partial}{\partial \theta_i} \left( \left[ \e_{Y \sim \pi_{\theta}}f(Y) - \e_{\bar x}f(\bar X_t^\theta) \right] \cdot \e_{x,\tilde x} \nabla f(X_t^\theta) \tilde X_t^\theta \right) dt \right| \\
\le& \int_0^\infty \int_{\mathbb{R}^d} \left| f(x') \frac{\partial}{\partial \theta_i} \left( p_\infty(\bar x', \theta) - p_t(\bar x, \bar x', \theta) \right) \right| d\bar x' \cdot \left| \e_{x,\tilde x} \nabla f(X_t^\theta) \tilde X_t^\theta \right| dt\\
+& \int_0^\infty \int_{\mathbb{R}^d} \left| f(x') \left( p_\infty(\bar x', \theta) - p_t(\bar  x, \bar x', \theta) \right) \right| d\bar x' \cdot \left| \frac{\partial }{\partial \theta_i}\e_{x,\tilde x} \nabla f(X_t^\theta) \tilde X_t^\theta \right| dt\\
\le& C\left(1+ |x|^{m'} + |\tilde x|^{m'} + |\bar x|^{m^{\prime}}\right),
\end{aligned}
\end{equation}
which together with the DCT derives
\beq
\left| \nabla_\theta v^2(x,\tilde x,\bar x, \theta) \right| \le C\left(1+ |x|^{m'} + |\tilde x|^{m'} + |\bar x|^{m^{\prime}}\right).
\eeq
Finally, note that 
\bae
&\left| \int_0^\infty \nabla_{\tilde x} \left( \left[ \e_{Y \sim \pi_{\theta}}f(Y) - \e_{\bar x}f(\bar X_t^\theta) \right] \cdot \e_{x,\tilde x} \nabla f(X_t^\theta) \tilde X_t^\theta \right) dt \right| \\
\le&  \int_0^\infty  \left| \left[ \e_{Y \sim \pi_{\theta}}f(Y) - \e_{\bar x} f(\bar X_t^\theta) \right] \right| \cdot \left| \nabla_{\tilde x} \e_{x,\tilde x} \nabla f(X_t^\theta) \tilde X_t^\theta \right| dt \\
\le& C\int_0^\infty  \left| \left[ \e_{Y \sim \pi_{\theta}}f(Y) - \e_{\bar x}f(\bar  X_t^\theta) \right] \right| \cdot \left| \e_{x}  \nabla f(X_t^\theta) e^{-h(\theta)t} \right| dt  \\
\le& C\left(1+ |x|^{m'} + |\bar x|^{m^{\prime}}\right)
\eae
and then by DCT 
\beq
\left| \nabla_{\tilde x}v^2(x,\tilde x, \bar x, \theta) \right| \le  C\left(1+ |x|^{m'} + |\bar x|^{m^{\prime}}\right).
\eeq
By the same calculations as in \eqref{weak solution}, we know $v^2$ is the classical solution of PDE \eqref{PDE 2} and the bound \eqref{control v2} holds. 
\end{proof}

\bibliographystyle{plain}

\bibliography{cite}

\end{document}